\documentclass[10pt,twocolumn,letterpaper]{article}

\usepackage[top=0.75in,bottom=0.75in,left=0.68in,right=0.68in,columnsep=0.22in]{geometry}

\usepackage{times}
\usepackage{microtype}
\usepackage[utf8]{inputenc}
\usepackage[T1]{fontenc}
\usepackage{float}
\usepackage{amsmath, amssymb, amsthm}
\usepackage{graphicx}
\usepackage{booktabs}
\usepackage{bm}
\usepackage{multirow}
\usepackage{makecell}
\usepackage{algorithm}
\usepackage{algorithmicx}
\usepackage{algpseudocode}
\usepackage[font=small,labelfont=bf]{caption}
\usepackage{subcaption}
\usepackage[hyphens]{url}
\usepackage{natbib}
\let\cite\citep
\setcitestyle{aysep={,}}

\usepackage[pagebackref=true,breaklinks=true,colorlinks,citecolor=blue,linkcolor=blue,urlcolor=blue]{hyperref}

\newcommand{\eg}{{\em e.g.}}

\newtheorem{theorem}{Theorem}
\newtheorem{lemma}{Lemma}
\newtheorem{definition}{Definition}

\newtheorem{remark}{Remark}
\newtheorem{corollary}{Corollary}

\newtheorem{assumption}{Assumption}

\title{\Large \bf Exact and Asymptotically Complete Robust Verification of Neural Networks via Ising Solvers}

\author{
    Wenxin Li$^{1}$, Wenchao Liu$^{1}$, Anzhe Jiang$^{1}$, Weihao Li$^{2}$, Chuan Wang$^{3,*}$\\[0.25em]
    Qi Gao$^{1}$, Yin Ma$^{1,3}$, Hai Wei$^{1}$, Kai Wen$^{1,*}$\\[0.45em]
    \small $^{1}$Beijing QBoson Quantum Technology Co., Ltd. \qquad $^{2}$Tsinghua University \qquad $^{3}$Beijing Normal University
}
\date{}

\begin{document}

\maketitle
{\renewcommand{\thefootnote}{\fnsymbol{footnote}}%
\footnotetext[1]{Corresponding authors.}%
\footnotetext[2]{Emails: \texttt{liwx@boseq.com}, \texttt{wangchuan@bnu.edu.cn}, \texttt{wenk@boseq.com}}%
}

\begin{abstract}
We present an Ising-compatible framework for formal neural-network robustness verification under bounded input perturbations. For piecewise activations (both continuous piecewise-linear and discrete piecewise-constant), the Exact Logarithmic Piecewise Model (Log-PW) provides an exact, sound, and complete formulation with a state-optimal logarithmic encoding, reducing the binary variables per neuron from linear $\mathcal{O}(\kappa)$ to information-theoretically minimal logarithmic complexity $\mathcal{O}(\lceil \log_2 \kappa \rceil)$ via logarithmic SOS2 (Log-PWL) and SOS1 (Log-PWC) constraints. For general element-wise activations, the Asymptotic Step-Envelope Model (Step-Env) leverages Log-PWC to construct sound, refinable endpoint enclosures. An exact monotonic reduction based on Jordan decomposition further extends this framework to arbitrary non-monotonic activations. We further develop a hybrid Benders solver with output-sensitive iteration bounds, leveraging a Helly-based conflict compression theory to avoid generic exponential cut explosion. In particular, for shallow networks with fixed input dimension, this establishes an unconditionally polynomial Benders iteration complexity in network size. Interval pruning, certificate transfer for pruned networks, and layerwise classical--Ising partitioning further reduce spin requirements. Experiments demonstrate exact certification fidelity for piecewise-linear networks and near-reference accuracy for non-linear activations under compact spin budgets. On MNIST benchmarks, Log-PW halves binary decision spins over One-Hot formulations, while Helly-based cuts achieve up to a $13.41\times$ speedup and an $87.6\%$ reduction in master queries over standard baselines.
\end{abstract}

\section{Introduction}

\begin{figure}[t]
    \centering
    \includegraphics[width=\columnwidth]{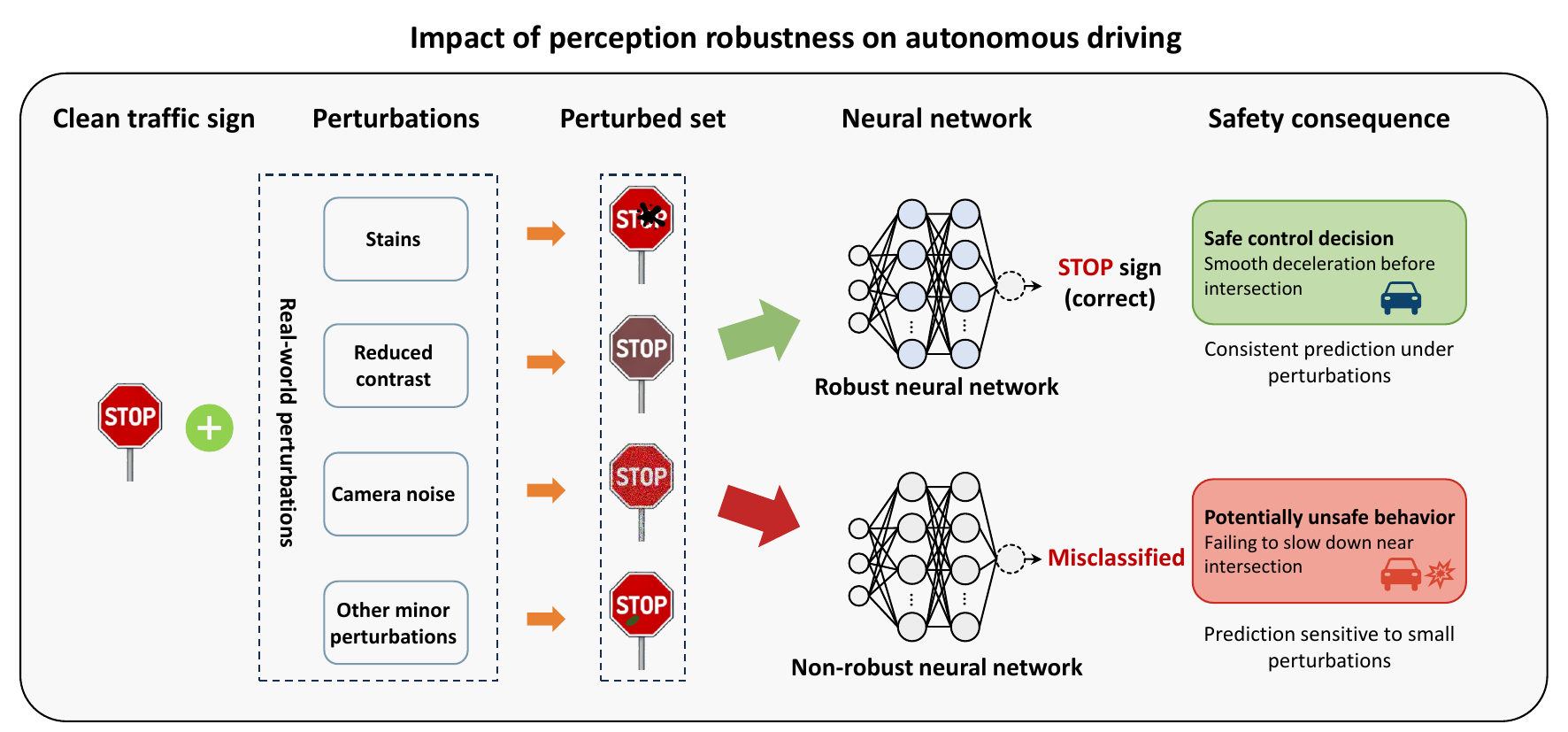} 
    \caption{Illustration of adversarial vulnerability in traffic sign recognition. Small, imperceptible perturbations can cause a neural network to misclassify a stop sign, highlighting the need for robust verification.} 
    \label{fig:traffic_sign} 
\end{figure}

\begin{figure*}[t]
    \centering
    \includegraphics[width=0.78\textwidth]{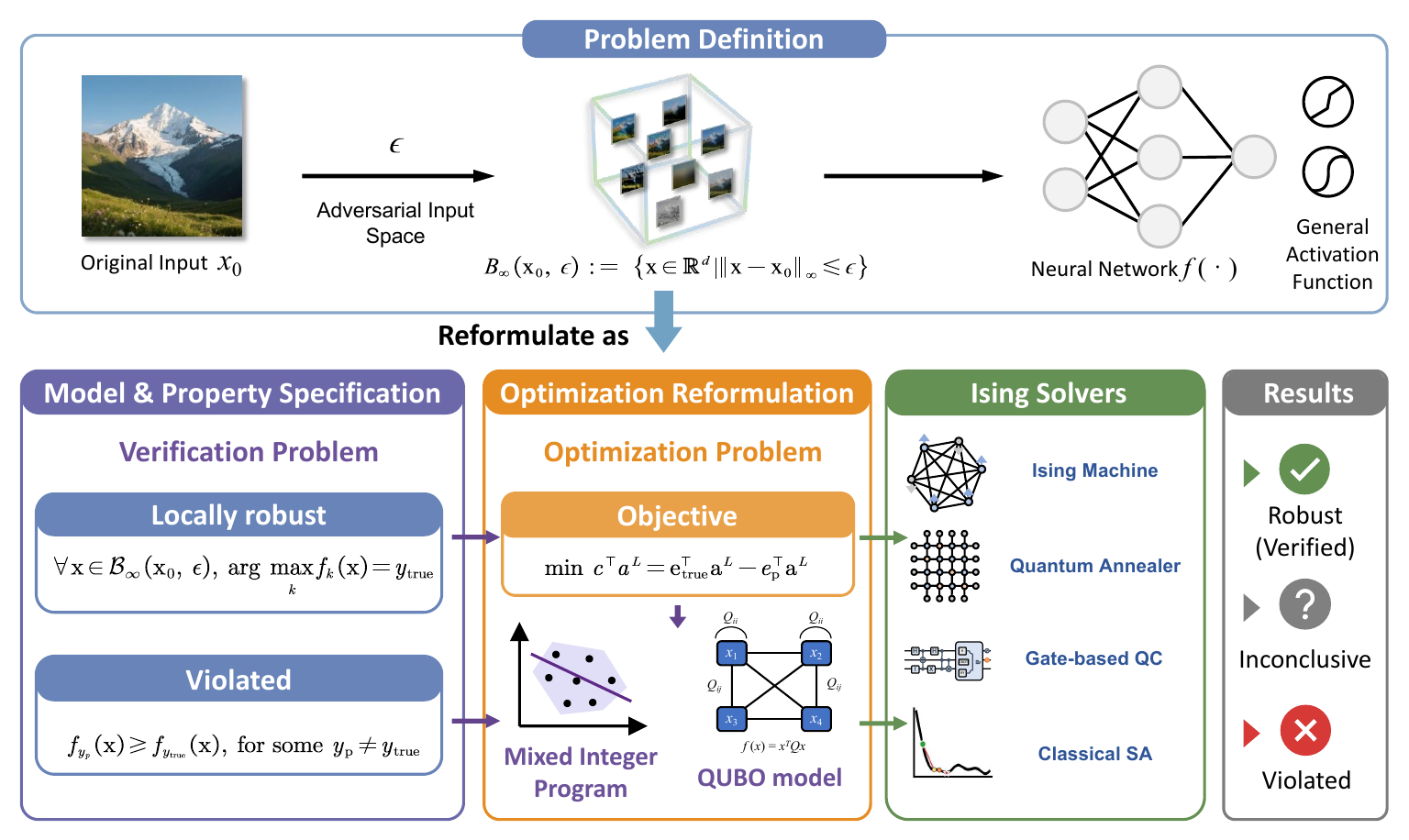} 
    \caption{Overview of the proposed Ising-compatible robustness verification framework. The end-to-end pipeline links network specification, logarithmic encoding, Step-Envelope construction, and hybrid Benders optimization for Ising solvers.} 
    \label{fig:workflow} 
\end{figure*}

Neural networks (NNs) have become foundational across computer vision \cite{krizhevsky2012imagenet, he2016deep}, natural language processing \cite{vaswani2017attention, devlin2019bert}, autonomous systems \cite{bojarski2016end}, and scientific discovery \cite{jumper2021highly}. Despite their empirical success, neural networks remain fundamentally vulnerable to small perturbations in their inputs. In safety-critical applications such as autonomous driving, robotic surgery, and automated medical diagnostics, even minor input alterations can lead to catastrophic failures. For example, as shown in Figure~\ref{fig:traffic_sign}, a vision model trained to recognize traffic signs may fail to detect a stop sign when inconspicuous noise or stickers are added, triggering hazardous vehicle maneuvers \cite{eykholt2018robust}. Such vulnerabilities emphasize a central challenge in trustworthy AI: mathematically certifying robustness against adversarial input variations.

Adversarial perturbations \cite{szegedy2013intriguing, goodfellow2014explaining} expose the fragility of neural decision boundaries. Beyond adversarial training \cite{madry2018towards} and certified defenses \cite{wong2018provable, cohen2019certified}, formal verification provides mathematical guarantees that predictions remain invariant within a bounded perturbation region (e.g., an $\ell_p$ ball). To provide such guarantees, classical optimization techniques, such as Satisfiability Modulo Theories (SMT) \cite{katz2017reluplex}, MILP encodings \cite{tjeng2019evaluating}, reachability analysis \cite{wang2018efficient}, and Branch-and-Bound \cite{bunel2018unified}, have been successfully developed, particularly for piecewise-linear ReLU networks. However, exact verification remains NP-hard \cite{katz2017reluplex}, with the combinatorial number of activation patterns growing rapidly with network size.

Verification beyond ReLU has been studied through linear relaxations such as CROWN \cite{zhang2018crown}, abstract interpretation such as DeepPoly \cite{singh2019deeppoly}, dual optimization \cite{dvijotham2018dual}, and refinement or branch-and-bound methods such as VeriNet \cite{henriksen2020verinet} and GenBaB \cite{shi2025genbab}. These methods provide sound guarantees for nonlinear activations such as sigmoid and tanh, but rely primarily on floating-point bound propagation, activation-specific affine transformers, or repeated LP/nonlinear subproblems. Such representations are effective on classical CPU/GPU systems but do not directly provide a compact encoding for spin-constrained Ising hardware. Naively discretizing their continuous variables and constraints can exhaust the available spin budget. The unresolved challenge addressed here is therefore not merely support for nonlinear activations, but how to represent their sound verification in a refinable, Ising-compatible form with explicit activation assumptions.

Physics-inspired Coherent Ising Machines (CIMs), quantum annealers, and related optical Ising solvers offer specialized hardware for Quadratic Unconstrained Boolean Optimization (QUBO) search \cite{100spincim, 2000nodecim, Marandi2014, Honjo2021, johnson2011quantum, Boixo2014, Goto2019, Goto2021, 20k-SpinIsingChip, Cai2020, DigitalAnnealer, song2023training}. In pioneering work, Franco et al. \cite{franco2022quantum} reformulated robustness certification as a hybrid quantum--classical QUBO procedure. In their subsequent study comparing Benders and Dantzig--Wolfe decompositions, Franco et al. \cite{franco2023milp} carefully evaluated generic Benders decomposition, showing that under general MILP conditions, encoding continuous cut slacks into the master problem can require substantial spin resources. Their foundational study focused on general worst-case scenarios and resource requirements, establishing key baselines for decomposition on quantum hardware. Complementing this perspective, our work investigates the underlying geometric structure of neural verification instances: when verification subproblems admit low-dimensional infeasibility certificates, a single conflict cut excludes an entire exponential family of activation choices. Through a connection to Helly's theorem \cite{helly1923mengen} on convex arrangements, we prove that the contextual conflict rank is geometrically bounded by the input dimension ($\rho \le n_0$). Consequently, Benders iteration counts do not suffer exponential explosion, but satisfy an output-sensitive bound ($N_{\mathrm{fail}} = \mathcal{O}(K_{\mathcal{X}} L P^{\rho+1})$) that yields a polynomial iteration multiplier $\mathcal{O}(P^{n_0+1})$ in network size for fixed input dimension $n_0$. Building upon these advances, our primary contributions are three-fold:
\begin{itemize}
    \item \textbf{Complementary Ising-Compatible Formulations:} We propose the Exact Logarithmic Piecewise Model (\textbf{Log-PW}) and the Asymptotic Step-Envelope Model (\textbf{Step-Env}). Log-PW accommodates arbitrary piecewise activations with an information-theoretically minimal logarithmic spin encoding ($\lceil \log_2 \kappa \rceil$ spins per neuron), unifying Log-PWL (via logarithmic SOS2) for continuous piecewise-linear functions with Log-PWC (via logarithmic SOS1) for piecewise-constant functions. Serving as the exact discrete substrate for Step-Env, Log-PWC enables sound, fine-grained enclosures for general nonlinear activations without exponential spin overhead. Step-Env achieves uniform convergence and strict-margin completeness under continuity, and an explicit linear convergence rate $\mathcal{O}(\Delta)$ under Lipschitz continuity.
    \item \textbf{Output-Sensitive Complexity Bounds for Hybrid Benders Decomposition:} We design a hybrid Benders solver separating combinatorial activation choices from continuous bound propagation. By introducing a Helly-based conflict-compression theory, we prove that the contextual conflict rank is geometrically bounded by the input dimension ($\rho \le n_0$), establishing an output-sensitive iteration bound ($N_{\mathrm{fail}} = \mathcal{O}(K_{\mathcal{X}} L P^{\rho+1})$) that guarantees a polynomial iteration multiplier $\mathcal{O}(P^{n_0+1})$ relative to reachable feasible pattern complexity $K_{\mathcal{X}}$ under fixed input dimension $n_0$. In particular, for shallow networks ($L = \mathcal{O}(1)$) with fixed input dimension ($n_0 = \mathcal{O}(1)$), the total Benders iteration complexity is unconditionally polynomial in network size: $T_{\mathrm{Benders}} = \mathcal{O}(P^{(L+1)n_0+1}) = \mathrm{poly}(P)$.
    \item \textbf{Spin Reduction Techniques and Empirical Validation:} We introduce interval pruning, pruning-induced robustness transfer, and layerwise classical--Ising partitioning to minimize hardware spin requirements. On full-scale vision architectures (MNIST 784-64-32-10), Log-PW (via Log-PWL) halves binary spin requirements (50.0\% reduction) and slashes linear constraints by 56.3\% compared to Multiple-Choice One-Hot baselines while strictly preserving verification decisions. Experiments on Coherent Ising Machines (CIMs) and classical solvers confirm exact certification fidelity for PWL networks, up to a $13.41\times$ cut acceleration on deep architectures, and near-reference accuracy for non-linear networks with significantly reduced spin budgets.
\end{itemize}
An overview of the overall verification workflow is illustrated in Figure~\ref{fig:workflow}. An extended review of related work across classical verification, quantized verification, and Ising/quantum computing paradigms is provided in Appendix~\ref{related_work}.

Table~\ref{tab:model_summary} summarizes the characteristics and theoretical guarantees of our proposed models.

\begin{table}[h]
\centering
\resizebox{\columnwidth}{!}{
\begin{tabular}{lccc}
\toprule
\textbf{Model Formulation} & \textbf{Activation Functions} & \textbf{Sound} & \textbf{Guarantee}\\
\midrule
Exact Log-PW (Log-PWL / Log-PWC) & Piecewise (PWL / PWC) & \checkmark & Exact and complete\\
Step-Envelope (Step-Env via Log-PWC) & General locally Lipschitz & \checkmark & Complete for strict margins as $\Delta\to0$\\
\bottomrule
\end{tabular}
}
\caption{Summary of the proposed verification models in this paper}
\label{tab:model_summary}
\end{table}

\section{Robustness Verification Formulations for Deep Neural Networks}

Formal robustness verification seeks to mathematically guarantee that a neural network's predictions remain consistent under bounded input perturbations. To establish a mathematically rigorous framework, we first formally define a feed-forward deep neural network (DNN) and its verification problem, adapting structures commonly used in formal verification literature \cite{liu2021algorithms}.

\begin{definition}[Feed-forward Deep Neural Network]
\label{def:dnn}
A feed-forward DNN $\mathcal{N} : \mathbb{R}^{n_0} \to \mathbb{R}^{n_L}$ with $L$ layers is defined as a composition of layer functions $\mathcal{N} \triangleq T^L \circ T^{L-1} \circ \dots \circ T^1$. For any input vector $\mathbf{x} \in \mathbb{R}^{n_0}$, the pre-activation vector $\mathbf{z}^l \in \mathbb{R}^{n_l}$ and activation vector $\mathbf{a}^l \in \mathbb{R}^{n_l}$ at layer $l \in \{1, \dots, L\}$ are computed recursively by $\mathbf{z}^l = \mathbf{W}^l \mathbf{a}^{l-1} + \mathbf{b}^l$ and $\mathbf{a}^l = \sigma^l(\mathbf{z}^l)$, where $\mathbf{a}^0 \triangleq \mathbf{x}$, and $\mathbf{W}^l \in \mathbb{R}^{n_l \times n_{l-1}}$ and $\mathbf{b}^l \in \mathbb{R}^{n_l}$ denote the weight matrix and bias vector of layer $l$, respectively, and $\sigma^l(\cdot)$ denotes the coordinate-wise activation function.
\end{definition}

Using this formal structure, the task of verifying local robustness under bounded adversarial input perturbations is equivalent to proving a lower bound on the output difference.

\begin{definition}[Adversarial Robustness Verification]
Given a neural network $\mathcal{N}$, a nominal input $\mathbf{x}_0 \in \mathbb{R}^{n_0}$ classified as class $c_{\text{true}}$, and an $\ell_\infty$ perturbation ball $\mathcal{X} \triangleq \{ \mathbf{x} \mid \|\mathbf{x} - \mathbf{x}_0\|_\infty \le \varepsilon \}$ at radius $\varepsilon > 0$, 
the pairwise classification margin against competing class $p \neq c_{\text{true}}$ over $\mathcal{X}$ is defined as 
\begin{equation}
\label{eq:pairwise_margin}
    m_p(\mathbf{x}) \triangleq \mathbf{e}_{\text{true}}^\top \mathbf{a}^L(\mathbf{x}) - \mathbf{e}_p^\top \mathbf{a}^L(\mathbf{x}).
\end{equation}
The overall classification margin is $m(\mathbf{x}) \triangleq \min_{p \neq c_{\text{true}}} m_p(\mathbf{x})$, where $\mathbf{e}_{\text{true}}, \mathbf{e}_p \in \{0, 1\}^{n_L}$ are standard one-hot indicator vectors.
The network is certified locally robust at $\varepsilon$ iff the ground-truth logit strictly dominates all competing classes $p \neq c_{\text{true}}$ over the perturbation ball, i.e., $\min_{\mathbf{x}\in\mathcal{X}} m(\mathbf{x}) = \min_{p \neq c_{\text{true}}} \min_{\mathbf{x}\in\mathcal{X}} m_p(\mathbf{x}) > 0$. Otherwise, robustness is violated. 
\end{definition}

\subsection{Exact Logarithmic Piecewise Model}\label{piecewise_linear}

For networks employing piecewise-linear (PWL) or piecewise-constant (PWC) activation functions (\eg, ReLU, Heaviside), the activation can be modeled exactly by a partitioning of the pre-activation domain that respects the piecewise property. For a neuron $j$ in layer $l$ with a $\kappa$-piece activation, let grid points $v_0 < v_1 < \dots < v_\kappa$ denote the partition of the pre-activation domain such that the activation is affine or constant on $[v_{i-1}, v_i]$. 
Piecewise refers to either PWL or PWC throughout this section. 

\paragraph{Conventional One-Hot Encoding.} In standard MILP formulations of networks with piecewise activation, the choice of segment at neuron $j$ in layer $l$ is represented by one-hot binary encoding $\bm{\beta}^l_j\in \{0,1\}^\kappa$ satisfying $\sum_{i=1}^\kappa \beta_{j,i}^l = 1$, where $\beta_{j,i}^l = 1$ iff pre-activation $z^l_j \in [v_{i-1}, v_i]$. $\bm{\beta} \in \{0, 1\}^{N_{\text{bin}}}$ collects all binary variables by concatenating $\bm{\beta}^l_j$ for all $l$ and $j$. 

Many formulations can incorporate one-hot segment selection into MILP: (i)~\emph{Standard Big-M Model} 
; (ii)~\emph{Multiple-Choice (MC) Model} 
; and (iii)~\emph{Special Ordered Set (SOS2) Model}. 
Complete mathematical formulations of Big-M, MC and SOS models, and unified matrix constructions are detailed in Appendix~\ref{appendix_one_hot_milp}. Despite all the formal differences, a one-hot formulation still requires $\kappa$ binary variables per neuron. 

\paragraph{State-Optimal Logarithmic Encoding.} A primary bottleneck in scaling Ising-compatible MILP formulations is the number of binary variables, which directly dictates spin requirements. One-hot encoding assigns $\kappa$ binary variables per neuron, leading to a combinatorial explosion of $\kappa \sum_{l=1}^L n_l$ binary variables, $\kappa$ being the number of segments per activation. While PWL activations such as ReLU usually have moderate $\kappa$, PWC activations can have large $\kappa$, as we will see in Section~\ref{Arbitrary}.

To overcome this limitation, we introduce a state-optimal formulation that scales logarithmically with the number of segments, leveraging the MILP framework of Vielma and Nemhauser \cite{vielma2010how}.

Let $K \triangleq \lceil \log_2 \kappa \rceil$. $K$ is the minimal number of binary variables required to represent $\kappa$ segments. 
We assign a unique $K$-bit Binary Reflected Gray Code (BRGC) \cite{vielma2010how} $\mathbf{c}_i\in\{0,1\}^K$ to each segment $i \in \{1, \dots, \kappa\}$. The BRGCs for adjacent segments differ in exactly one bit position. For each bit position $k \in \{1, \dots, K\}$, we define the partition of segment indices
\begin{equation}
\begin{aligned}
    S_k^0 \triangleq \{i \in \{1, \dots, \kappa\} \mid c_{ik} = 0\}, \\ S_k^1 \triangleq \{i \in \{1, \dots, \kappa\} \mid c_{ik} = 1\}.
\end{aligned}
\label{eq:brgc_segment_partition}
\end{equation}
In contrast to $\kappa$ one-hot binary variables, we use $K$ binary code variables $\mathbf{g}^l_j \triangleq [g^l_{j,1}, \dots, g^l_{j,K}]^\top \in \{0, 1\}^K$ per neuron for segment selection. 

Depending on whether the activation is PWL or PWC, the binary code variables $\mathbf{g}^l_j$ govern either Special Ordered Set of Type 2 (SOS2) or Special Ordered Set of Type 1 (SOS1).

\smallskip
\noindent\textit{1) PWL Activations via Logarithmic SOS2 (Log-PWL).}
For continuous piecewise-linear activations (\eg, ReLU, Hardtanh), the activation value interpolates linearly across grid vertices. The pre-activation $z^l_j$ and activation $a^l_j$ are represented as convex combinations of the grid points using continuous weights $\boldsymbol{\lambda}^l_j \triangleq [\lambda^l_{j,0}, \dots, \lambda^l_{j,\kappa}]^\top \ge 0$ satisfying $\sum_{p=0}^\kappa \lambda^l_{j,p} = 1$:
\begin{equation}
z^l_j = \sum_{p=0}^\kappa \lambda^l_{j,p} v_p, \qquad a^l_j = \sum_{p=0}^\kappa \lambda^l_{j,p} \sigma(v_p),
\label{eq:log_pwl_convex}
\end{equation}
where $\boldsymbol{\lambda}^l_j$ must satisfy the SOS2 condition: at most two adjacent weights can be non-zero. The segment partitions $S_k^0$ and $S_k^1$ induce the valid grid-point subsets:
\begin{equation}
\begin{aligned}
    V_k^0 \triangleq \{ v_p \mid [v_{p-1}, v_p] \in S_k^0 \text{ or } [v_p, v_{p+1}] \in S_k^0 \}, \\
    V_k^1 \triangleq \{ v_p \mid [v_{p-1}, v_p] \in S_k^1 \text{ or } [v_p, v_{p+1}] \in S_k^1 \}.
\end{aligned}
\label{eq:induced_vertices}
\end{equation}
The SOS2 adjacency requirement is enforced by $2K$ linear coupling constraints. For all $k = 1, \dots, K$, we require 
\begin{equation}
\sum_{p: v_p \notin V_k^0} \lambda^l_{j,p} \le g^l_{j,k}, \qquad \sum_{p: v_p \notin V_k^1} \lambda^l_{j,p} \le 1 - g^l_{j,k}.
\label{eq:vn_coupling_pwl}
\end{equation}

\smallskip
\noindent\textit{2) PWC Activations via Logarithmic SOS1 (Log-PWC).}
For piecewise-constant activations (\eg, Heaviside), the activation evaluates to discrete values $\gamma_i$ over each interval $[v_{i-1}, v_i]$ without interpolation. We introduce continuous segment weights $\boldsymbol{\mu}^l_j \triangleq [\mu^l_{j,1}, \dots, \mu^l_{j,\kappa}]^\top \ge 0$ normalized such that $\sum_{i=1}^\kappa \mu^l_{j,i} = 1$:
\begin{equation}
\sum_{i=1}^\kappa \mu^l_{j,i} v_{i-1} \le z^l_j \le \sum_{i=1}^\kappa \mu^l_{j,i} v_i, \quad a^l_j = \sum_{i=1}^\kappa \mu^l_{j,i} \gamma_i,
\label{eq:log_pwc_bounds}
\end{equation}
where $\boldsymbol{\mu}^l_j$ must satisfy the SOS1 condition (exactly one weight is non-zero). The SOS1 adjacency requirement is enforced by $2K$ linear coupling constraints. For all $k = 1, \dots, K$, we require 
\begin{equation}
\sum_{i \in S_k^0} \mu^l_{j,i} \le 1 - g^l_{j,k}, \qquad \sum_{i \in S_k^1} \mu^l_{j,i} \le g^l_{j,k}.
\label{eq:vn_coupling_pwc}
\end{equation}

\smallskip

Under this BRGC formulation, the total binary requirement is reduced to $N_{\text{log}} \triangleq \lceil \log_2 \kappa \rceil \sum_{l=1}^L n_l$, achieving the information-theoretic minimum number of binary variables and drastically reducing the spin overhead for Ising solvers for large $\kappa$.

\paragraph{Unified MILP Formulation.} Generally, for any formulation (one-hot or logarithmic, PWL or PWC), the network's continuous states $\mathbf{y} \in \mathbb{R}^{N_{\text{cont}}}$ (comprising perturbed input $\mathbf{x}$, activations $\mathbf{a}^l$, pre-activations $\mathbf{z}^l$, and continuous auxiliary variables) and binary segment choices $\mathbf{w}$ ($\bm{\beta}$ or $\mathbf{g}$) compile into a unified mixed-integer linear system:
\begin{equation}
\begin{aligned}
\mathbf{A}_{\text{eq}} \mathbf{y} &= \mathbf{b}_0 + \mathbf{B}_{\text{eq}} \mathbf{w}, \\
\mathbf{C} \mathbf{y} &\leq \mathbf{d}_0 + \mathbf{D} \mathbf{w}.
\end{aligned}
\label{eq:milp_system}
\end{equation}

\subsection{Asymptotic Step-Envelope Model (Step-Env)}\label{Arbitrary}

\begin{figure}[t]
    \centering
    \includegraphics[width=\linewidth]{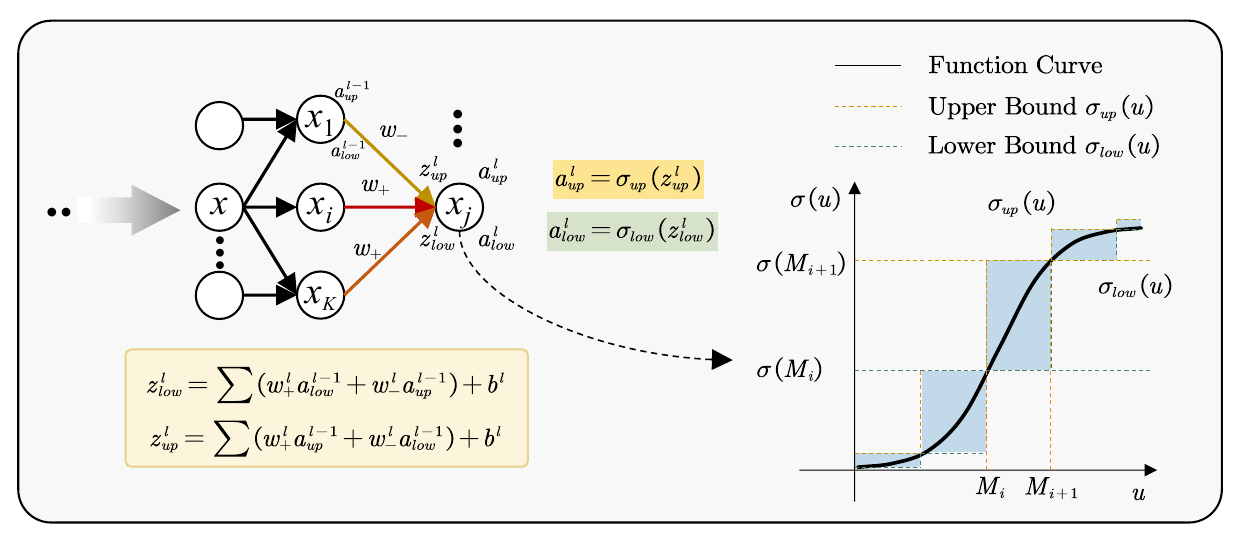} 
    \caption{Step-envelope enclosure of a general activation. Lower and upper step functions approximates the activation with discrete segment-selection states. Refining the partition reduces the activation-envelope error at the cost of Ising-spin budget.}
    \label{fig:example} 
\end{figure}

Existing MILP framework for network verification is limited to PWL or PWC activations. Step-Env generalizes this framework to general activations (\eg, sigmoid, GELU). 
Step-Env approximates activations by piecewise-constant (PWC) functions. By dividing the reachable domain of a general activation function into segments, we can construct lower and upper step envelopes (Figure~\ref{fig:example}). By passing to a sign-aware argument, Step-Env provides upper and lower bounds on the network output for any given input. Under Assumption~\ref{ass:step_env_refinements}, the approximation gap decays linearly in maximal segment width, allowing for asymptotically complete verification. 

\begin{definition}[Step Envelopes]
For a scalar activation $\sigma$ bounded on a compact interval $I$, let $v_0<\cdots<v_\kappa$ be a partition of $I$, $I_i\triangleq[v_{i-1},v_i]$ and define
\begin{equation}
\underline{\gamma}_i=\inf_{u\in I_i}\sigma(u),\qquad
\overline{\gamma}_i=\sup_{u\in I_i}\sigma(u).
\label{eq:step_env_functions}
\end{equation}
If $z\in I_i$, selecting cell $i$ gives $\underline\sigma(z)=\underline\gamma_i\le\sigma(z)\le\overline{\gamma}_i=\overline\sigma(z)$. Exactly one cell is selected; either adjacent cell is admissible at a shared grid point. For non-decreasing $\sigma$, $\underline{\gamma}_i = \sigma(v_{i-1})$ and $\overline{\gamma}_i = \sigma(v_i)$. 
\end{definition}

The step envelopes are neuron-specific. In practice, we construct one by evenly dividing the reachable domain of each neuron into segments. The reachable domain of each neuron is estimated by forward interval propagation, as proposed in Assumption~\ref{ass:step_env_refinements}. 

\paragraph{Bound Propagation and Envelope Formulation.}
We first describe the bound propagation mechanism assuming non-decreasing activations, and then present our exact reduction that extends it to arbitrary non-monotonic activations.

Consider an $L$-layer network with $\underline{\mathbf a}^0=\overline{\mathbf a}^0=\mathbf x\in\mathcal X$. Define the continuous bound vector $\mathbf{y}\in \mathbb{R}^{N_{\text{step}}}$ comprising $\mathbf{x}$, $\underline{\mathbf{a}}^{1}$, $\overline{\mathbf{a}}^{1}$, $\underline{\mathbf{z}}^{1}$, $\overline{\mathbf{z}}^{1}$, $\dots$, $\underline{\mathbf{a}}^{L}$, $\overline{\mathbf{a}}^{L}$, $\underline{\mathbf{z}}^{L}$, $\overline{\mathbf{z}}^{L}$, where $N_{\text{step}} \triangleq n_0 + 4 \sum_{l=1}^{L} n_l$. Pre-activation bounds are propagated through affine layers via interval arithmetic:
\begin{equation}
\begin{aligned}
\underline{z}^l_j &= \sum_{k: w^l_{jk} \ge 0} w^l_{jk} \underline{a}^{l-1}_k + \sum_{k: w^l_{jk} < 0} w^l_{jk} \overline{a}^{l-1}_k + b^l_j, \\
\overline{z}^l_j &= \sum_{k: w^l_{jk} < 0} w^l_{jk} \underline{a}^{l-1}_k + \sum_{k: w^l_{jk} \ge 0} w^l_{jk} \overline{a}^{l-1}_k + b^l_j.
\end{aligned}
\label{eq:step_env_interval_prop}
\end{equation}
For neuron $j$ in layer $l$, the activation $\sigma^l_j$ is approximated by neuron-specific lower and upper step envelopes $\underline{\sigma}^l_j$ and $\overline{\sigma}^l_j$ as defined in~(\ref{eq:step_env_functions}). 
%
The bounds are propagated through activation via
\begin{equation}
\underline{a}^l_j = \underline{\sigma}^l_j(\underline{z}^l_j), \qquad 
\overline{a}^l_j = \overline{\sigma}^l_j(\overline{z}^l_j), 
\label{eq:step_env_bounds}
\end{equation}
Since step-envelopes $\underline{\sigma}^l_j$ and $\overline{\sigma}^l_j$ are PWC, they can be incorporated into MILP framework by either one-hot or logarithmic encoding as discussed in Section~\ref{piecewise_linear}. Let $\mathbf{w}$ collect all binary segment choices. 
The bound propagation process again compiles into a mixed-integer linear system:

\begin{equation}
\begin{aligned}
\mathbf{A}_{\text{eq}} \mathbf{y} &= \mathbf{b}_0 + \mathbf{B}_{\text{eq}} \mathbf{w}, \\
\mathbf{C} \mathbf{y} &\leq \mathbf{d}_0 + \mathbf{D} \mathbf{w}.
\end{aligned}
\label{eq:step_env_system}
\end{equation}

While system~\eqref{eq:step_env_system} shares the algebraic form of~\eqref{eq:milp_system}, note that~\eqref{eq:milp_system} tracks explicit feedforward values, whereas~\eqref{eq:step_env_system} tracks sound lower and upper bounds. 

\paragraph{Exact Monotonic Reduction for General Activations.}
While the interval propagation in~\eqref{eq:step_env_bounds} naturally exploits monotonicity, real-world deep networks widely adopt non-monotonic activations, such as GELU~\cite{hendrycks2016gelu} and SiLU/Swish~\cite{ramachandran2017swish}. We apply the classical Jordan decomposition for functions of bounded variation~\cite{bell2016bv} to locally Lipschitz activations, then incorporate the resulting non-decreasing components into an exact layerwise network reduction:

\begin{lemma}[Jordan Decomposition for Locally Lipschitz Activations]\label{lem:mono_decomp}
Any locally Lipschitz continuous activation function $\sigma:\mathbb{R}\to\mathbb{R}$ can be decomposed as the difference of two non-decreasing functions:
\begin{equation}
\sigma(x) = \sigma_{+}(x) - \sigma_{-}(x),
\label{eq:mono_decomp}
\end{equation}
where $\sigma_{+}$ and $\sigma_{-}$ are non-decreasing and locally Lipschitz.
\end{lemma}

The proof is given in Appendix~\ref{appendix:non_monotonic}.

To integrate this decomposition into the layerwise network equations, let $\sigma^l:\mathbb{R}^{n_l}\to\mathbb{R}^{n_l}$ denote the coordinate-wise activation vector at layer $l$, followed by the linear mapping $(\mathbf{W}^{l+1}, \mathbf{b}^{l+1})$. Let $\sigma^l_{+}, \sigma^l_{-}:\mathbb{R}^{n_l}\to\mathbb{R}^{n_l}$ collect the non-decreasing components of each neuron from Lemma~\ref{lem:mono_decomp}. We define the augmented activation vector $\widehat{\sigma}^l$ and the corresponding expanded weight matrix $\widehat{\mathbf{W}}^{l+1}$:
\begin{equation}
\begin{aligned}
\widehat{\sigma}^l(\mathbf{z}^l) &\triangleq \begin{bmatrix} \sigma^l_{+}(\mathbf{z}^l) \\ \sigma^l_{-}(\mathbf{z}^l) \end{bmatrix} \in \mathbb{R}^{2n_l}, \\
\widehat{\mathbf{W}}^{l+1} &\triangleq \begin{bmatrix} \mathbf{W}^{l+1}, & -\mathbf{W}^{l+1} \end{bmatrix} \in \mathbb{R}^{n_{l+1} \times 2n_l}.
\end{aligned}
\label{eq:mono_reduction_weights}
\end{equation}
By linearity of matrix multiplication, the affine pre-activations of the subsequent layer are preserved identically:
\begin{equation}
\widehat{\mathbf{W}}^{l+1}\widehat{\sigma}^l(\mathbf{z}^l) = \mathbf{W}^{l+1}\sigma^l_{+}(\mathbf{z}^l) - \mathbf{W}^{l+1}\sigma^l_{-}(\mathbf{z}^l) = \mathbf{W}^{l+1}\sigma^l(\mathbf{z}^l).
\end{equation}
Because $\widehat{\sigma}^l$ is coordinate-wise non-decreasing, replacing each non-monotonic layer pair $(\mathbf{W}^{l+1}, \sigma^l)$ with $(\widehat{\mathbf{W}}^{l+1}, \widehat{\sigma}^l)$ reduces general activations to an equivalent non-decreasing formulation. Consequently, the sign-aware interval arithmetic~\eqref{eq:step_env_interval_prop} and step-enclosure bounds~\eqref{eq:step_env_bounds} apply rigorously to arbitrary locally Lipschitz architectures, requiring only doubling the activation variables at non-monotonic layers (Appendix~\ref{appendix:non_monotonic}). 

\paragraph{Theoretical Guarantees.}
\begin{assumption}[Admissible Step-Env Refinements]\label{ass:step_env_refinements}
The input set $\mathcal X$ is nonempty and compact. Each activation $\sigma^l_j$ is locally Lipschitz. The domain of step envelopes cover all propagated pre-activation values for all input. 
\end{assumption}

\begin{lemma}[Soundness of Step-Env Over-Approximation]\label{lem:soundness}
Under Assumption~\ref{ass:step_env_refinements}, every $\mathbf x\in\mathcal X$ admits a feasible Step-Env bound trajectory, and every such trajectory satisfies $\underline{\mathbf a}^l\le\mathbf a^l(\mathbf x)\le\overline{\mathbf a}^l$ at each layer. Hence a positive lower bound on the class margin proves robustness, and a negative upper bound proves violation. 
\end{lemma}

\begin{theorem}[Uniform Convergence of Step-Env Optimal Bounds]\label{thm:convergence}
Let $m_p(\mathbf{x})$ denote the pairwise classification margin for competing class $p \neq c_{\text{true}}$ at input $\mathbf{x}$ as defined in~\eqref{eq:pairwise_margin}: 
\[  m_p(\mathbf{x}) \triangleq  \mathbf{e}_{\text{true}}^\top \mathbf{a}^L(\mathbf{x}) -  \mathbf{e}_p^\top \mathbf{a}^L(\mathbf{x}), \]
and let $\underline{m}_{p,\Delta}(\mathbf{x})$, $\overline{m}_{p,\Delta}(\mathbf{x})$ denote the lower and upper bounds of the margin provided by Step-Env with maximal segment width $\Delta$ across all neurons, respectively:
\begin{equation}
\label{eq:step_env_objs}
    \begin{aligned}
    \underline{m}_{p,\Delta}(\mathbf{x}) \triangleq  \mathbf{e}_{\text{true}}^\top \underline{\mathbf{a}}^L(\mathbf{x}) - \mathbf{e}_p^\top \overline{\mathbf{a}}^L(\mathbf{x}), \\
    \overline{m}_{p,\Delta}(\mathbf{x}) \triangleq  \mathbf{e}_{\text{true}}^\top \overline{\mathbf{a}}^L(\mathbf{x}) - \mathbf{e}_p^\top \underline{\mathbf{a}}^L(\mathbf{x}). 
\end{aligned}
\end{equation}
Under Assumption~\ref{ass:step_env_refinements}, there exists a network-dependent constant $C_{\text{net}} = 2 \sum_{l=1}^L \big(\prod_{r=l+1}^L M^r \|\mathbf{W}^r\|_\infty\big) M^l > 0$ (derived in Appendix~\ref{appendix_step_env_theory}) such that the globally optimized Step-Env bounds satisfy:
\begin{equation}
\begin{aligned}
- C_{\text{net}} \Delta &\leq \min_{\mathbf{x}\in \mathcal{X}}\underline{m}_{p,\Delta}(\mathbf{x}) - \min_{\mathbf{x}\in \mathcal{X}} m_p(\mathbf{x})
 \le 0, \\
0 &\leq \min_{\mathbf{x}\in \mathcal{X}}\overline{m}_{p,\Delta}(\mathbf{x}) - \min_{\mathbf{x}\in \mathcal{X}} m_p(\mathbf{x})
 \le C_{\text{net}} \Delta.
\end{aligned}
\label{eq:optimal_bound_convergence}
\end{equation}
\end{theorem}

Theorem~\ref{thm:convergence} shows that by minimizing the width-$\Delta$ Step-Env bounds~\eqref{eq:step_env_objs}, we can bound the exact minimal margin by an interval $I_{p,\Delta} \triangleq [\min_{\mathbf{x}\in \mathcal{X}}\underline{m}_{p,\Delta}(\mathbf{x}), \min_{\mathbf{x}\in \mathcal{X}}\overline{m}_{p,\Delta}(\mathbf{x})]$ for any competing class $p$. The verification returns inconclusive iff $0\in I_{p,\Delta}$, but the width of $I_{p,\Delta}$ is uniformly bounded by $2C_{\text{net}}\Delta$. Specifically, we have the following corollary. 
\begin{corollary}[Asymptotic Completeness for Pairwise Classification Margins]\label{cor:asymptotic_complete}
Under Theorem~\ref{thm:convergence}, let $m_p^\star \triangleq \min_{\mathbf{x}\in\mathcal{X}} m_p(\mathbf{x})$ denote the exact robust margin for competing class $p \neq c_{\text{true}}$. If $m_p^\star \neq 0$, any Step-Env partition with maximal segment width
\begin{equation}
\Delta < \frac{|m_p^\star|}{2C_{\text{net}}}
\end{equation}
guarantees a conclusive verification. Because the lower bound and upper bound are either both negative or both positive. 
\end{corollary}

Proofs and supplementary reduction remarks appear in Appendix~\ref{appendix_step_env_theory}. These guarantees concern the mathematical formulation; numerical drift and heuristic sub-optimality can occur in practical implementations. 

\paragraph{Activation Scope.}
Our framework accommodates any locally Lipschitz activation function. This scope encompasses virtually all standard deep learning activations: PWL activations (\eg, ReLU, Leaky ReLU, PReLU, Hardtanh), smooth sigmoidal activations (\eg, Sigmoid, Tanh, Softplus), exponential variants (\eg, ELU, CELU), and non-monotone activations widely adopted in modern architectures, such as GELU~\cite{hendrycks2016gelu} and SiLU/Swish~\cite{ramachandran2017swish}. Under Assumption~\ref{ass:step_env_refinements}, the Step-Env formulation guarantees sound enclosures, uniform convergence of optimal bounds at an explicit linear rate $\mathcal{O}(\Delta)$ (Theorem~\ref{thm:convergence}), and asymptotic completeness (Corollary~\ref{cor:asymptotic_complete}). The behavior of other activation types is discussed in Appendix~\ref{subsec:step_env_scope}.

\paragraph{Application and More Advantages of Logarithmic Encoding.}
Applying logarithmic encoding to Step-Env provides two key benefits: (i)~\emph{Exponential Spin Reduction}: replacing linear one-hot indicators with binary code variables $\mathbf{g} \in \{0, 1\}^{\lceil \log_2 \kappa \rceil}$ compresses spin complexity from $\mathcal{O}(\kappa)$ to $\mathcal{O}(\log_2 \kappa)$ per neuron (e.g., from 64 to 10 spins at $\kappa=32$); (ii)~\emph{Scalable Fine-Grained Certification}: it permits fine-grained step envelopes (with $\kappa \ge 16$ segments) to suppress envelope approximation error without triggering spin budget explosion. 

\section{Ising Optimization Mappings and Hybrid Algorithms}\label{sec:quantum_opt}

To solve the verification formulations on Ising hardware, one direct approach is constructing a monolithic QUBO model by discretizing continuous state variables and inequality slacks into binary spins via quadratic penalties (detailed in Appendix~\ref{appendix_qubo}). However, monolithic encoding suffers from severe spin explosion and penalty sensitivity. To overcome these bottlenecks, we propose a hybrid Benders decomposition solver that cleanly separates discrete decision states from continuous verification bounds.

\subsection{Hybrid Benders Decomposition Framework}\label{subsec:benders}

In this hybrid architecture, the discrete \emph{master problem}, formulated as a QUBO and executed on high-speed Ising hardware, searches over the combinatorial space of binary activation decisions $\mathbf{w} \in \{0, 1\}^P$. Given a candidate activation pattern $\mathbf{w}$, the continuous \emph{subproblem} $\mathrm{SP}(\mathbf{w})$ is solved classically via linear programming to compute verification bounds over state variables $\mathbf{y} \in \mathbb{R}^{N_{\text{cont}}}$:
\begin{equation}\label{eq:benders_subproblem}
\begin{aligned}
\mathrm{SP}(\mathbf{w}): \quad \min_{\mathbf{y} \in \mathbb{R}^{N_{\text{cont}}}} \quad & \tilde{\mathbf{c}}^\top \mathbf{y} \\
\text{s.t.} \quad & \mathbf{A}_{\text{eq}}\mathbf{y} = \mathbf{b}_0 + \mathbf{B}_{\text{eq}}\mathbf{w}, \quad \mathbf{C}\mathbf{y} \le \mathbf{d}_0 + \mathbf{D}\mathbf{w},
\end{aligned}
\end{equation}
where $\mathbf{y} \in \mathbb{R}^{N_{\text{cont}}}$ collects input perturbations and layerwise continuous bounds, and $\tilde{\mathbf{c}}^\top \mathbf{y}$ extracts the pairwise classification margin $m_p(\mathbf{x}) = \mathbf{e}_{\text{true}}^\top \mathbf{a}^L - \mathbf{e}_p^\top \mathbf{a}^L$. If $\mathrm{SP}(\mathbf{w})$ is feasible and bounded, solving its LP dual yields an optimality cut $\theta \ge \bm{\pi}^\top(\mathbf{b}_0 + \mathbf{B}_{\text{eq}}\mathbf{w}) - \bm{\lambda}^\top(\mathbf{d}_0 + \mathbf{D}\mathbf{w})$; if $\mathrm{SP}(\mathbf{w})$ is infeasible, an extreme dual direction $(\hat{\bm{\pi}}, \hat{\bm{\lambda}})$ yields a Farkas feasibility cut $\hat{\bm{\pi}}^\top(\mathbf{b}_0 + \mathbf{B}_{\text{eq}}\mathbf{w}) + \hat{\bm{\lambda}}^\top(\mathbf{d}_0 + \mathbf{D}\mathbf{w}) \ge 0$, progressively constraining the master search space (Appendix~\ref{appendix_benders}).

\subsection{Combinatorial Complexity via Contextual Helly Compression}\label{subsec:benders_complexity}

In classical and quantum Benders decomposition \cite{franco2023milp}, generic formulations must account for worst-case instances where a large family of cuts is required, each introducing encoded slack variables into the master problem. Although such general bounds provide important conservative guarantees across arbitrary mixed-integer programs, neural network verification problems exhibit rich geometric structure. In particular, whereas the combinatorial domain of potential activation patterns is exponentially large, infeasible candidate assignments admit compact, context-conditioned certificates characterized by Helly's Theorem. Consequently, a single short cut eliminates an entire exponential family of invalid assignments. This establishes that for bounded contextual conflict rank, the number of infeasible Benders discoveries is polynomially bounded relative to the number of reachable feasible activation patterns ($N_{\mathrm{fail}} \le K_{\mathcal{X}} \cdot \mathrm{poly}(P, L)$), providing an output-sensitive complexity bound.

\paragraph{Prefix Regions and Exact Reachable Contexts.}
Let $\mathcal{X} \subseteq \mathbb{R}^{n_0}$ be a non-empty, convex input perturbation set (e.g., an $\ell_\infty$ ball). For depth $t \in \{0, 1, \dots, L\}$, an activation prefix is denoted by $p = (\mathbf{w}^1, \dots, \mathbf{w}^t)$, where $p = \varnothing$ for $t=0$.

\begin{definition}[Consistent Region and Exact Reachable Prefixes]\label{def:reachable_prefix}
For any prefix $p$ of depth $t$, its consistent input region is defined as:
\begin{equation}
R_p \triangleq \{ \mathbf{x} \in \mathcal{X} \mid \mathbf{x} \text{ is consistent with activation states in } p \}.
\end{equation}
The set of exact depth-$t$ reachable prefixes is defined as $\mathcal{P}_t \triangleq \{ p \mid R_p \neq \varnothing \}$, with $\mathcal{P}_0 = \{\varnothing\}$. Let $\mathcal{A}_{\mathcal{X}} \triangleq \mathcal{P}_L$ denote the set of complete reachable activation patterns in $\mathcal{X}$, and let $K_{\mathcal{X}} \triangleq |\mathcal{A}_{\mathcal{X}}| \ge 1$. For any candidate $\mathbf{w}$ and depth $t$, we write $R^t(\mathbf{w}) \triangleq R_{(\mathbf{w}^1, \dots, \mathbf{w}^t)}$, and denote the full-depth consistent region by $R_{\mathbf{w}} \triangleq R^L(\mathbf{w})$.
\end{definition}

By convexity of $\mathcal{X}$ and piecewise-affinity of activations, $R_p$ is a convex set in $\mathbb{R}^{n_0}$ for general piecewise (PW) networks, encompassing both PWL and PWC activations (Appendix~\ref{appendix_benders_proofs}). For PWL activations, fixing segment choices fixes affine linear branches; for PWC activations, intermediate activations evaluate to discrete constants and $R_p$ degenerates to a single point. Crucially, every exact reachable prefix $p \in \mathcal{P}_t$ can be extended to at least one complete feasible pattern $\mathbf{w} = [p,*] \in \mathcal{A}_{\mathcal{X}}$ via feedforward forward-pass execution, inducing a surjective projection $\pi_t : \mathcal{A}_{\mathcal{X}} \to \mathcal{P}_t$. This yields the fundamental counting bound:
\begin{equation}\label{eq:prefix_count_bound}
|\mathcal{P}_t| \le |\mathcal{A}_{\mathcal{X}}| = K_{\mathcal{X}}, \quad \forall t \in \{0, 1, \dots, L\}.
\end{equation}

\paragraph{Earliest Infeasible Depth and Reachable Parent Localization.}
When the Benders master problem generates an infeasible candidate $\mathbf{w} \notin \mathcal{A}_{\mathcal{X}}$ (meaning $R^L(\mathbf{w}) = \varnothing$), the sequence of nested prefix regions $\mathcal{X} = R^0(\mathbf{w}) \supseteq R^1(\mathbf{w}) \supseteq \dots \supseteq R^L(\mathbf{w}) = \varnothing$ transitions from non-empty to empty.

\begin{definition}[Earliest Infeasible Depth]\label{def:earliest_depth}
For any infeasible assignment $\mathbf{w} \notin \mathcal{A}_{\mathcal{X}}$, its earliest infeasible depth $t^\star(\mathbf{w})$ is defined as:
\begin{equation}\label{eq:t_star_def}
t^\star(\mathbf{w}) \triangleq \min \{ t \in \{1, \dots, L\} \mid R^t(\mathbf{w}) = \varnothing \}.
\end{equation}
\end{definition}

\begin{lemma}[Maximal Feasible Parent Context]\label{lem:parent_feasibility}
For any infeasible candidate $\mathbf{w}$, its parent prefix $p^\star \triangleq \mathbf{w}_{<t^\star} = (\mathbf{w}^1, \dots, \mathbf{w}^{t^\star-1})$ is an exact reachable prefix ($R_{p^\star} \neq \varnothing$), whereas the current-layer extension $(p^\star, \mathbf{w}^{t^\star})$ is infeasible ($R_{p^\star, \mathbf{w}^{t^\star}} = \varnothing$).
\end{lemma}

Lemma~\ref{lem:parent_feasibility} guarantees that every failure is localized to a maximal feasible parent context $p^\star \in \mathcal{P}_{t^\star-1}$, allowing failure events to be charged directly to reachable parent prefixes.

\paragraph{Contextual Helly Compression.}
Within a fixed reachable parent context $p \in \mathcal{P}_{l-1}$, pre-activations at layer $l$ are affine transformations of input $\mathbf{x} \in R_p$, given by $z_j^l(\mathbf{x}) = (\bm{\alpha}_{p,j}^l)^\top \mathbf{x} + \gamma_{p,j}^l$. Let $V_p$ be the direction space of the affine hull $\mathrm{aff}(R_p)$. Projecting the switching vectors $\{\bm{\alpha}_{p,j}^l\}_{j=1}^{q_p^l}$ onto $V_p$ defines the effective rank $\rho_p^l \triangleq \mathrm{rank} \{ \Pi_{V_p} \bm{\alpha}_{p,j}^l \}_{j=1}^{q_p^l}$. Crucially, because $V_p$ is a linear subspace of the input perturbation domain $\mathbb{R}^{n_0}$, its dimension satisfies $\dim(V_p) \le n_0$. This geometrically guarantees that the contextual conflict rank is upper-bounded by the input dimension:
\begin{equation}\label{eq:rho_input_bound}
\rho_p^l \le \dim(V_p) \le n_0, \quad \forall p, l.
\end{equation}

\begin{lemma}[Contextual Helly Compression]\label{lem:helly_compression}
Let $p \in \mathcal{P}_{l-1}$ be a reachable parent prefix ($R_p \neq \varnothing$), and let $s = (s_1, \dots, s_q) \in \{0, 1\}^q$ be a current-layer assignment such that $R_p \cap \bigcap_{i=1}^q H_i^{s_i} = \varnothing$, where $H_i^{s_i}$ are the current-layer activation halfspaces. Then there exists a literal index subset $S \subseteq \{1, \dots, q\}$ of size
\begin{equation}
|S| \le \rho_p^l + 1
\end{equation}
such that $R_p \cap \bigcap_{i \in S} H_i^{s_i} = \varnothing$.
\end{lemma}

\begin{proof}[Proof Sketch]
Projecting onto the span of the projected switching vectors reduces feasibility to dimension $\rho_p^l$. Helly's theorem then gives a conflict involving at most $\rho_p^l+1$ literals; see Appendix~\ref{appendix_benders_proofs}.
\end{proof}

\paragraph{Context-Guarded No-Good Cuts.}
Because linear coefficients at layer $l$ depend on upstream activation context $p$, a learned certificate $S$ is conditionally valid under $p$. To enforce this cut globally without invalidating feasible patterns under other contexts, we introduce the parent mismatch penalty function $D_p(\mathbf{w}) \triangleq \sum_{k: p_k=1} (1 - w_k) + \sum_{k: p_k=0} w_k$. The \emph{context-guarded no-good cut} is formulated as:
\begin{equation}\label{eq:guarded_cut}
\sum_{i \in S} L_i(w_i) \le |S| - 1 + D_p(\mathbf{w}),
\end{equation}
where $L_i(w_i) = w_i$ if $s_i = 1$ and $1 - w_i$ if $s_i = 0$. When $\mathbf{w}$ matches context $p$, $D_p(\mathbf{w}) = 0$ and the cut enforces $\sum_{i \in S} L_i(w_i) \le |S| - 1$, eliminating certificate $S$; when context mismatches, $D_p(\mathbf{w}) \ge 1$ and the inequality holds redundantly, maintaining global soundness.

\paragraph{Generating Contextual Helly Cuts.}
Algorithm~\ref{alg:helly_cut_generation} finds the first infeasible prefix and removes redundant current-layer literals. The resulting deletion-minimal core satisfies $|S|\le\rho_p^t+1\le n_0+1$ by Lemma~\ref{lem:helly_compression}. Prefix localization and core extraction require at most $t+q^t\le L+P$ LP feasibility queries over $\mathbb R^{n_0}$.

\begin{algorithm}[tb]
\caption{Contextual Helly Cut Generation}
\label{alg:helly_cut_generation}
\begin{algorithmic}[1]\small
\Require Network $\mathcal N$, input set $\mathcal X$, candidate $\mathbf{w}$.
\Ensure Feasibility status and, if infeasible, a guarded cut.
\State Test prefixes in order to find the first $t$ with $R^t(\mathbf{w})=\varnothing$.
\State If no such $t$ exists, \Return \textsc{Feasible}.
\State $p\gets\mathbf{w}_{<t}$; $S\gets\{1,\ldots,q^t\}$.
\For{$j=1,\ldots,q^t$}
    \State Remove $j$ from $S$ if $R_p\cap\bigcap_{k\in S\setminus\{j\}}H_k^{w_k^t}=\varnothing$.
\EndFor
\State \Return \textsc{Infeasible} and cut~\eqref{eq:guarded_cut} for $(p,S)$.
\end{algorithmic}
\end{algorithm}

\paragraph{Main Complexity Bound.}
Combining reachable-prefix charging with contextual Helly compression yields our main theoretical result:

\begin{theorem}[Output-Sensitive Benders Iteration Bound]\label{thm:benders_complexity}
Let $\mathcal{X} \subseteq \mathbb{R}^{n_0}$ be non-empty and convex. Consider a hybrid Benders verifier that, for every infeasible candidate $\mathbf{w} \notin \mathcal{A}_{\mathcal{X}}$, identifies its earliest infeasible depth $t^\star$, extracts a contextual Helly certificate $S$ ($|S| \le \rho_p^{t^\star} + 1$) within maximal feasible parent context $p = \mathbf{w}_{<t^\star} \in \mathcal{P}_{t^\star-1}$, and permanently adds context-guarded cut~\eqref{eq:guarded_cut}. The total number of distinct infeasibility discoveries $N_{\mathrm{fail}}$ satisfies:
\begin{equation}\label{eq:n_fail_exact_sum}
\begin{aligned}
N_{\mathrm{fail}} &\le \sum_{l=1}^{L} \sum_{p \in \mathcal{P}_{l-1}} \sum_{r=1}^{\rho_p^l+1} 2^r \binom{q_p^l}{r} \\
&\le K_{\mathcal{X}} \sum_{l=1}^{L} \max_{p \in \mathcal{P}_{l-1}} \sum_{r=1}^{\rho_p^l+1} 2^r \binom{q_p^l}{r}.
\end{aligned}
\end{equation}
If there exist uniform bounds $q_p^l \le P$ and $\rho_p^l \le \rho$ across all contexts and layers, then for constant $\rho$:
\begin{equation}\label{eq:n_fail_poly_bound}
N_{\mathrm{fail}} \le K_{\mathcal{X}} L \sum_{r=1}^{\rho+1} 2^r \binom{P}{r} = \mathcal{O}\left( K_{\mathcal{X}} L P^{\rho+1} \right).
\end{equation}
\end{theorem}

\begin{remark}[Input-Dimension Bound on Conflict Rank and Fixed Input Dimension Scaling]\label{rem:input_dim_rank}
By~\eqref{eq:rho_input_bound}, the contextual conflict rank cannot exceed the input dimension: $\rho_p^l \le \dim(V_p) \le n_0$. Consequently, for any network architecture operating on a fixed input dimension $n_0 = \mathcal{O}(1)$ (e.g., fixed-dimensional tabular features, low-dimensional sensor signals, or patch-based vision modules), the conflict rank is uniformly bounded by $\rho \le n_0 = \mathcal{O}(1)$. Theorem~\ref{thm:benders_complexity} then guarantees $N_{\mathrm{fail}} \le K_{\mathcal{X}} L \sum_{r=1}^{n_0+1} 2^r \binom{P}{r} = \mathcal{O}(K_{\mathcal{X}} L P^{n_0+1})$, proving that for fixed input dimension $n_0$, the ratio of infeasible iterations to reachable patterns, $N_{\mathrm{fail}} / K_{\mathcal{X}}$, is strictly bounded by a polynomial in $P$.
\end{remark}

\begin{remark}[Exact Master Solver vs. Physical Ising Heuristics]\label{rem:hardware_qualification}
Theorem~\ref{thm:benders_complexity} establishes an output-sensitive iteration bound for any exact master solver (or any hybrid solver equipped with a classical cut-pool filter that rejects repeated candidate configurations). On physical Ising hardware (such as CIMs or quantum annealers), cuts are mapped to quadratic penalties; exact exclusion is maintained either by calibrating penalty multipliers to exceed the violation threshold or by incorporating classical candidate-rejection caching.
\end{remark}

\paragraph{Feasible Activation Geometry and Overall Iteration Bound.}
To characterize the total Benders iterations $T_{\mathrm{Benders}} \le K_{\mathcal{X}} + N_{\mathrm{fail}}$, we recall foundational results on the geometry of deep neural networks:

\begin{lemma}[Complexity of Feasible Activation Patterns \cite{montufar2014number, serra2018bounding, hanin2019deep}]\label{lem:feasible_geometry}
For a feed-forward neural network with input dimension $n_0$, depth $L$, layer width $n$, and $P$ total neurons:
\begin{enumerate}
    \item[(i)] The maximum number of linear regions in the input domain scales as $\mathcal{O}\big((n/n_0)^{(L-1)n_0} n^{n_0}\big)$ in the worst case \cite{montufar2014number, serra2018bounding};
    \item[(ii)] For fixed depth $L$ and input dimension $n_0$, the number of reachable activation patterns is polynomially bounded by width $n$;
    \item[(iii)] The average number of active linear regions at initialization is bounded by $\mathcal{O}(P^{n_0})$, independent of depth $L$ \cite{hanin2019deep}.
\end{enumerate}
\end{lemma}

Crucially, Theorem~\ref{thm:benders_complexity} establishes an \emph{output-sensitive complexity bound}: $N_{\mathrm{fail}} \le K_{\mathcal{X}} \cdot \mathrm{poly}(P, L) = \mathcal{O}(K_{\mathcal{X}} L P^{\rho+1})$. For bounded contextual conflict rank $\rho = \mathcal{O}(1)$, the number of infeasible Benders discoveries is polynomially bounded relative to the number of reachable feasible activation patterns $K_{\mathcal{X}}$. Combining this with Lemma~\ref{lem:feasible_geometry}, the total Benders iteration count satisfies:
\begin{equation}\label{eq:total_benders_bound}
T_{\mathrm{Benders}} \le K_{\mathcal{X}} + \mathcal{O}\left( K_{\mathcal{X}} L P^{\rho+1} \right).
\end{equation}
Rather than asserting an unconditional polynomial-time runtime for arbitrary deep architectures (where depth $L$ may grow with $P$), this result provides a rigorous instance-dependent guarantee: infeasible combinatorial exploration is tightly governed by the reachable feasible pattern complexity $K_{\mathcal{X}}$. In practice, local perturbation regions $\mathcal{X}$ and interval bound propagation (IBP) substantially constrain $K_{\mathcal{X}} \ll 2^P$, explaining why the hybrid Benders solver scales efficiently without suffering exponential cut explosion.

\paragraph{Unconditional Polynomial Complexity for Shallow Networks.}
Importantly, when restricted to the widely studied class of \emph{shallow networks with fixed input dimension}, the output-sensitive iteration bound collapses to an unconditional polynomial complexity guarantee:

\begin{corollary}[Unconditional Polynomial Iteration Complexity for Shallow Networks with Fixed Input Dimension]\label{cor:shallow_poly}
Suppose the neural network is shallow with constant depth $L = \mathcal{O}(1)$ and operates on a fixed input dimension $n_0 = \mathcal{O}(1)$, with total hidden neuron count $P$. Then:
\begin{enumerate}
    \item[(i)] The number of reachable activation patterns is polynomially bounded: $K_{\mathcal{X}} = \mathcal{O}\big( P^{L \cdot n_0} \big)$;
    \item[(ii)] The contextual conflict rank is uniformly bounded by the input dimension: $\rho \le n_0 = \mathcal{O}(1)$;
    \item[(iii)] The total Benders iteration count $T_{\mathrm{Benders}}$ (and continuous LP oracle query complexity) is unconditionally polynomial in total neuron count $P$:
    \begin{equation}\label{eq:shallow_poly_bound}
    \begin{aligned}
    T_{\mathrm{Benders}} &\le K_{\mathcal{X}} + N_{\mathrm{fail}} \\
    &= \mathcal{O}\left( L \cdot P^{(L+1)n_0 + 1} \right) = \mathrm{poly}(P).
    \end{aligned}
    \end{equation}
    In particular, for single-hidden-layer networks ($L=1$), the total iteration complexity satisfies $T_{\mathrm{Benders}} = \mathcal{O}\left( P^{2n_0 + 1} \right)$.
\end{enumerate}
\end{corollary}

\begin{proof}[Proof Sketch]
Combine $K_{\mathcal X}=\mathcal O(P^{Ln_0})$ from Lemma~\ref{lem:feasible_geometry}, $\rho\le n_0$, and Theorem~\ref{thm:benders_complexity} in $T_{\mathrm{Benders}}\le K_{\mathcal X}+N_{\mathrm{fail}}$; see Appendix~\ref{appendix_proof_shallow_poly}.
\end{proof}

\section{Spin Complexity Reduction Techniques}

To verify large neural networks under physical Ising spin limitations, we introduce three complementary complexity reduction techniques: (i)~interval-based state and segment pruning, (ii)~pruning-induced robustness transfer, and (iii)~layerwise classical--Ising partitioning.

\subsection{Interval-Based State and Segment Pruning}\label{sec:interval_pruning}\label{sec:interval_details}

Before compiling the verification problem into an Ising master problem or QUBO formulation, the combinatorial configuration space can be substantially reduced by pre-computing neuron reachability bounds. Given the input perturbation set $\mathcal{X} \triangleq \{\mathbf{x} \mid \|\mathbf{x} - \mathbf{x}_0\|_\infty \le \varepsilon\}$, pre-activation bounds $[\underline{z}^l_j, \overline{z}^l_j]$ for each neuron $j$ at layer $l \in \{1, \dots, L\}$ are computed via forward interval arithmetic propagation through the affine layers (as defined in~\eqref{eq:step_env_interval_prop}). 

These pre-activation bounds are evaluated against the partition grid points $v_0 < v_1 < \dots < v_\kappa$ of the piecewise activation function:
\begin{enumerate}
    \item \textbf{Infeasible Segment Elimination:} A segment $i \in \{1, \dots, \kappa\}$ with pre-activation interval $[v_{i-1}, v_i]$ is provably unreachable under perturbation $\mathcal{X}$ if $[\underline{z}^l_j, \overline{z}^l_j] \cap [v_{i-1}, v_i] = \varnothing$ (i.e., if $\overline{z}^l_j < v_{i-1}$ or $\underline{z}^l_j > v_i$). Let $\mathcal{I}_j^l \triangleq \{i \in \{1, \dots, \kappa\} \mid [v_{i-1}, v_i] \cap [\underline{z}^l_j, \overline{z}^l_j] \neq \varnothing\}$ denote the set of surviving reachable segments.
    
    \item \textbf{Pruning Under One-Hot Encoding:} In one-hot formulations where segment choices are represented by binary indicator variables $\mathbf{w}_j^l = \bm{\beta}_j^l \in \{0, 1\}^\kappa$ (Section~\ref{piecewise_linear}), any unreachable segment $i \notin \mathcal{I}_j^l$ is pruned by permanently setting $\beta_{j,i}^l = 0$, removing the variable and its associated constraints. If a neuron is completely stabilized to a single segment $i^*$ ($|\mathcal{I}_j^l| = 1$), setting $\beta_{j,i^*}^l = 1$ eliminates all $\kappa$ binary variables for that neuron.
    
    \item \textbf{Pruning Under Logarithmic Encoding:} In the state-optimal logarithmic model (Log-PW) where segment selection is encoded by binary code variables $\mathbf{w}_j^l = \mathbf{g}_j^l \in \{0, 1\}^K$ with $K = \lceil \log_2 \kappa \rceil$ bits, unreachable segments restrict the admissible Gray code configurations: any code $\mathbf{g}_j^l = \mathbf{c}_i$ with $i \notin \mathcal{I}_j^l$ is eliminated. Crucially, if for any bit position $k \in \{1, \dots, K\}$, all reachable segments in $\mathcal{I}_j^l$ share the same bit value $b \in \{0, 1\}$ (i.e., $c_{ik} = b$ for all $i \in \mathcal{I}_j^l$), the decision variable is constant and can be permanently fixed: $g_{j,k}^l = b$. This directly eliminates the $k$-th spin from the physical Ising search space. In particular, for 2-segment activations such as ReLU ($\kappa=2$), if $[\underline{z}^l_j, \overline{z}^l_j]$ lies entirely in the inactive or active domain, the single binary decision spin $w_j^l \in \{0, 1\}$ is fixed, completely removing the neuron from the master search problem.
\end{enumerate}
By eliminating unreachable segments and fixing stable bits, the dimension of the binary decision vector $\mathbf{w}$ in the unified mixed-integer system~\eqref{eq:milp_system} (or~\eqref{eq:step_env_system}) is compressed from $P$ to the active subset $P_{\text{active}} \ll P$. This contracts the configuration space $\{0, 1\}^{P_{\text{active}}}$ explored by the Benders master problem or QUBO solver, significantly accelerating verification without introducing any relaxation error.

\subsection{Pruning-Induced Robustness Transfer}\label{sec:prune_bit}

Network pruning removes redundant parameters, narrowing variable ranges and reducing the Ising spin budget. We establish a theoretical framework that transfers formal verification guarantees from a simplified pruned model $\mathcal{N}_{\text{pruned}}$ to the original network $\mathcal{N}$. Let $f \equiv \mathcal{N}$ denote the original network and $g \equiv \mathcal{N}_{\text{pruned}}$ denote the pruned model.

Let $\mathbf{r}(\mathbf{x}) \triangleq f(\mathbf{x}) - g(\mathbf{x})$ denote the pruning residual. Assuming a uniform residual bound $\tau \ge 0$ such that $\|\mathbf{r}(\mathbf{x}+\bm{\delta})\|_\infty \le \tau$ for all $\|\bm{\delta}\|_\infty \le \varepsilon$, we first establish a margin stability property:

\begin{lemma}[Margin Stability]\label{lem:margin_stability}
For any logit vectors $\mathbf{a}, \mathbf{b} \in \mathbb{R}^{n_L}$ and target class $c_{\text{true}}$, classification margin $m(\mathbf{a}) \triangleq \mathbf{e}_{\text{true}}^\top \mathbf{a} - \max_{p \neq c_{\text{true}}} \mathbf{e}_{p}^\top \mathbf{a}$ satisfies the Lipschitz continuity bound:
\begin{equation}\label{eq:margin_2inf}
|m(\mathbf{a}) - m(\mathbf{b})| \le 2 \|\mathbf{a} - \mathbf{b}\|_\infty.
\end{equation}
\end{lemma}

Applying Lemma~\ref{lem:margin_stability} to the residual $\mathbf{r}(\mathbf{x}+\bm{\delta})$ yields the robustness transfer theorem. For any model $h \in \{f, g\}$, let $\Phi_h(\mathbf{x};\varepsilon) \triangleq \min_{\|\bm{\delta}\|_\infty \le \varepsilon} m\big(h(\mathbf{x}+\bm{\delta})\big)$ denote its worst-case robust classification margin over the perturbation ball around $\mathbf{x}$.

\begin{theorem}[Pruning-Induced Robustness Transfer]\label{thm:transfer}
Fix nominal sample $\mathbf{x}$ with label $c_{\text{true}}$ and perturbation radius $\varepsilon > 0$. Given computable bounds $L_g(\mathbf{x};\varepsilon) \le \Phi_g(\mathbf{x};\varepsilon) \le U_g(\mathbf{x};\varepsilon)$ on pruned model $g$, the original model's margin satisfies:
\begin{equation}\label{eq:transfer}
L_g(\mathbf{x};\varepsilon)-2\tau \;\le\; \Phi_f(\mathbf{x};\varepsilon) \;\le\; U_g(\mathbf{x};\varepsilon)+2\tau.
\end{equation}
Hence, if $L_g(\mathbf{x};\varepsilon) > 2\tau$, the original network $f$ is certified robust at $\mathbf{x}$; if $U_g(\mathbf{x};\varepsilon) \le -2\tau$, it is provably non-robust.
\end{theorem}

Dataset-level certification bounds $\underline{\mathrm{CA}}_f(\varepsilon), \overline{\mathrm{CA}}_f(\varepsilon)$, margin stability proofs, and closed-form derivations of $\tau$ are detailed in Appendix~\ref{appendix_pruning_details}.

\subsection{Layerwise Partitioning for Scalable Verification}\label{subsec:scalable-partition-cim}

To verify deep architectures beyond raw hardware spin limits, we propose a layerwise partitioning scheme. We split an $L$-layer network $\mathcal{N} = T^L \circ \cdots \circ T^1$ at a cut index $l_{\text{cut}} \in \{1, \dots, L-1\}$ into a prefix $\mathsf{P} \triangleq T^{l_{\text{cut}}} \circ \dots \circ T^1$ and a suffix $\mathsf{S} \triangleq T^L \circ \dots \circ T^{l_{\text{cut}}+1}$, so that $\mathcal{N} = \mathsf{S} \circ \mathsf{P}$. 

The prefix reachable set $\mathsf{P}(\mathcal{X})$ is soundly outer-approximated by a box $\widehat{\mathsf{P}}(\mathcal{X}) = \prod_{j=1}^{n_{l_{\text{cut}}}} [\underline{a}^{l_{\text{cut}}}_j, \overline{a}^{l_{\text{cut}}}_j]$ constructed via classical bound propagation (e.g., Interval Bound Propagation (IBP) or CROWN). Robustness verification is then restricted exclusively to the suffix:
\begin{align}
\label{eq:margin}
\min_{\mathbf{a} \in \widehat{\mathsf{P}}(\mathcal{X})} \widehat{m}_p(\mathbf{a}),\quad \widehat{m}_p(\mathbf{a}) \triangleq \mathbf{e}_{\text{true}}^\top \mathsf{S}(\mathbf{a}) - \mathbf{e}_p^\top \mathsf{S} (\mathbf{a}).
\end{align}
Compared with~\eqref{eq:pairwise_margin} or~\eqref{eq:step_env_objs}, problem~\eqref{eq:margin} drastically reduces spin requirements via a suffix-only formulation. The suffix optimization relaxes the full-scale problem~\eqref{eq:pairwise_margin} by replacing the exact reachable set $\mathsf{P}(\mathcal{X})$ with its outer bounding box $\widehat{\mathsf{P}}(\mathcal{X})$. This relaxation preserves soundness, though completeness is no longer guaranteed.

\section{Experimental Evaluation}
\label{sec:experiments}

We evaluate our framework along three core dimensions: (i) mathematical formulation compactness and asymptotic completeness (Sections~\ref{subsec:log_compactness_exp}--\ref{subsec:step_env_asymt_completeness_exp}); (ii) verification fidelity and combinatorial acceleration of hybrid Benders decomposition (Sections~\ref{subsec:benders_fidelity_exp}--\ref{subsec:contextual_helly_exp}); and (iii) monolithic QUBO verification fidelity and physical Ising hardware computational scaling (Section~\ref{subsec:qubo_exp}). Detailed experimental setup parameters and baseline verifier specifications are provided in Appendix~\ref{appendix_exp_setup}.

For the exact MILP framework (Section~\ref{piecewise_linear}), we built, inspected, and solved exact MILP models for an MNIST classifier with architecture 784-64-32-10 activated by ReLU. For the Step-Env framework, we adopted an MNIST classifier with the same architecture using sigmoid activations. Instead of directly assigning maximal segment width $\Delta$, we evenly divided the propagated interval domain into $\kappa=2,4,8,\dots,64$ segments. This approach allows us to systematically evaluate the relationship between asymptotic completeness and binary spin requirements.

\subsection{Logarithmic Encoding Compactness}
\label{subsec:log_compactness_exp}
In Section~\ref{piecewise_linear}, we showed that logarithmic encoding reduces the number of required binary variables. For the PWL activation ReLU, each neuron has $\kappa=2$ linear segments, which can be encoded via BRGC using only $\lceil \log_2 2 \rceil = 1$ binary variable. On a ReLU network, logarithmic encoding halves the number of binary variables, as shown in Table~\ref{tab:mnist_gurobi_params}. 

\begin{table}[htb]
\centering
\small
\resizebox{\columnwidth}{!}{%
\begin{tabular}{lcccccccc}
\toprule
\textbf{Model Formulation} & \multicolumn{2}{c}{\textbf{Constraints}} & \multicolumn{2}{c}{\textbf{Binary Vars}} & \multicolumn{2}{c}{\textbf{Continuous Vars}} & \multicolumn{2}{c}{\textbf{Non-Zeros (NZs)}} \\
\cmidrule(lr){2-3} \cmidrule(lr){4-5} \cmidrule(lr){6-7} \cmidrule(lr){8-9}
 & One-Hot & Log & One-Hot & Log & One-Hot & Log & One-Hot & Log \\
\midrule
Exact MILP & 1,364 & \textbf{596} & 192 & \textbf{96} & \textbf{1,188} & 1,284 & 55,550 & \textbf{53,822} \\
\bottomrule
\end{tabular}%
}
\caption{MILP formulation parameters for 784-64-32-10 ReLU network verification.}
\label{tab:mnist_gurobi_params}
\end{table}

For Step-Env MILP, the number of binary variables is drastically compressed as segments per neuron $\kappa$ grows from $2$ to $64$ (Table~\ref{tab:step_env_params}). Although logarithmic encoding reduces binary decision variables, it trades off an increase in the number of continuous variables and linear coupling constraints. While such an expanded continuous state space is intractable for monolithic QUBO mappings, it is seamlessly accommodated under our hybrid Benders decomposition framework (Section~\ref{subsec:benders}). In Benders decomposition, all continuous variables reside in the subproblem solved by classical LP routines, while the master problem comprises only the binary variables and cuts, which scale gracefully in $\kappa$ and enable scalable Ising optimization.
\begin{table}[htb]
\centering
\small
\resizebox{\columnwidth}{!}{%
\begin{tabular}{lcccccccc}
\toprule
\textbf{Model Formulation} & \multicolumn{2}{c}{\textbf{Constraints}} & \multicolumn{2}{c}{\textbf{Binary Vars}} & \multicolumn{2}{c}{\textbf{Continuous Vars}} & \multicolumn{2}{c}{\textbf{Non-Zeros (NZs)}} \\
\cmidrule(lr){2-3} \cmidrule(lr){4-5} \cmidrule(lr){6-7} \cmidrule(lr){8-9}
 & One-Hot & Log & One-Hot & Log & One-Hot & Log & One-Hot & Log \\
\midrule
Step-Env ($\kappa=2$)  & \textbf{250} & 346 & 96    & \textbf{48}  & \textbf{112} & 208   & \textbf{930}    & 1,122  \\
Step-Env ($\kappa=4$)  & \textbf{250} & 442 & 192   & \textbf{96}  & \textbf{112} & 304   & \textbf{1,314}  & 1,890  \\
Step-Env ($\kappa=8$)  & \textbf{250} & 538 & 384   & \textbf{144} & \textbf{112} & 496   & \textbf{2,082}  & 3,522  \\
Step-Env ($\kappa=16$) & \textbf{250} & 634 & 768   & \textbf{192} & \textbf{112} & 880   & \textbf{3,618}  & 7,074  \\
Step-Env ($\kappa=32$) & \textbf{250} & 730 & 1,536 & \textbf{240} & \textbf{112} & 1,648 & \textbf{6,690}  & 14,850 \\
Step-Env ($\kappa=64$) & \textbf{250} & 826 & 3,072 & \textbf{288} & \textbf{112} & 3,184 & \textbf{12,834} & 31,842 \\
\bottomrule
\end{tabular}%
}
\caption{Step-Env formulation parameters for 784-64-32-10 sigmoid network verification.}
\label{tab:step_env_params}
\end{table}

\subsection{Step-Env Asymptotic Completeness}
\label{subsec:step_env_asymt_completeness_exp}
Theorem~\ref{thm:convergence} proves asymptotic completeness of Step-Env at linear rate $C_{\text{net}}\Delta$. While $C_{\text{net}}$ depends on network depth and Lipschitz bounds, we demonstrate empirically that high certification success rates (ratio of conclusive verification verdicts) are achieved at moderate segment widths with compact spin budgets.
\begin{figure}[htb]
\centering
\includegraphics[width=0.9\linewidth]{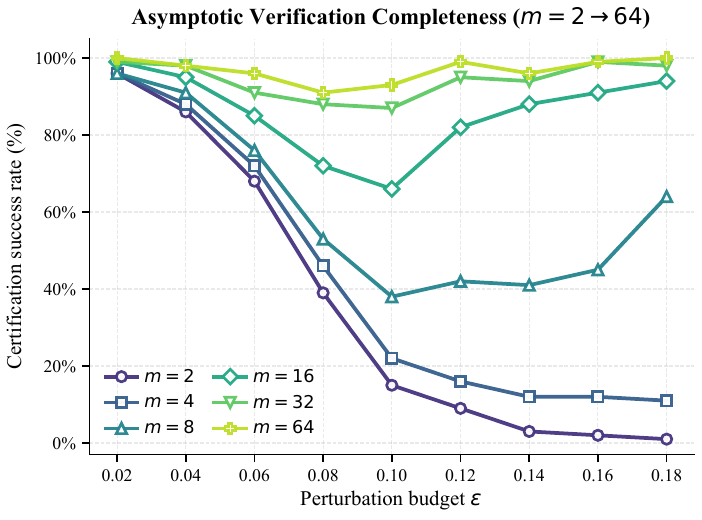}
\caption{Asymptotic completeness of the Step-Env model on the 784-64-32-10 sigmoid MNIST classifier across perturbation radii $\varepsilon$ from 0.02 to 0.18, step size 0.02. As segment count $\kappa$ increases from 2 to 64 (decreasing maximal width $\Delta$), certification success rates uniformly converge toward 100\% across all radii.}
\label{fig:step_env_completeness}
\end{figure}
The experiment adopted a sigmoid-activated 784-64-32-10 MNIST classifier that takes a $[0, 1]$-normalized $28\times 28$ grayscale image as input. The 100 samples for verification were randomly chosen and balanced across all 10 classes. The perturbation radius $\varepsilon$ sweeps from 0.02 to 0.18, step size 0.02 with corresponding true robustness ratio from $94\%$ to $1\%$. The sweep covers the boundary where robustness is marginally satisfied or violated.
The result is summarized in Figure~\ref{fig:step_env_completeness}. As the grid is refined, the success rate uniformly converges toward $100\%$ for all perturbation radii. 

\subsection{Benders Fidelity}
\label{subsec:benders_fidelity_exp}

\begin{figure}[htb]
\centering
\includegraphics[width=0.9\linewidth]{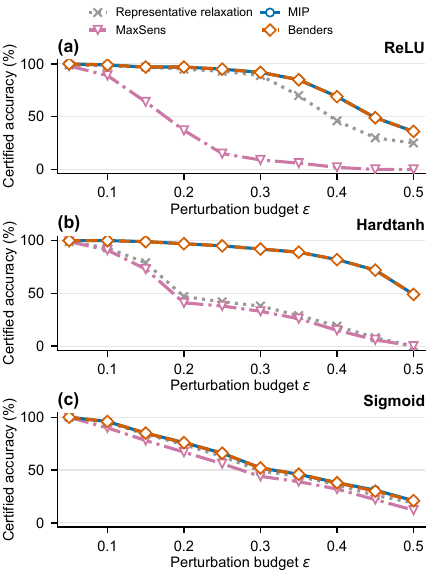}
\caption{Certified accuracy of the hybrid Benders verifier on \texttt{make\_moons} \cite{pedregosa2011scikit}. Benders coincides with exact MIP across all budgets for ReLU and Hardtanh, and differs by at most 1\% for Sigmoid.}
\label{fig:benders_certified_accuracy}
\end{figure}

To evaluate whether separating discrete combinatorial activation decisions from continuous bound propagation preserves verification exactness, we benchmark certified accuracy on a fully connected neural classifier trained on the \texttt{make\_moons} dataset \cite{pedregosa2011scikit} over 100 test samples across expanding perturbation budgets with $\varepsilon$ from 0.05 to 0.50, step size 0.05. As shown in Figure~\ref{fig:benders_certified_accuracy}, for continuous piecewise-linear networks (ReLU and Hardtanh), the hybrid Benders decomposition matches the monolithic MIP baseline certified accuracy identically across all evaluated budgets (100\% agreement across all 20 tested configurations). For non-linear Sigmoid networks verified under Step-Env, Benders certification tracks the exact MIP reference within a maximum deviation of $\le 1\%$ (differing by only a single sample at $\varepsilon=0.45$, where Benders certifies 30 samples versus MIP's 31). In contrast, standard incomplete reachability and convex relaxation baselines (e.g., MaxSens, Duality, DLV) degrade rapidly as perturbation budget $\varepsilon$ expands, severely under-certifying robust samples due to accumulated convex relaxation error. These results confirm that our hybrid Benders decomposition maintains exact verification fidelity without suffering precision loss from problem partitioning. Complete numerical breakdowns across all baseline methods are provided in Appendix~\ref{appendix_numerical_results} (Table~\ref{tab:benders_results}).

\subsection{Contextual Helly Compression}
\label{subsec:contextual_helly_exp}\label{subsec:exp_banknotes}
To evaluate whether our theoretical cut compression framework (Theorem~\ref{thm:benders_complexity} and Remark~\ref{rem:input_dim_rank}) translates into practical scalability on deeper architectures, we benchmark verification on a deep fully connected network (ReLU architecture 4-8-8-8-1, 24 ReLU hidden neurons, 24 binary activation spins) trained on the UCI Banknote Authentication dataset ($n_0=4$). Robustness verification is performed on 100 random, balanced samples under perturbation budgets with $\varepsilon$ from 0.02 to 0.12, step size 0.02. The master IP problem is solved by Gurobi. We evaluated three Benders cut strategies:
(i)~\textbf{Standard Farkas}: classical dual extreme ray feasibility cuts~\cite{franco2022quantum};
(ii)~\textbf{Contextual Helly (Ours)}: context-guarded combinatorial conflict cuts constructed via Lemma~\ref{lem:helly_compression} and Eq.~\eqref{eq:guarded_cut}; and
(iii)~\textbf{Hybrid Cut (Ours)}: the unified strategy adding both Farkas feasibility hyperplanes and contextual Helly conflict certificates simultaneously.

\begin{table*}[t]
\centering
\small
\resizebox{0.95\textwidth}{!}{
\begin{tabular}{l ccccc ccc ccc}
\toprule
\multirow{2}{*}{\textbf{Perturbation}} & \multicolumn{2}{c}{\textbf{Verification Outcome}} & \multicolumn{3}{c}{\textbf{Mean CIM Master Calls / Sample} $\downarrow$} & \multicolumn{3}{c}{\textbf{Mean Solve Time / Sample (s)} $\downarrow$} & \multicolumn{2}{c}{\textbf{Efficiency Gains vs. Farkas} $\uparrow$} \\
\cmidrule(lr){2-3} \cmidrule(lr){4-6} \cmidrule(lr){7-9} \cmidrule(lr){10-11}
\textbf{Budget} $\epsilon$ & Robust (\%) & Inconclusive (\%) & Farkas (B1) & Helly (B3) & Hybrid (B4) & Farkas (B1) & Helly (B3) & Hybrid (B4) & Query Reduction & Time Speedup \\
\midrule
0.02 & 93.0 & 0.00 & 53 & \textbf{3} & \textbf{3} & 0.198 & \textbf{0.012} & 0.013 & 94.3\% & $\mathbf{16.50\times}$ \\
0.04 & 78.0 & 0.00 & 214 & 23 & \textbf{22} & 1.082 & \textbf{0.084} & 0.090 & 89.7\% & $12.88\times$ \\
0.06 & 55.0 & 0.00 & 483 & 55 & \textbf{52} & 2.421 & \textbf{0.217} & 0.335 & 89.2\% & $11.16\times$ \\
0.08 & 31.0 & 0.00 & 809 & 100 & \textbf{94} & 5.434 & \textbf{0.451} & 0.938 & 88.4\% & $12.05\times$ \\
0.10 & 14.0 & 0.00 & 1,195 & 169 & \textbf{157} & 11.124 & \textbf{0.891} & 1.430 & 86.9\% & $12.48\times$ \\
0.12 & 2.0 & 0.00 & 1,578 & 230 & \textbf{210} & 20.559 & \textbf{1.388} & 1.693 & 86.7\% & $14.81\times$ \\
\midrule
\multicolumn{3}{l}{\textbf{Cumulative Totals (100 Samples All $\epsilon$)}} & 4,332 & 580 & \textbf{538} & 40.8 & \textbf{3.04} & 4.50 & \textbf{87.6\%} & $\mathbf{13.41\times}$ \\
\bottomrule
\end{tabular}
}
\caption{Complete performance comparison of Benders decomposition cut strategies on the 4-8-8-8-1 Banknotes network (100 test samples). Certified robustness decisions are 100\% identical across all methods with 0.00\% inconclusive rates, confirming exactness preservation. Contextual Helly cuts achieve the fastest wall-clock execution ($13.41\times$ overall acceleration), while Hybrid cuts achieve the most aggressive master query compression (87.6\% total reduction).}
\label{tab:banknotes_results}
\end{table*}

The empirical scaling in Table~\ref{tab:banknotes_results} and Figure~\ref{fig:banknotes_scaling} highlights three central characteristics of our cut framework. First, all cut strategies preserve strict exactness: across all 100 samples and perturbation radii, robust-versus-vulnerable verdicts are $100\%$ identical with $0.00\%$ inconclusive rate, verifying that Helly conflict compression prunes exclusively provably infeasible branches without distortion. Second, while standard Farkas cuts suffer combinatorial saturation at wider perturbation radii (averaging $1,578$ master calls and $20.56$\,s per sample at $\epsilon=0.12$ due to loose continuous halfspaces over the discrete hypercube), Contextual Helly cuts exploit the low input dimension ($n_0=4$, bounding conflict rank to $\rho \le 4$) to extract certificates of at most $5$ literals. This immediately prunes exponential subcubes ($\ge 2^{19}$ configurations), reducing wall-clock runtime to $1.39$\,s at $\epsilon=0.12$ and delivering a ${13.41\times}$ cumulative end-to-end speedup ($3.04$\,s vs $40.8$\,s). Finally, Hybrid cuts achieve maximum physical query suppression, slashing cumulative CIM invocations by ${87.6\%}$ (from $4,332$ to $538$ calls) by combining discrete combinatorial exclusion with continuous dual rays. This trade-off establishes a practical Pareto frontier where Contextual Helly maximizes overall wall-clock throughput and Hybrid cuts minimize physical quantum hardware query overhead.

\begin{figure}[htb]
\centering
\includegraphics[width=0.9\linewidth]{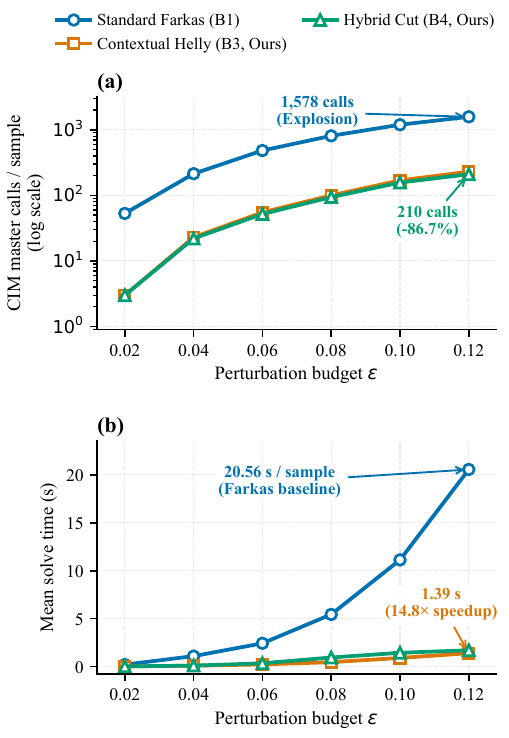}
\caption{Empirical evaluation of Benders cut mechanisms on the 4-8-8-8-1 Banknotes network (24 ReLU neurons, 24 binary spins) across 100 test samples. 
\textbf{(a)} Average CIM master queries per sample versus perturbation budget $\epsilon$ (log scale). Standard Farkas cuts suffer severe combinatorial explosion up to 1,578 calls at $\epsilon=0.120$, whereas Contextual Helly and Hybrid cuts enforce compact polynomial scaling (up to 87.6\% reduction).
\textbf{(b)} Single-sample wall-clock solving time (seconds) across $\epsilon$. Helly cuts eliminate master query bottlenecks with minimal subproblem extraction overhead, achieving the fastest wall-clock performance across all budgets ($1.39$\,s vs $20.56$\,s at $\epsilon=0.120$).}
\label{fig:banknotes_scaling}
\end{figure}
Figure~\ref{fig:banknotes_scaling} summarizes the count of master problem calls and average runtime per sample on a 4-8-8-8-1 banknotes dataset classifier. As is seen, contextual Helly cut and the hybrid strategy can drastically reduce master problem calls and accelerate Benders decomposition framework.

\subsection{Monolithic QUBO Experiments}
\label{subsec:qubo_exp}
We tested monolithic QUBO on ReLU and Sigmoid networks. All experiments used one-hot encoding, where continuous state variables and inequality slacks are discretized using fixed-point binary expansions and mapped onto Ising spin Hamiltonians with quadratic penalties (Appendix~\ref{appendix_qubo}).

\paragraph{Verification Fidelity.}
Figure~\ref{fig:global_qubo_fidelity} summarizes verification outcomes (number of vulnerable samples identified out of 100 test samples) on an Iris classification network \cite{fisher1936use} across perturbation budgets with $\varepsilon$ from 0.1 to 0.6, step size 0.1 for ReLU and from 0.1 to 1.0, step size 0.1 for Sigmoid. We compare monolithic MIP-Gurobi, QUBO-Gurobi, and QUBO-CIM against incomplete relaxation verifiers. For ReLU, both QUBO-Gurobi and QUBO-CIM identically reproduce the exact MIP-Gurobi reference across all perturbation budgets (identifying exactly 0, 0, 0, 2, 6, and 12 vulnerable samples), confirming 100\% certification fidelity for piecewise-linear networks (the misclassified boundary samples at $\varepsilon = 0.5$ are visualized in Figure~\ref{fig:iris_vulnerability}, Appendix~\ref{appendix_numerical_results}). For Sigmoid networks, the 5-segment Step-Env piecewise-constant enclosure solved on QUBO matches the exact MIP baseline almost identically across all budgets, with QUBO-CIM showing only a single-sample deviation at $\varepsilon=0.8$ (52 vs.\ 53). These results confirm that monolithic QUBO encodings faithfully preserve verification semantics.

\begin{figure}[htb]
\centering
\includegraphics[width=0.9\linewidth]{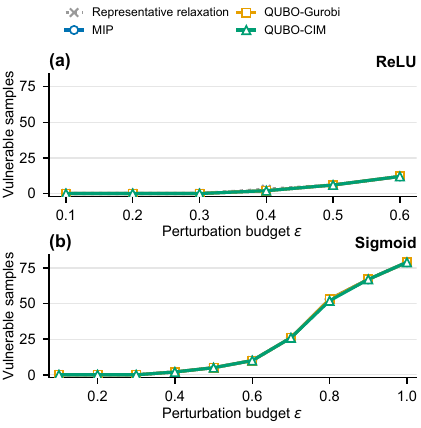}
\caption{Monolithic QUBO verification fidelity across activation functions. For ReLU, both QUBO solvers identically match the exact MIP reference at all budgets. For Sigmoid, Step-Env tracks MIP almost exactly (with only a 1-sample deviation at $\epsilon=0.8$).}
\label{fig:global_qubo_fidelity}
\end{figure}

\paragraph{Computational Profile and Spin Complexity.}
Tables~\ref{tab:relu_results} and \ref{tab:sigmoid_results} in Appendix~\ref{appendix_numerical_results} detail the solver runtime profiles and physical Ising spin budgets across perturbation radii $\varepsilon$. While MIP-Gurobi completes in $0.30$--$1.23$~ms on these shallow instances, QUBO-CIM solves the unconstrained binary Ising instances in $1.79$--$11.97$~ms, maintaining rapid millisecond-scale execution via physical optical spin dynamics. In contrast, standard classical branch-and-bound QUBO-Gurobi encounters severe combinatorial bottlenecks, hitting the 5.0-second timeout for nearly all $\varepsilon \ge 0.2$ across both activations due to dense quadratic penalty interactions.

As input bounds expand with increasing $\varepsilon$, the continuous PWL ReLU activation requires $198.0$--$338.3$ average spins to discretize continuous state variables. For nonlinear Sigmoid networks, Step-Env requires only $64.4$--$103.5$ average spins across all budgets with $\varepsilon$ from 0.1 to 1.0, step size 0.1---roughly one-third of the ReLU spin requirements. This confirms that piecewise-constant step enclosures effectively compress hardware spin requirements by eliminating continuous state discretization overhead, providing a compact, hardware-efficient foundation for physical Ising processors. Detailed numerical runtime and spin breakdowns are tabulated in Appendix~\ref{appendix_numerical_results} (Tables~\ref{tab:relu_results} and \ref{tab:sigmoid_results}).

\subsection{Discussion}
Together, these empirical findings demonstrate that Ising-compatible formulations can preserve exact verification semantics while drastically reducing hardware resource requirements. On single-layer and vision-scale instances, logarithmic and step-envelope encodings achieve minimal spin footprints without relaxation looseness. On deeper architectures where monolithic models exceed hardware capacity, hybrid Benders decomposition with contextual Helly conflict compression mitigates combinatorial cut saturation, offloading continuous bounds while enabling physical Ising coprocessors to operate within a provably output-sensitive search regime.

\section{Conclusion}
We presented an exact and asymptotically complete Ising-compatible framework for formal neural network robustness verification. By developing the state-optimal logarithmic piecewise model (\textbf{Log-PW})---unifying logarithmic SOS2 encodings (\textbf{Log-PWL}) for piecewise-linear activations with logarithmic SOS1 encodings (\textbf{Log-PWC}) for piecewise-constant functions---our approach achieves an information-theoretically minimal spin complexity of $\lceil \log_2 \kappa \rceil$ per neuron. Furthermore, Log-PWC directly powers the asymptotic step-envelope model (\textbf{Step-Env}) for general activations, enabling fine-grained, sound enclosures with linear error convergence $\mathcal{O}(\Delta)$ and asymptotic completeness without spin budget explosion. To scale verification to deeper architectures, our hybrid Benders solver cleanly decouples discrete activation choices from continuous bound propagation, leveraging Contextual Helly Compression to establish output-sensitive iteration bounds. Extensive evaluations validate exact verification fidelity and substantial spin and constraint reductions under compact hardware budgets.

\clearpage
\bibliographystyle{plainnat}
\bibliography{nnverify}

@INPROCEEDINGS{DigitalAnnealer,
  author={Matsubara, Satoshi and Takatsu, Motomu and Miyazawa, Toshiyuki and Shibasaki, Takayuki and Watanabe, Yasuhiro and Takemoto, Kazuya and Tamura, Hirotaka},
  booktitle={2020 25th Asia and South Pacific Design Automation Conference (ASP-DAC)}, 
  title={Digital Annealer for High-Speed Solving of Combinatorial optimization Problems and Its Applications}, 
  year={2020},
  pages={667-672}}

@article{Cai2020,
  title = {Power-efficient combinatorial optimization using intrinsic noise in memristor Hopfield neural networks},
  volume = {3},
  number = {7},
  journal = {Nature Electronics},
  author = {Cai,  Fuxi and Kumar,  Suhas and Van Vaerenbergh,  Thomas and Sheng,  Xia and Liu,  Rui and Li,  Can and Liu,  Zhan and Foltin,  Martin and Yu,  Shimeng and Xia,  Qiangfei and Yang,  J. Joshua and Beausoleil,  Raymond and Lu,  Wei D. and Strachan,  John Paul},
  year = {2020},
  pages = {409–418}
}

@ARTICLE{20k-SpinIsingChip,
  author={Yamaoka, Masanao and Yoshimura, Chihiro and Hayashi, Masato and Okuyama, Takuya and Aoki, Hidetaka and Mizuno, Hiroyuki},
  journal={IEEE Journal of Solid-State Circuits}, 
  title={A 20k-Spin Ising Chip to Solve Combinatorial Optimization Problems With CMOS Annealing}, 
  year={2016},
  volume={51},
  number={1},
  pages={303-309}}

@article{Goto2021,
  title = {High-performance combinatorial optimization based on classical mechanics},
  volume = {7},
  number = {6},
  journal = {Science Advances},
  author = {Goto,  Hayato and Endo,  Kotaro and Suzuki,  Masaru and Sakai,  Yoshisato and Kanao,  Taro and Hamakawa,  Yohei and Hidaka,  Ryo and Yamasaki,  Masaya and Tatsumura,  Kosuke},
  year = {2021}
}

@article{Goto2019,
  title = {Combinatorial optimization by simulating adiabatic bifurcations in nonlinear Hamiltonian systems},
  volume = {5},
  number = {4},
  journal = {Science Advances},
  author = {Goto,  Hayato and Tatsumura,  Kosuke and Dixon,  Alexander R.},
  year = {2019}
}

@article{Boixo2014,
  title = {Evidence for quantum annealing with more than one hundred qubits},
  volume = {10},
  number = {3},
  journal = {Nature Physics},
  publisher = {Springer Science and Business Media LLC},
  author = {Boixo,  Sergio and Rønnow,  Troels F. and Isakov,  Sergei V. and Wang,  Zhihui and Wecker,  David and Lidar,  Daniel A. and Martinis,  John M. and Troyer,  Matthias},
  year = {2014},
  pages = {218–224}
}

@article{johnson2011quantum,
  title = {Quantum annealing with manufactured spins},
  volume = {473},
  number = {7346},
  journal = {Nature},
  author = {Johnson,  M. W. and Amin,  M. H. S. and Gildert,  S. and Lanting,  T. and Hamze,  F. and Dickson,  N. and Harris,  R. and Berkley,  A. J. and Johansson,  J. and Bunyk,  P. and Chapple,  E. M. and Enderud,  C. and Hilton,  J. P. and Karimi,  K. and Ladizinsky,  E. and Ladizinsky,  N. and Oh,  T. and Perminov,  I. and Rich,  C. and Thom,  M. C. and Tolkacheva,  E. and Truncik,  C. J. S. and Uchaikin,  S. and Wang,  J. and Wilson,  B. and Rose,  G.},
  year = {2011},
  pages = {194–198}
}

@article{Honjo2021,
  title = {100, 000-spin coherent Ising machine},
  volume = {7},
  number = {40},
  journal = {Science Advances},
  author = {Honjo,  Toshimori and Sonobe,  Tomohiro and Inaba,  Kensuke and Inagaki,  Takahiro and Ikuta,  Takuya and Yamada,  Yasuhiro and Kazama,  Takushi and Enbutsu,  Koji and Umeki,  Takeshi and Kasahara,  Ryoichi and Kawarabayashi,  Ken-ichi and Takesue,  Hiroki},
  year = {2021}
}

@article{Marandi2014,
  title = {Network of time-multiplexed optical parametric oscillators as a coherent Ising machine},
  volume = {8},
  number = {12},
  journal = {Nature Photonics},
  author = {Marandi,  Alireza and Wang,  Zhe and Takata,  Kenta and Byer,  Robert L. and Yamamoto,  Yoshihisa},
  year = {2014},
  pages = {937–942}
}

@article{100spincim,
author = {Peter L. McMahon  and Alireza Marandi  and Yoshitaka Haribara  and Ryan Hamerly  and Carsten Langrock  and Shuhei Tamate  and Takahiro Inagaki  and Hiroki Takesue  and Shoko Utsunomiya  and Kazuyuki Aihara  and Robert L. Byer  and M. M. Fejer  and Hideo Mabuchi  and Yoshihisa Yamamoto },
title = {A fully programmable 100-spin coherent Ising machine with all-to-all connections},
journal = {Science},
volume = {354},
number = {6312},
pages = {614-617},
year = {2016}}

@article{2000nodecim,
author = {Takahiro Inagaki  and Yoshitaka Haribara  and Koji Igarashi  and Tomohiro Sonobe  and Shuhei Tamate  and Toshimori Honjo  and Alireza Marandi  and Peter L. McMahon  and Takeshi Umeki  and Koji Enbutsu  and Osamu Tadanaga  and Hirokazu Takenouchi  and Kazuyuki Aihara  and Ken-ichi Kawarabayashi  and Kyo Inoue  and Shoko Utsunomiya  and Hiroki Takesue },
title = {A coherent Ising machine for 2000-node optimization problems},
journal = {Science},
volume = {354},
number = {6312},
pages = {603-606},
year = {2016}}

@article{song2023training,
  title={Training Multi-layer Neural Networks on Ising Machine},
  author={Song, Xujie and Liu, Tong and Li, Shengbo Eben and Duan, Jingliang and Wang, Wenxuan and Li, Keqiang},
  journal={arXiv:2311.03408},
  year={2023}
}

@article{helly1923mengen,
  title={{\"U}ber Mengen konvexer K{\"o}rper mit gemeinschaftlichen Punkten},
  author={Helly, Eduard},
  journal={Jahresber. Deutsch. Math.-Verein.},
  volume={32},
  pages={175--176},
  year={1923}
}

@article{vadlamani2026scalable,
  title   = {Scalable Digital Compute-in-Memory Ising Machines for Robustness Verification of Binary Neural Networks},
  author  = {Vadlamani, Madhav and Singh, Rahul and Kong, Yuyao and Zhang, Zheng and Yu, Shimeng},
  journal = {arXiv preprint arXiv:2603.05677},
  year    = {2026}
}

@inproceedings{zhang2023qvip,
author = {Zhang, Yedi and Zhao, Zhe and Chen, Guangke and Song, Fu and Zhang, Min and Chen, Taolue and Sun, Jun},
title = {QVIP: An ILP-based Formal Verification Approach for Quantized Neural Networks},
year = {2023},
booktitle = {Proceedings of the 37th IEEE/ACM International Conference on Automated Software Engineering},
articleno = {82},
numpages = {13},
series = {ASE '22}
}

@article{singh2026robustness,
  title   = {Robustness Verification of Binary Neural Networks: An Ising and Quantum-Inspired Framework},
  author  = {Singh, Rahul and Saeedi, Seyran and Zhang, Zheng},
  journal = {arXiv preprint arXiv:2602.13536},
  year    = {2026}
}

@article{bastani2016measuring,
  title={Measuring neural net robustness with constraints},
  author={Bastani, Osbert and Ioannou, Yani and Lampropoulos, Leonidas and Vytiniotis, Dimitrios and Nori, Aditya and Criminisi, Antonio},
  journal={Advances in neural information processing systems},
  volume={29},
  year={2016}
}

@inproceedings{gehr2018ai2,
  title={Ai2: Safety and robustness certification of neural networks with abstract interpretation},
  author={Gehr, Timon and Mirman, Matthew and Drachsler-Cohen, Dana and Tsankov, Petar and Chaudhuri, Swarat and Vechev, Martin},
  booktitle={2018 IEEE symposium on security and privacy (SP)},
  pages={3--18},
  year={2018},
  organization={IEEE}
}

@inproceedings{huang2017safety,
  title={Safety verification of deep neural networks},
  author={Huang, Xiaowei and Kwiatkowska, Marta and Wang, Sen and Wu, Min},
  booktitle={International conference on computer aided verification},
  pages={3--29},
  year={2017},
  organization={Springer}
}

@article{xiang2017reachable,
  title={Reachable set computation and safety verification for neural networks with relu activations},
  author={Xiang, Weiming and Tran, Hoang-Dung and Johnson, Taylor T},
  journal={arXiv preprint arXiv:1712.08163},
  year={2017}
}

@inproceedings{weng2018towards,
  title={Towards fast computation of certified robustness for relu networks},
  author={Weng, Lily and Zhang, Huan and Chen, Hongge and Song, Zhao and Hsieh, Cho-Jui and Daniel, Luca and Boning, Duane and Dhillon, Inderjit},
  booktitle={International Conference on Machine Learning},
  pages={5276--5285},
  year={2018},
  organization={PMLR}
}

@article{xiang2018output,
  title={Output reachable set estimation and verification for multilayer neural networks},
  author={Xiang, Weiming and Tran, Hoang-Dung and Johnson, Taylor T},
  journal={IEEE transactions on neural networks and learning systems},
  volume={29},
  number={11},
  pages={5777--5783},
  year={2018},
  publisher={IEEE}
}

@inproceedings{wang2018formal,
  title={Formal security analysis of neural networks using symbolic intervals},
  author={Wang, Shiqi and Pei, Kexin and Whitehouse, Justin and Yang, Junfeng and Jana, Suman},
  booktitle={27th USENIX Security Symposium (USENIX Security 18)},
  pages={1599--1614},
  year={2018}
}

@inproceedings{dvijotham2018dual,
  title={A Dual Approach to Scalable Verification of Deep Networks.},
  author={Dvijotham, Krishnamurthy and Stanforth, Robert and Gowal, Sven and Mann, Timothy A and Kohli, Pushmeet},
  booktitle={UAI},
  volume={1},
  number={2},
  pages={3},
  year={2018}
}

@article{lomuscio2017approach,
  title={An approach to reachability analysis for feed-forward relu neural networks},
  author={Lomuscio, Alessio and Maganti, Lalit},
  journal={arXiv preprint arXiv:1706.07351},
  year={2017}
}

@article{dutta2017output,
  title={Output range analysis for deep neural networks},
  author={Dutta, Souradeep and Jha, Susmit and Sanakaranarayanan, Sriram and Tiwari, Ashish},
  journal={arXiv preprint arXiv:1709.09130},
  year={2017}
}

@article{pedregosa2011scikit,
  title={Scikit-learn: Machine learning in Python},
  author={Pedregosa, Fabian and Varoquaux, Ga{\"e}l and Gramfort, Alexandre and Michel, Vincent and Thirion, Bertrand and Grisel, Olivier and Blondel, Mathieu and Prettenhofer, Peter and Weiss, Ron and Dubourg, Vincent and others},
  journal={the Journal of machine Learning research},
  volume={12},
  pages={2825--2830},
  year={2011},
  publisher={JMLR. org}
}

@article{fisher1936use,
  title={The use of multiple measurements in taxonomic problems},
  author={Fisher, Ronald A},
  journal={Annals of eugenics},
  volume={7},
  number={2},
  pages={179--188},
  year={1936},
  publisher={Wiley Online Library}
}

@article{liu2021algorithms,
  title={Algorithms for verifying deep neural networks},
  author={Liu, Changliu and Arnon, Tomer and Lazarus, Christopher and Strong, Christopher and Barrett, Clark and Kochenderfer, Mykel J and others},
  journal={Foundations and Trends{\textregistered} in Optimization},
  volume={4},
  number={3-4},
  pages={244--404},
  year={2021},
  publisher={Now Publishers, Inc.}
}

@inproceedings{krizhevsky2012imagenet,
  title={ImageNet Classification with Deep Convolutional Neural Networks},
  author={Krizhevsky, Alex and Sutskever, Ilya and Hinton, Geoffrey E.},
  booktitle={Advances in Neural Information Processing Systems (NeurIPS)},
  year={2012}
}

@inproceedings{he2016deep,
  title={Deep Residual Learning for Image Recognition},
  author={He, Kaiming and Zhang, Xiangyu and Ren, Shaoqing and Sun, Jian},
  booktitle={Proceedings of the IEEE Conference on Computer Vision and Pattern Recognition (CVPR)},
  year={2016}
}

@inproceedings{vaswani2017attention,
  title={Attention Is All You Need},
  author={Vaswani, Ashish and others},
  booktitle={Advances in Neural Information Processing Systems (NeurIPS)},
  year={2017}
}

@inproceedings{devlin2019bert,
  title={Bert: Pre-training of deep bidirectional transformers for language understanding},
  author={Devlin, Jacob and Chang, Ming-Wei and Lee, Kenton and Toutanova, Kristina},
  booktitle={Proceedings of the 2019 conference of the North American chapter of the association for computational linguistics: human language technologies, volume 1 (long and short papers)},
  pages={4171--4186},
  year={2019}
}

@article{bojarski2016end,
  title={End to end learning for self-driving cars},
  author={Bojarski, Mariusz and Del Testa, Davide and Dworakowski, Daniel and Firner, Bernhard and Flepp, Beat and Goyal, Prasoon and Jackel, Lawrence D and Monfort, Mathew and Muller, Urs and Zhang, Jiakai and others},
  journal={arXiv preprint arXiv:1604.07316},
  year={2016}
}

@article{jumper2021highly,
  title={Highly accurate protein structure prediction with AlphaFold},
  author={Jumper, John and Evans, Richard and Pritzel, Alexander and Green, Tim and Figurnov, Michael and Ronneberger, Olaf and Tunyasuvunakool, Kathryn and Bates, Russ and {\v{Z}}{\'\i}dek, Augustin and Potapenko, Anna and others},
  journal={nature},
  volume={596},
  number={7873},
  pages={583--589},
  year={2021}
}

@inproceedings{eykholt2018robust,
  title={Robust physical-world attacks on deep learning visual classification},
  author={Eykholt, Kevin and Evtimov, Ivan and Fernandes, Earlence and Li, Bo and Rahmati, Amir and Xiao, Chaowei and Prakash, Atul and Kohno, Tadayoshi and Song, Dawn},
  booktitle={Proceedings of the IEEE conference on computer vision and pattern recognition},
  pages={1625--1634},
  year={2018}
}

@article{szegedy2013intriguing,
  title={Intriguing properties of neural networks},
  author={Szegedy, Christian and Zaremba, Wojciech and Sutskever, Ilya and Bruna, Joan and Erhan, Dumitru and Goodfellow, Ian and Fergus, Rob},
  journal={arXiv preprint arXiv:1312.6199},
  year={2013}
}

@article{goodfellow2014explaining,
  title={Explaining and harnessing adversarial examples},
  author={Goodfellow, Ian J and Shlens, Jonathon and Szegedy, Christian},
  journal={arXiv preprint arXiv:1412.6572},
  year={2014}
}

@inproceedings{madry2018towards,
  title={Towards Deep Learning Models Resistant to Adversarial Attacks},
  author={Madry, Aleksander and Makelov, Aleksandar and Schmidt, Ludwig and Tsipras, Dimitris and Vladu, Adrian},
  booktitle={International Conference on Learning Representations},
  year={2018}
}

@inproceedings{wong2018provable,
  title={Provable defenses against adversarial examples via the convex outer adversarial polytope},
  author={Wong, Eric and Kolter, Zico},
  booktitle={International conference on machine learning},
  pages={5286--5295},
  year={2018},
  organization={PMLR}
}

@inproceedings{cohen2019certified,
  title={Certified adversarial robustness via randomized smoothing},
  author={Cohen, Jeremy and Rosenfeld, Elan and Kolter, Zico},
  booktitle={international conference on machine learning},
  pages={1310--1320},
  year={2019},
  organization={PMLR}
}

@inproceedings{katz2017reluplex,
  title={Reluplex: An efficient SMT solver for verifying deep neural networks},
  author={Katz, Guy and Barrett, Clark and Dill, David L and Julian, Kyle and Kochenderfer, Mykel J},
  booktitle={International conference on computer aided verification},
  pages={97--117},
  year={2017},
  organization={Springer}
}

@inproceedings{tjeng2019evaluating,
  title={Evaluating Robustness of Neural Networks with Mixed Integer Programming},
  author={Tjeng, Vincent and Xiao, Kai Y and Tedrake, Russ},
  booktitle={International Conference on Learning Representations},
  year={2019}
}

@article{bunel2018unified,
  title={A unified view of piecewise linear neural network verification},
  author={Bunel, Rudy R and Turkaslan, Ilker and Torr, Philip and Kohli, Pushmeet and Mudigonda, Pawan K},
  journal={Advances in neural information processing systems},
  volume={31},
  year={2018}
}

@article{wang2018efficient,
  title={Efficient formal safety analysis of neural networks},
  author={Wang, Shiqi and Pei, Kexin and Whitehouse, Justin and Yang, Junfeng and Jana, Suman},
  journal={Advances in neural information processing systems},
  volume={31},
  year={2018}
}

@inproceedings{franco2022quantum,
  title={Quantum robustness verification: A hybrid quantum-classical neural network certification algorithm},
  author={Franco, Nicola and Wollschl{\"a}ger, Tom and Gao, Nicholas and Lorenz, Jeanette Miriam and G{\"u}nnemann, Stephan},
  booktitle={2022 IEEE International Conference on Quantum Computing and Engineering (QCE)},
  pages={142--153},
  year={2022}
}

@article{vielma2010how,
  title = {Mixed-Integer Models for Nonseparable Piecewise-Linear Optimization: Unifying Framework and Extensions},
  volume = {58},
  number = {2},
  journal = {Operations Research},
  author = {Vielma,  Juan Pablo and Ahmed,  Shabbir and Nemhauser,  George},
  year = {2010},
  pages = {303–315}
}

@inproceedings{montufar2014number,
  title={On the number of linear regions of deep neural networks},
  author={Montufar, Guido F and Pascanu, Razvan and Cho, Kyunghyun and Bengio, Yoshua},
  booktitle={Advances in neural information processing systems},
  volume={27},
  year={2014}
}

@inproceedings{serra2018bounding,
  title={Bounding and counting linear regions of deep neural networks},
  author={Serra, Thiago and Tjandraatmadja, Christian and Ramalingam, Srikumar},
  booktitle={International Conference on Machine Learning},
  pages={4558--4566},
  year={2018},
  organization={PMLR}
}

@inproceedings{hanin2019deep,
  title={Deep {ReLU} networks have surprisingly few activation patterns},
  author={Hanin, Boris and Rolnick, David},
  booktitle={Advances in Neural Information Processing Systems},
  volume={32},
  year={2019}
}

@inproceedings{zhang2018crown,
  title={Efficient Neural Network Robustness Certification with General Activation Functions},
  author={Zhang, Huan and Weng, Tsui-Wei and Chen, Pin-Yu and Hsieh, Cho-Jui and Daniel, Luca},
  booktitle={Advances in Neural Information Processing Systems},
  volume={31},
  year={2018}
}

@article{singh2019deeppoly,
  title={An Abstract Domain for Certifying Neural Networks},
  author={Singh, Gagandeep and Gehr, Timon and P{\"u}schel, Markus and Vechev, Martin},
  journal={Proceedings of the ACM on Programming Languages},
  volume={3},
  number={POPL},
  pages={1--30},
  year={2019}
}

@inproceedings{henriksen2020verinet,
  title={Efficient Neural Network Verification via Adaptive Refinement and Adversarial Search},
  author={Henriksen, Patrick and Lomuscio, Alessio},
  booktitle={European Conference on Artificial Intelligence},
  pages={2513--2520},
  year={2020}
}

@inproceedings{shi2025genbab,
author = {Shi, Zhouxing and Jin, Qirui and Kolter, Zico and Jana, Suman and Hsieh, Cho-Jui and Zhang, Huan},
title = {Neural Network Verification with Branch-and-Bound for General Nonlinearities},
year = {2025},
booktitle = {31st International Conference on Tools and Algorithms for the Construction and Analysis of Systems (TACAS)},
pages = {315–335},
numpages = {21}
}

@inproceedings{franco2023milp,
  title={Efficient MILP Decomposition in Quantum Computing for ReLU Network Robustness},
  author={Franco, Nicola and Wollschl{\"a}ger, Tom and Poggel, Benedikt and G{\"u}nnemann, Stephan and Lorenz, Jeanette Miriam},
  booktitle={2023 IEEE International Conference on Quantum Computing and Engineering (QCE)},
  pages={524--534},
  year={2023},
  organization={IEEE},
  doi={10.1109/QCE57702.2023.00066}
}

@misc{dua2019uci,
  author = {Dua, Dheeru and Graff, Casey},
  year = {2017},
  title = {{UCI} Machine Learning Repository},
  url = {http://archive.ics.uci.edu/ml},
  institution = {University of California, Irvine, School of Information and Computer Sciences}
}

@article{hendrycks2016gelu,
  title={Gaussian Error Linear Units ({GELUs})},
  author={Hendrycks, Dan and Gimpel, Kevin},
  journal={arXiv preprint arXiv:1606.08415},
  year={2016},
  url={https://arxiv.org/abs/1606.08415}
}

@article{ramachandran2017swish,
  title={Searching for Activation Functions},
  author={Ramachandran, Prajit and Zoph, Barret and Le, Quoc V.},
  journal={arXiv preprint arXiv:1710.05941},
  year={2017},
  url={https://arxiv.org/abs/1710.05941}
}

@misc{bell2016bv,
  author={Bell, Jordan},
  title={Functions of Bounded Variation and Differentiability},
  year={2016},
  note={Theorems 1 and 2},
  url={https://jordanbell.info/LaTeX/mathematics/BVdifferentiable/}
}

\clearpage
\appendix

\section{Related Work}\label{related_work}
We provide a review of literature across formal neural network verification, discrete optimization, and quantum/Ising computing, contextualizing how our proposed framework advances beyond existing paradigms.

\subsection{Classical and Quantized Network Verification}
Formal verification certifies prediction invariance within a bounded input region. Exact classical solvers employ Satisfiability Modulo Theories (SMT) \cite{katz2017reluplex}, Mixed-Integer Linear Programming (MILP) \cite{tjeng2019evaluating, bastani2016measuring}, and Branch-and-Bound (BaB) \cite{bunel2018unified} for piecewise-linear (PWL) networks. To scale to continuous non-linear activations (e.g., Sigmoid, Tanh), linear bound propagation (CROWN \cite{zhang2018crown}), abstract interpretation (DeepPoly \cite{singh2019deeppoly}), and Lagrangian duality \cite{dvijotham2018dual} compute sound outer-approximations. For discrete low-bit architectures, Zhang et al. \cite{zhang2023qvip} introduced QVIP, an Integer Linear Programming (ILP) framework that encodes fixed-point quantized weights and operations into ILP constraints, enabling exact verification on classical solvers.

Classical verification algorithms (SMT/MILP/LP/ILP) rely on von Neumann CPU/GPU architectures to solve continuous or integer linear constraints. They explore high-dimensional activation state spaces via sequential Branch-and-Bound tree search, facing severe computational bottlenecks as network depth and width increase, and their floating-point/integer representations cannot be mapped onto non-von Neumann parallel physical solvers. In contrast, our framework reformulates combinatorial verification directly into quadratic spin interactions ($\mathbf{H} = \mathbf{s}^\top \mathbf{Q} \mathbf{s}$) tailored for high-speed parallel physical Ising machines (e.g., CIMs). Furthermore, while classical MILP/ILP tools require linear binary variables per segment, our Exact Log-PW model achieves an information-theoretically minimal logarithmic spin complexity ($\lceil \log_2 \kappa \rceil$ spins per neuron) across both continuous piecewise-linear (Log-PWL) and piecewise-constant (Log-PWC) activations. In turn, Log-PWC provides the scalable discrete foundation for our Step-Env model, which guarantees sound enclosures and strict-margin asymptotic completeness for general activations.

\subsection{Ising and Quantum-Inspired Robustness Verification}
Specialized Ising hardware (e.g., Coherent Ising Machines, quantum annealers, compute-in-memory chips) solves Quadratic Unconstrained Boolean Optimization (QUBO) with high physical parallelism. Franco et al. \cite{franco2022quantum} first mapped network certification to hybrid QUBO subproblems. Recently, Singh, Saeedi, and Zhang \cite{singh2026robustness} proposed a Quadratic Constrained Boolean Optimization (QCBO) to QUBO transformation for Binary Neural Networks (BNNs), while Vadlamani et al. \cite{vadlamani2026scalable} designed an SRAM-based Digital Compute-in-Memory (DCIM) Ising machine implementing in-memory annealing for BNN robustness verification. Unlike the BNN verification schemes of Singh, Saeedi, and Zhang \cite{singh2026robustness} and Vadlamani et al. \cite{vadlamani2026scalable}, which rely strictly on discrete binary weights and activation states ($\{-1, +1\}$), our framework accommodates arbitrary continuous PWL and PWC activations via Log-PW with minimal logarithmic spin overhead and establishes uniform error bounds for general Lipschitz activations via Step-Env, bridging general deep neural networks with Ising computing hardware.

\subsection{Decomposition Methods and Cut Complexity in Quantum Optimization}
Franco et al. \cite{franco2023milp} provided an insightful comparative study of Benders vs.\ Dantzig--Wolfe decomposition on quantum hardware, analyzing how generic Benders decomposition scales when every infeasibility cut introduces encoded slack/binary decision variables into the master problem. Their rigorous analysis establishes important resource baselines for generic MIP formulations, showing how cut accumulation can impact physical spin requirements and leading them to explore Dantzig--Wolfe decomposition to manage master problem dimensions.

Building upon and complementing their foundational work, our study investigates the instance-specific geometric structure of neural verification problems. Specifically, when verification subproblems admit low-dimensional infeasibility certificates, a single conflict cut excludes an entire exponential family of activation choices. Through Helly's theorem \cite{helly1923mengen} on convex arrangements, we establish that the number of visited infeasible Benders iterations is bounded by a polynomial $\mathcal{O}(K_{\mathcal{X}} L P^{\rho+1})$. This result complements generic worst-case analysis by identifying the explicit geometric properties (low-dimensional convex arrangements in perturbation space) that naturally mitigate cut accumulation, offering theoretical support for the practical viability of hybrid Benders decomposition on Ising hardware.

\section{Details of the Conventional One-Hot MILP Formulation}\label{appendix_one_hot_milp}

We use the following definitions throughout the paper. 
\begin{definition}[Piecewise Linear]
\label{def:pwl}
    A piecewise linear function is a continuous function that is an affine function on each interval of a subdivision of its domain. 
\end{definition}
\begin{definition}[Piecewise Constant]
\label{def:pwc}
    A piecewise constant function is a function that is locally constant on each (open) interval of a subdivision of its domain. 
\end{definition}
\begin{remark}
\label{rmk:ambiguity}
    Since PWL functions are continuous, no ambiguity exists on the common grid point of two adjacent segments. However, such ambiguity can exist for PWC functions. More concretely, $\sigma(v_i)$ can take any value in between $\gamma_i$ and $\gamma_{i+1}$ under the MILP framework. In this paper, we do not address this technicality for two reasons: (i) the choice between $\gamma_i$ and $\gamma_{i+1}$ at boundary points is arbitrary in Step-Env, a primary application of PWC activations; (ii) the set of inputs that suffer from boundary ambiguity has Lebesgue measure zero for non-degenerate networks. 
\end{remark}
In standard mixed-integer linear programming (MILP) verification of piecewise-linear (PWL) neural networks, the pre-activation domain $[v_0, v_\kappa]$ of each neuron $j$ at layer $l$ is partitioned into $\kappa$ segments $[v_{i-1}, v_i]$ for $i \in \{1, \dots, \kappa\}$. Over each segment $i$, the PWL activation function $\sigma(\cdot)$ is represented by an affine mapping $\sigma(z) = \alpha_i z + \gamma_i$, where $\alpha_i$ is the segment slope and $\gamma_i$ is the intercept.

\subsection{Standard Bilinear One-Hot MILP (Big-M Formulation)}

To construct a globally sound MILP model using a direct one-hot encoding:
\begin{enumerate}
  \item \textbf{One-Hot Segment Selection and Pre-Activation Bounds}: We assign a binary indicator variable $\beta_{j,i}^l \in \{0, 1\}$ to each segment $i \in \{1, \dots, \kappa\}$. The active segment is selected via the uniqueness equality:
  \begin{equation}
  \sum_{i=1}^\kappa \beta_{j,i}^l = 1.
  \end{equation}
  The pre-activation $z_j^l$ is restricted to the active segment $[v_{i-1}, v_i]$ via the segment bounding inequalities:
  \begin{equation}
  \sum_{i=1}^\kappa v_{i-1} \beta_{j,i}^l \le z_j^l \le \sum_{i=1}^\kappa v_i \beta_{j,i}^l.
  \end{equation}

  \item \textbf{Auxiliary Gated Pre-Activations $u_{j,i}^l$ and Big-M Linearization}: The output activation $a_j^l = \sum_{i=1}^\kappa \beta_{j,i}^l (\alpha_i z_j^l + \gamma_i)$ contains bilinear terms $z_j^l \beta_{j,i}^l$. We introduce continuous auxiliary variables $u_{j,i}^l \triangleq z_j^l \beta_{j,i}^l$. Given pre-activation bounds $z_j^l \in [L_j^l, U_j^l]$, $u_{j,i}^l$ is linearized via standard Big-M inequalities:
  \begin{equation}
  \begin{aligned}
  z_j^l - U_j^l (1 - \beta_{j,i}^l) &\le u_{j,i}^l \le z_j^l - L_j^l (1 - \beta_{j,i}^l), \\
  L_j^l \beta_{j,i}^l &\le u_{j,i}^l \le U_j^l \beta_{j,i}^l.
  \end{aligned}
  \end{equation}
  When $\beta_{j,i}^l = 1$, these force $u_{j,i}^l = z_j^l$; when $\beta_{j,i}^l = 0$, they force $u_{j,i}^l = 0$.

  \item \textbf{Linear Activation Reconstruction}: Utilizing $u_{j,i}^l$, the output activation $a_j^l$ is expressed as an exact linear combination:
  \begin{equation}
  a_j^l = \sum_{i=1}^\kappa \left( \alpha_i u_{j,i}^l + \gamma_i \beta_{j,i}^l \right).
  \end{equation}

  \item \textbf{Affine Feedforward Layer Coupling}: Pre-activations $z_j^l$ are coupled to previous layer activations via:
  \begin{equation}
  z_j^l = \sum_{k=1}^{n_{l-1}} w_{jk}^l a_k^{l-1} + b_j^l.
  \end{equation}
\end{enumerate}

Collecting all continuous variables into $\mathbf{y} \triangleq [\mathbf{x}^\top, (\mathbf{a}^1)^\top, (\mathbf{z}^1)^\top, (\mathbf{u}^1)^\top, \dots, (\mathbf{a}^L)^\top, (\mathbf{z}^L)^\top, (\mathbf{u}^L)^\top ]^\top \in \mathbb{R}^{N_{\text{cont}}}$ and binary choices into $\bm{\beta} \in \{0, 1\}^{N_{\text{bin}}}$, the full network constraints compile into global linear systems $\mathbf{A}_{\text{eq}} \mathbf{y} = \mathbf{b}_0 + \mathbf{B}_{\text{eq}} \bm{\beta}$ and $\mathbf{C} \mathbf{y} \leq \mathbf{d}_0 + \mathbf{D} \bm{\beta}$.

\subsection{Big-M-Free and Auxiliary-Free One-Hot Formulations}
A natural theoretical question is whether one-hot binary encoding ($\sum_{i=1}^\kappa \beta_{j,i}^l = 1$) can completely avoid continuous auxiliary product variables $u_{j,i}^l$ and global Big-M relaxation bounds $[L_j^l, U_j^l]$. In disjunctive programming, two classical modeling paradigms achieve this:

\paragraph{1. Multiple-Choice (MC) One-Hot Model.}
Instead of maintaining a single global pre-activation variable $z_j^l$ multiplied by binary variables $\beta_{j,i}^l$, the pre-activation $z_j^l$ and activation $a_j^l$ are decomposed into local segment-restricted continuous variables $z_{j,i}^l$ and $a_{j,i}^l$:
\begin{equation}
z_j^l = \sum_{i=1}^\kappa z_{j,i}^l, \qquad a_j^l = \sum_{i=1}^\kappa a_{j,i}^l.
\end{equation}
Each local pre-activation $z_{j,i}^l$ is constrained directly within its local segment bounds using the segment endpoints $v_{i-1}$ and $v_i$:
\begin{equation}
v_{i-1} \beta_{j,i}^l \le z_{j,i}^l \le v_i \beta_{j,i}^l, \quad \forall i \in \{1, \dots, \kappa\}.
\end{equation}
The local activation $a_{j,i}^l$ is then defined by the exact linear mapping:
\begin{equation}
a_{j,i}^l = \alpha_i z_{j,i}^l + \gamma_i \beta_{j,i}^l.
\end{equation}
When segment $\beta_{j,i}^l = 0$, the bounds force local states $z_{j,i}^l = 0$ and $a_{j,i}^l = 0$; when $\beta_{j,i}^l = 1$, $z_{j,i}^l \in [v_{i-1}, v_i]$ and $a_{j,i}^l = \alpha_i z_{j,i}^l + \gamma_i$. The key advantage of the MC model is that it completely eliminates auxiliary gated variables $u_{j,i}^l \triangleq z_j^l \beta_{j,i}^l$ and requires \emph{no global Big-M bounds} $[L_j^l, U_j^l]$, utilizing only the fixed local segment grid endpoints $v_{i-1}$ and $v_i$, thereby forming the exact convex hull relaxation of the piecewise-linear graph over one-hot variables.

\paragraph{2. Special Ordered Set (SOS2) One-Hot Model.}
In the Special Ordered Set (SOS2) formulation, pre-activation $z_j^l$ and activation $a_j^l$ are expressed as convex combinations of the grid points $v_0 < v_1 < \dots < v_\kappa$ using continuous grid point weights $\lambda_{j,p}^l \ge 0$ ($p \in \{0, \dots, \kappa\}$) satisfying $\sum_{p=0}^\kappa \lambda_{j,p}^l = 1$:
\begin{equation}
z_j^l = \sum_{p=0}^\kappa \lambda_{j,p}^l v_p, \qquad a_j^l = \sum_{p=0}^\kappa \lambda_{j,p}^l \sigma(v_p).
\end{equation}
The continuous grid point weights $\lambda_{j,p}^l$ are coupled to the one-hot segment binary indicators $\beta_{j,i}^l \in \{0, 1\}$ ($\sum_{i=1}^\kappa \beta_{j,i}^l = 1$) via purely linear inequalities:
\begin{equation}
\begin{aligned}
\lambda_{j,0}^l &\le \beta_{j,1}^l, \\
\lambda_{j,p}^l &\le \beta_{j,p}^l + \beta_{j,p+1}^l \quad (\forall p \in \{1, \dots, \kappa-1\}), \\
\lambda_{j,\kappa}^l &\le \beta_{j,\kappa}^l.
\end{aligned}
\end{equation}
Like the MC model, the key advantage of the Special Ordered Set (SOS2) model is that it operates entirely without continuous-binary products $z_j^l \beta_{j,i}^l$, dispenses with auxiliary gated variables $u_{j,i}^l$, and requires no Big-M relaxations, with all coupling constraints between continuous convex weights and binary segment selection indicators being purely linear.

\subsection{Comparative Summary across PWL Verification Encodings}
Table~\ref{tab:pwl_encoding_comparison} summarizes the structural properties, variable requirements, and relaxation tightness of different piecewise-linear network encodings.

\begin{table*}[t]
\centering
\resizebox{\textwidth}{!}{
\begin{tabular}{lcccccc}
\toprule
\textbf{Model Formulation} & \textbf{Binary Variables / Neuron} & \textbf{Constraints / Neuron} & \textbf{Aux. Vars $u_i$ Needed?} & \textbf{Big-M Bounds Needed?} & \textbf{Coupling Constraints} & \textbf{LP Relaxation Tightness}\\
\midrule
Standard One-Hot (Big-M)  & $\kappa$ (Linear)             & $\mathcal{O}(\kappa)$ ($4\kappa+3$)     & Yes ($\kappa$ continuous)           & Yes ($[L_j, U_j]$)            & Bilinear Big-M                 & Weak (bounds dependent)\\
Multiple-Choice (MC)      & $\kappa$ (Linear)             & $\mathcal{O}(\kappa)$ ($3\kappa+3$)     & No                             & No                            & Local Linear Disjunctive       & Tight (Convex Hull)\\
Special Ordered Set (SOS2)   & $\kappa$ (Linear)             & $\mathcal{O}(\kappa)$ ($\kappa+4$)      & No                             & No                            & Linear Grid Weight Coupling    & Tight (Convex Hull)\\
\textbf{Exact Log-PW (Log-PWL)} & $\lceil \log_2 \kappa \rceil$ (\textbf{Logarithmic}) & $\mathcal{O}(\log \kappa)$ ($2\lceil \log_2 \kappa \rceil$) & \textbf{No} & \textbf{No} & \textbf{Linear Vielma--Nemhauser} & \textbf{Tight (Convex Hull)}\\
\bottomrule
\end{tabular}
}
\caption{Comparison of MILP formulations for piecewise-linear verification with $\kappa$ segments per neuron. Log-PW generalizes this logarithmic compactness to both PWL (SOS2) and PWC (SOS1) activations.}
\label{tab:pwl_encoding_comparison}
\end{table*}

As shown in Table~\ref{tab:pwl_encoding_comparison}, while the MC and Special Ordered Set (SOS2) One-Hot formulations eliminate auxiliary variables $u_{j,i}^l$ and Big-M bounds, they still require $\kappa$ one-hot binary variables per neuron ($\mathcal{O}(\kappa)$ spin complexity). In contrast, our proposed \textbf{Exact Log-PW} model achieves both advantages simultaneously across piecewise activations: for PWL activations, Log-PWL completely eliminates auxiliary variables $u_{j,i}^l$ and Big-M relaxations via linear Vielma--Nemhauser SOS2 coupling constraints, while compressing binary decisions from $\kappa$ one-hot indicators down to an information-theoretically minimal logarithmic spin complexity ($\lceil \log_2 \kappa \rceil$ binary variables per neuron). Analogously, for PWC activations, Log-PWC enforces SOS1 selection with the same logarithmic state efficiency.

This theoretical reduction is empirically validated on the 96-ReLU MNIST benchmark in Section~\ref{subsec:log_compactness_exp} (Table~\ref{tab:mnist_gurobi_params}): for standard 2-segment ReLU activations ($\kappa=2$), Log-PW (via Log-PWL) reduces binary variables by exactly $50.0\%$ (from 192 down to 96, matching $\lceil \log_2 2 \rceil = 1$ vs. $\kappa=2$ per neuron) and cuts linear constraints by $56.3\%$ (from 1,364 to 596), while preserving $100\%$ certificate equivalence without introducing any relaxation slack.

\subsection{Exactness and Completeness of the Logarithmic Piecewise Model}\label{appendix_log_pw_exactness}\label{appendix_log_pwl_exactness}

The following statement shows the mathematical exactness of our Log piecewise model. The exactness concerns the mixed-integer formulation, not a finite-precision QUBO implementation or a heuristic Ising-solver solution.

\begin{theorem}[Exactness of Log-PW]\label{thm:log_pw_exactness}\label{thm:log_pwl_exactness}
Suppose that every activation is piecewise linear or piecewise constant; the grid points $v_0,\dots,v_\kappa$ cover the reachable domain for all $\mathbf{x}\in \mathcal X$; the activation is affine or constant on $[v_{i-1},v_{i}]$; and the input does not suffer from ambiguity (Remark~\ref{rmk:ambiguity}) on PWC activations. Set \(K=\lceil\log_2 \kappa\rceil\), and assign distinct BRGC codes \(\mathbf c_1,\ldots,\mathbf c_\kappa\in\{0,1\}^K\) so that adjacent segments have codes differing in exactly one bit. 
Then the projection of the Log-PW feasible set onto \((\mathbf x,\mathbf z^1,\mathbf a^1,\ldots,\mathbf z^L,\mathbf a^L)\) is exactly the trajectory graph of the original PW network on \(\mathcal X\). Consequently, if the mixed-integer problem is solved globally, direct minimization of every class margin is sound and complete for that network.
\end{theorem}

\begin{proof}
We only prove for PWL activations. The proof for Log-PWC is analogous. 
For a fixed neuron, constraints~\eqref{eq:vn_coupling_pwl} imply $\lambda_p>0$ only if $v_p\in\bigcap_k V_k^{g_k}$. Each vertex has one incident segment code or two codes differing in one bit; hence $\mathbf g$ must equal an incident code. Thus $\mathbf g=\mathbf c_i$ restricts the support to $\{v_{i-1},v_i\}$, while unused codes are infeasible. Segment affinity gives $a=\sigma(z)$. Conversely, any $z\in[v_{i-1},v_i]$ lifts to
\[
\lambda_{i-1}=\frac{v_i-z}{v_i-v_{i-1}},\qquad
\lambda_i=\frac{z-v_{i-1}}{v_i-v_{i-1}},\qquad
\mathbf g=\mathbf c_i,
\]
with all other weights zero. Continuity of PWL guarantees well-defined activation value at grid points. The proof is concluded by an induction argument. 
\end{proof}

\begin{remark}[Finite-precision QUBO qualification]
Theorem~\ref{thm:log_pw_exactness} does not imply that every fixed-point QUBO is exact for the continuous problem. Such equivalence additionally requires exact representability of an optimum and all slacks, penalty weights large enough to exclude every infeasible encoded assignment, and a globally optimal QUBO solution. Without these conditions, agreement with an exact MIP reference is an empirical result on the tested instances rather than a completeness guarantee.
\end{remark}

\section{Theoretical Properties and Proofs for Step-Env Model}\label{appendix_step_env_theory}

We first establish an error propagation bound, then prove soundness and convergence of Step-Env.

\subsection{General Layerwise Error Propagation Lemma}\label{subsec:general_propagation_lemma}

\begin{lemma}[General Layerwise Error Propagation in Deep Networks]\label{lem:general_layerwise_error}
Consider two $L$-layer feedforward composite maps $\mathcal{N} \triangleq T^L \circ \dots \circ T^1$ and $\mathcal{N}' \triangleq {T'}^L \circ \dots \circ {T'}^1$ from $\mathbb{R}^{n_0}$ to $\mathbb{R}^{n_L}$. Suppose that for each layer $l \in \{1, \dots, L\}$:
\begin{enumerate}
  \item The layer map ${T'}^l$ is Lipschitz continuous with constant $\alpha^l \ge 0$ under the $\ell_\infty$-norm, i.e., $\|{T'}^l(\mathbf{u}) - {T'}^l(\mathbf{v})\|_\infty \le \alpha^l \|\mathbf{u} - \mathbf{v}\|_\infty$;
  \item The local layerwise discrepancy is uniformly bounded by $\|T^l(\mathbf{u}) - {T'}^l(\mathbf{u})\|_\infty \le \delta^l$ for all reachable inputs $\mathbf{u}$.
\end{enumerate}
Then for identical initial inputs $\mathbf{h}^0 = {\mathbf{h}'}^0$, the accumulated end-to-end output error vector $e^L \triangleq \|\mathcal{N}(\mathbf{h}^0) - \mathcal{N}'(\mathbf{h}^0)\|_\infty$ satisfies:
\begin{equation}\label{eq:general_layerwise_propagation}
e^L \le \sum_{l=1}^L \left( \prod_{r=l+1}^L \alpha^r \right) \delta^l.
\end{equation}
\end{lemma}

\begin{proof}
Let $\mathbf h^l=T^l(\mathbf h^{l-1})$, ${\mathbf h'}^l={T'}^l({\mathbf h'}^{l-1})$, and $e^l=\|\mathbf h^l-{\mathbf h'}^l\|_\infty$. Adding and subtracting ${T'}^l(\mathbf h^{l-1})$ gives $e^l\le\delta^l+\alpha^l e^{l-1}$. Unrolling from $e^0=0$ yields~\eqref{eq:general_layerwise_propagation}.
\end{proof}

\subsection{Monotonic Reduction Details}\label{appendix:non_monotonic}
As introduced in Section~\ref{Arbitrary}, the bound propagation formula~\eqref{eq:step_env_bounds} natively assumes coordinate-wise non-decreasing activations. For general non-monotonic activations, this requirement is addressed via the classical Jordan decomposition specialized in Lemma~\ref{lem:mono_decomp} and the corresponding layerwise augmented mapping $(\widehat{\mathbf{W}}^{l+1}, \widehat{\sigma}^l)$.

\begin{proof}[Proof of Lemma~\ref{lem:mono_decomp}]
A locally Lipschitz function is absolutely continuous, and hence of bounded variation, on every compact interval. Thus the classical Jordan decomposition applies on each such interval; finite total variation on all of $\mathbb{R}$ is not required. For an explicit construction valid on $\mathbb{R}$, local absolute continuity gives
$\sigma(x)-\sigma(0)=\int_0^x\sigma'(t)\,{\rm d}t$,
where $\sigma'$ exists almost everywhere and is locally essentially bounded. Define $u^+=\max\{u,0\}$ and $u^-=\max\{-u,0\}$, and set
\[
\begin{aligned}
\sigma_+(x)&\triangleq\sigma(0)+\int_0^x(\sigma'(t))^+\,{\rm d}t,\\
\sigma_-(x)&\triangleq\int_0^x(\sigma'(t))^-\,{\rm d}t.
\end{aligned}
\]
Since $\sigma'=(\sigma')^+-(\sigma')^-$ almost everywhere, these functions satisfy~\eqref{eq:mono_decomp}. Their increments over any interval $[x,y]$ with $x<y$ are integrals of nonnegative functions, so both are non-decreasing. Local essential boundedness of the integrands also makes both components locally Lipschitz.
\end{proof}

The layerwise augmented activation and weight matrix are defined in~\eqref{eq:mono_reduction_weights}. For the subsequent analysis, the output activation $\sigma^L$ is the identity mapping $\mathrm{Id}_{n_L}$. Although the augmented formulation preserves the exact network mapping, separate interval bounds on the two components may yield looser propagated bounds.

\subsection{Rigorous Step-Env Bounds and Convergence Proofs}


The analysis and proofs below use Assumption~\ref{ass:step_env_refinements}. In view of the reduction technique in Appendix~\ref{appendix:non_monotonic}, we assume in addition that the activation is non-decreasing. 

Let $\Delta$ denote the maximal step segment width, \(\underline{\sigma}^l\) and \(\overline\sigma^l\) denote the lower and upper step envelopes at layer \(l\).

\begin{proof}[Proof of Lemma~\ref{lem:soundness}]
Starting from $\underline{\mathbf a}^0=\overline{\mathbf a}^0=\mathbf x$. Assume $\underline{\mathbf a}^l \leq {\mathbf a}^l \leq  \overline{\mathbf a}^l$ for some $l$. The sign-aware affine propagation preserves $\underline{\mathbf z}^{l+1}\le\mathbf z^{l+1}(\mathbf x)\le\overline{\mathbf z}^{l+1}$. By construction of step envelopes and monotonicity, 
\begin{multline*}
    \underline{\mathbf a}^{l+1} = \underline{\sigma}^{l+1}(\underline{\mathbf z}^{l+1}) \le\sigma^{l+1}(\underline{\mathbf z}^{l+1})
    \le\mathbf a^{l+1} \\
    \le\sigma^{l+1}(\overline{\mathbf z}^{l+1})
    \le\overline{\sigma}^{l+1}(\overline{\mathbf z}^{l+1})
    =\overline{\mathbf a}^{l+1}.
\end{multline*}
The assumption on domain coverage permits this construction at each layer $l$ and for all input $\mathbf x$. The proof is concluded by an inductive argument. 
\end{proof}

\begin{lemma}[Uniform Step-Env Bounding Error]\label{lem:step_env_endpoint_error}
For a fixed input \(\mathbf x\), define
\[
\mathcal{E}^l \triangleq \max \left\{ \|\overline{\mathbf a}^l - \mathbf a^l(\mathbf x)\|_\infty, \|\mathbf a^l(\mathbf x) - \underline{\mathbf a}^l\|_\infty \right\}.
\]
By the locally Lipschitz assumption, $\sigma^l_j$ is Lipschitz on the Step-Env domain $I^l_j$. Let $M^{l}$ denote the Lipschitz constant of vectorized $\sigma^l$ under the $\ell_\infty$-norm, and let $\mathbf{W}^l$ be equipped with the operator norm $\|\mathbf{W}^l\|_\infty$. Then
\begin{equation}\label{eq:corrected_endpoint_recursion}
\mathcal{E}^l \le M^l \|\mathbf W^l\|_\infty \mathcal{E}^{l-1} + M^l \Delta.
\end{equation}
Consequently, unrolling from $\mathcal{E}^0 = 0$ yields:
\begin{equation}\label{eq:corrected_endpoint_error}
\mathcal{E}^L \le \sum_{l=1}^L \left(\prod_{r=l+1}^L M^r \|\mathbf W^r\|_\infty \right) M^l \Delta \triangleq C\Delta,
\end{equation}
where $C \triangleq \sum_{l=1}^L \left(\prod_{r=l+1}^L M^r \|\mathbf W^r\|_\infty \right) M^l > 0$ bounds the error amplification uniformly over all $\mathbf x \in \mathcal{X}$ and all feasible Step-Env trajectories.
\end{lemma}

\begin{proof}
Sign-aware affine propagation admits multiplicative error bound $\|\mathbf W^l\|_\infty \mathcal{E}^{l-1}$. Lipschitz continuity bounds the propagated activation error by $M^l \|\mathbf W^l\|_\infty \mathcal{E}^{l-1}$ and the piecewise envelope approximation error by $M^l \Delta$. Summing the errors together yields~\eqref{eq:corrected_endpoint_recursion}; unrolling from $\mathcal{E}^0=0$ yields~\eqref{eq:corrected_endpoint_error}.
\end{proof}

\begin{proof}[Proof of Theorem~\ref{thm:convergence}]
Let $\min_{\mathbf{x}\in\mathcal{X}}m_p(\mathbf{x})$ be attained at $\mathbf{x}_0$, and $\min_{\mathbf{x}\in\mathcal{X}}\underline{m}_{p,\Delta}(\mathbf{x})$ attained at $\mathbf{x}_1$. 
Lemma~\ref{lem:soundness} implies 
\begin{equation}
    \begin{aligned}
        \underline{m}_{p,\Delta}(\mathbf{x}_0) \leq m_p(\mathbf{x}_0),\qquad
        m_p(\mathbf{x}_1) \leq \overline{m}_{p,\Delta}(\mathbf{x}_1).
    \end{aligned}
\end{equation}
By optimality, 
\begin{equation}
    \begin{aligned}
        m_p(\mathbf{x}_0) \leq m_p(\mathbf{x}_1),\qquad
        \underline{m}_{p,\Delta}(\mathbf{x}_1) \leq \underline{m}_{p,\Delta}(\mathbf{x}_0).
    \end{aligned}
\end{equation}
Putting them together, we obtain a chain of inequalities:
\[ \underline{m}_{p,\Delta}(\mathbf{x}_1) \leq \underline{m}_{p,\Delta}(\mathbf{x}_0) \leq m_p(\mathbf{x}_0) \leq m_p(\mathbf{x}_1) \leq \overline{m}_{p,\Delta}(\mathbf{x}_1). \]
This yields the lower-bound sub-optimality bound:
\[ 0 \le m_p(\mathbf{x}_0) - \underline{m}_{p,\Delta}(\mathbf{x}_1) \le m_p(\mathbf{x}_1) - \underline{m}_{p,\Delta}(\mathbf{x}_1). \]
Analogously, let $\min_{\mathbf{x}\in\mathcal{X}}\overline{m}_{p,\Delta}(\mathbf{x})$ be attained at $\mathbf{x}_2$. We obtain:
\[ \underline{m}_{p,\Delta}(\mathbf{x}_0) \leq m_p(\mathbf{x}_0) \leq m_p(\mathbf{x}_2) \leq \overline{m}_{p,\Delta}(\mathbf{x}_2) \leq \overline{m}_{p,\Delta}(\mathbf{x}_0), \]
which gives the upper-bound sub-optimality bound:
\[ 0 \le \overline{m}_{p,\Delta}(\mathbf{x}_2) - m_p(\mathbf{x}_0) \le \overline{m}_{p,\Delta}(\mathbf{x}_0) - m_p(\mathbf{x}_0). \]
By Lemma~\ref{lem:step_env_endpoint_error}, for any $\mathbf{x}\in\mathcal{X}$ and each neuron $j$ at output layer $L$, we have $|a_j^L(\mathbf{x}) - \underline{a}_j^L(\mathbf{x})| \le \mathcal{E}^L \le C\Delta$ and $|\overline{a}_j^L(\mathbf{x}) - a_j^L(\mathbf{x})| \le \mathcal{E}^L \le C\Delta$. 
Note that 
\[
m_p(\mathbf{x}) - \underline{m}_{p,\Delta}(\mathbf{x}) = a^L_{c_{\text{true}}}(\mathbf{x}) - \underline{a}^L_{c_{\text{true}}}(\mathbf{x}) + \overline{a}^L_p(\mathbf{x}) - a^L_p(\mathbf{x}).
\]
Defining $C_{\text{net}} \triangleq 2C$, 
by triangle inequality, the one-sided deviations from the exact margin satisfy:
\[ m_p(\mathbf{x}) - \underline{m}_{p,\Delta}(\mathbf{x}) \le 2\mathcal{E}^L \le 2C\Delta = C_{\text{net}}\Delta. \]
The inequality for the upper bound is analogous, establishing~\eqref{eq:optimal_bound_convergence}. Furthermore, the interval width satisfies $\overline{m}_{p,\Delta}(\mathbf{x}) - \underline{m}_{p,\Delta}(\mathbf{x}) \le 2C_{\text{net}}\Delta$.
\end{proof}

\subsection{Beyond Locally Lipschitz}\label{subsec:step_env_scope}

\paragraph{Continuous BV activations.}
Assumption~\ref{ass:step_env_refinements} requires locally Lipschitz activations. In fact, soundness and the convergence property survive for continuous activations of bounded variation (BV). Step-Env is still asymptotically complete, though the convergence rate need not be linear. 

Let $\sigma^l_j$ be continuous and of bounded variation on a compact interval covering the propagated pre-activation values. The classical Jordan decomposition~\cite{bell2016bv} yields two continuous non-decreasing components on this interval. Applying the augmented network construction in Appendix~\ref{appendix:non_monotonic}, we may therefore work with non-decreasing activations and the corresponding expanded weights. Each component is uniformly continuous on its compact interval. Unlike the locally Lipschitz case in Lemma~\ref{lem:mono_decomp}, this general BV decomposition uses positive and negative variation rather than integrals of the derivative.

Define $\omega^l:r \mapsto \sup\{|\sigma^l_j(u)-\sigma^l_j(v)|:u,v\in I^l_j,\ |u-v|\le r, 1\leq j \leq n_l\}$ and set
\[
e^0=0,\qquad e^l=\omega^l\!\left(\|\mathbf W^l\|e^{l-1}+\Delta\right).
\]
Every trajectory lies within $\|\mathbf W^l\| e^{l-1}+\Delta$ from the exact pre-activation, so $\mathcal{E}^l\le e^l$ inductively. Uniform continuity gives $e^L\to0$ as $\Delta\to0$. The convergence proofs still apply by replacing $C_{\mathrm{net}}\Delta$ with $e^L$ . For example, $\sigma(z)=\sqrt{\max\{z,0\}}$ is a continuous BV activation, and the error converges at the rate of $O(\sqrt\Delta)$.

\paragraph{Discontinuous BV activations.}
On a discontinuous activation function, soundness (Lemma~\ref{lem:soundness}) still survives, whereas convergence can fail, thus completeness is no longer guaranteed. 


\section{Monolithic QUBO Formulation Details}\label{appendix_qubo}
In a monolithic QUBO mapping, both the continuous network states $\mathbf{y} \triangleq [\mathbf{x}^\top, \underline{\mathbf{a}}^{1\top}, \overline{\mathbf{a}}^{1\top}, \dots, \underline{\mathbf{z}}^{L\top}, \overline{\mathbf{z}}^{L\top}]^\top \in \mathbb{R}^{N_{\text{cont}}}$ and the continuous inequality slack variables $\bm{\xi} \ge \mathbf{0}$ are discretized using fixed-point binary expansion schemes:
\begin{align}
y_i &= \ell_i + \sum_{k=1}^{K_{\text{y}}} t^{\text{y}}_k z_{ik}, \quad z_{ik} \in \{0, 1\}, \\
\xi_j &= \sum_{k=1}^{K_\xi} t^\xi_k \zeta_{jk}, \quad \zeta_{jk} \in \{0, 1\},
\end{align}
where $t^{\text{y}}_k$ and $t^\xi_k$ are predetermined precision bit weights. By concatenating all decision variables into a single unified binary decision vector $\mathbf{u} \triangleq [ (\mathbf{z}^{\text{y}})^\top, \bm{\zeta}^\top, \mathbf{w}^\top ]^\top$, the constrained verification optimization $\min_{\mathbf{y}, \mathbf{w}} \tilde{\mathbf{c}}^\top \mathbf{y}$ subject to $\mathbf{A}_{\text{eq}}\mathbf{y} = \mathbf{b}_0 + \mathbf{B}_{\text{eq}}\mathbf{w}$ and $\mathbf{C}\mathbf{y} + \bm{\xi} = \mathbf{d}_0 + \mathbf{D}\mathbf{w}$ (where $\tilde{\mathbf{c}} \in \mathbb{R}^{N_{\text{cont}}}$ extracts the margin $m_p(\mathbf{x})$) is compiled into an unconstrained monolithic QUBO objective:
\begin{align}
\min_{\mathbf{u} \in \{0, 1\}^N} \mathbf{u}^\top \mathbf{Q} \mathbf{u} \triangleq \; & \tilde{\mathbf{c}}^\top \mathbf{y}(\mathbf{u}) + P_{\text{eq}} \|\mathbf{A}_{\text{eq}}\mathbf{y}(\mathbf{u}) - \mathbf{b}_0 - \mathbf{B}_{\text{eq}}\mathbf{w}(\mathbf{u})\|_2^2 \nonumber \\
& + P_{\text{ineq}} \|\mathbf{C}\mathbf{y}(\mathbf{u}) + \bm{\xi}(\mathbf{u}) - \mathbf{d}_0 - \mathbf{D}\mathbf{w}(\mathbf{u})\|_2^2,
\end{align}
where $P_{\text{eq}}, P_{\text{ineq}} > 0$ are penalty multipliers enforcing equality and inequality constraints on Ising hardware.

\section{Details of the Benders Decomposition Framework}\label{appendix_benders}
Under Benders decomposition, the complicating integer variables $\mathbf{w}$ are separated from the continuous state variables $\mathbf{y} \in \mathbb{R}^{N_{\text{cont}}}$, where $\tilde{\mathbf{c}} \in \mathbb{R}^{N_{\text{cont}}}$ extracts the pairwise classification margin $m_p(\mathbf{x}) = \mathbf{e}_{\text{true}}^\top \mathbf{a}^L - \mathbf{e}_p^\top \mathbf{a}^L$. For any fixed candidate activation pattern $\mathbf{w}$, the continuous subproblem is defined as the linear program:
\begin{align}
\mathrm{SP}(\mathbf{w}): \quad \min_{\mathbf{y} \in \mathbb{R}^{N_{\text{cont}}}} \quad & \tilde{\mathbf{c}}^\top \mathbf{y} \label{eq:appendix_benders_sp_obj} \\
\text{subject to} \quad & \mathbf{A}_{\text{eq}}\mathbf{y} = \mathbf{b}_0 + \mathbf{B}_{\text{eq}}\mathbf{w}, \label{eq:appendix_benders_sp_eq} \\
& \mathbf{C}\mathbf{y} \le \mathbf{d}_0 + \mathbf{D}\mathbf{w}. \label{eq:appendix_benders_sp_ineq}
\end{align}
Let \(\bm{\pi}\) and \(\bm{\lambda} \ge \mathbf{0}\) denote the dual variables associated with the equality and inequality constraints, respectively. The dual subproblem is given by:
\begin{align}
\max_{\bm{\pi},\bm{\lambda}} \quad & (\mathbf{b}_0 + \mathbf{B}_{\text{eq}}\mathbf{w})^\top \bm{\pi} - (\mathbf{d}_0 + \mathbf{D}\mathbf{w})^\top \bm{\lambda} \label{eq:appendix_benders_dual_obj} \\
\text{subject to} \quad & \mathbf{A}_{\text{eq}}^\top \bm{\pi} - \mathbf{C}^\top \bm{\lambda} = \tilde{\mathbf{c}}, \label{eq:appendix_benders_dual_eq} \\
& \bm{\lambda} \ge \mathbf{0}. \label{eq:appendix_benders_dual_ineq}
\end{align}
By strong duality, if the primal subproblem is feasible and bounded, its optimal objective value matches that of the dual LP, yielding the valid optimality cut:
\begin{equation}
\theta \ge (\mathbf{b}_0 + \mathbf{B}_{\text{eq}}\mathbf{w})^\top \bm{\pi}^k - (\mathbf{d}_0 + \mathbf{D}\mathbf{w})^\top \bm{\lambda}^k.
\end{equation}
Conversely, when $\mathrm{SP}(\mathbf{w})$ is infeasible for a candidate assignment $\mathbf{w}^*$, by Farkas' Lemma of Alternatives, there exists a certificate direction $(\hat{\bm{\pi}}, \hat{\bm{\lambda}})$ satisfying:
\begin{equation}\label{eq:farkas_certificate}
\mathbf{A}_{\text{eq}}^\top \hat{\bm{\pi}} + \mathbf{C}^\top \hat{\bm{\lambda}} = \mathbf{0}, \quad \text{with } \hat{\bm{\lambda}} \ge \mathbf{0},
\end{equation}
such that at the infeasible assignment $\mathbf{w}^*$,
\begin{equation}
\hat{\bm{\pi}}^\top (\mathbf{b}_0 + \mathbf{B}_{\text{eq}}\mathbf{w}^*) + \hat{\bm{\lambda}}^\top (\mathbf{d}_0 + \mathbf{D}\mathbf{w}^*) < 0.
\end{equation}
For any primal-feasible configuration $(\mathbf{y}, \mathbf{w})$, multiplying~\eqref{eq:appendix_benders_sp_eq} by $\hat{\bm{\pi}}$ and~\eqref{eq:appendix_benders_sp_ineq} by $\hat{\bm{\lambda}} \ge \mathbf{0}$ yields:
\begin{equation}
\begin{aligned}
0 &= (\mathbf{A}_{\text{eq}}^\top \hat{\bm{\pi}} + \mathbf{C}^\top \hat{\bm{\lambda}})^\top \mathbf{y} = \hat{\bm{\pi}}^\top \mathbf{A}_{\text{eq}}\mathbf{y} + \hat{\bm{\lambda}}^\top \mathbf{C}\mathbf{y} \\
&\le \hat{\bm{\pi}}^\top (\mathbf{b}_0 + \mathbf{B}_{\text{eq}}\mathbf{w}) + \hat{\bm{\lambda}}^\top (\mathbf{d}_0 + \mathbf{D}\mathbf{w}).
\end{aligned}
\end{equation}
Consequently, every feasible candidate $\mathbf{w}$ satisfies the Farkas feasibility cut $\hat{\bm{\pi}}^\top (\mathbf{b}_0 + \mathbf{B}_{\text{eq}}\mathbf{w}) + \hat{\bm{\lambda}}^\top (\mathbf{d}_0 + \mathbf{D}\mathbf{w}) \ge 0$, which strictly eliminates $\mathbf{w}^*$.

The Benders master problem is then formulated as:
\begin{align}
\min_{\mathbf{w}, \theta} \quad & \theta \label{eq:appendix_benders_master_obj} \\
\text{subject to} \quad & \mathbf{w} \in \mathcal{B}, \label{eq:appendix_benders_master_set} \\
& \theta \ge \mathbf{b}_0^\top \bm{\pi}^k - \mathbf{d}_0^\top \bm{\lambda}^k \nonumber \\
& \quad + \mathbf{w}^\top (\mathbf{B}_{\text{eq}}^\top \bm{\pi}^k - \mathbf{D}^\top \bm{\lambda}^k), \quad \forall k \in \mathcal{K}_{\text{opt}}, \label{eq:appendix_benders_master_opt_cut} \\
& \mathbf{b}_0^\top \hat{\bm{\pi}}^j + \mathbf{d}_0^\top \hat{\bm{\lambda}}^j \nonumber \\
& \quad + \mathbf{w}^\top (\mathbf{B}_{\text{eq}}^\top \hat{\bm{\pi}}^j + \mathbf{D}^\top \hat{\bm{\lambda}}^j) \ge 0, \quad \forall j \in \mathcal{K}_{\text{feas}}, \label{eq:appendix_benders_master_feas_cut}
\end{align}
where \(\mathcal{K}_{\text{opt}}\) and \(\mathcal{K}_{\text{feas}}\) index the dual extreme points (optimality cuts) and Farkas certificates (feasibility cuts) generated in previous iterations, respectively. The Benders decomposition algorithm proceeds by alternating between solving the master problem over \(\mathbf{w} \in \mathcal{B}\) to obtain a candidate configuration and solving the continuous linear subproblem classically to generate new cuts.

\subsection{Hybrid Benders Verification}\label{appendix_benders_full_algo}

Algorithm~\ref{alg:hybrid_benders_full} outlines this procedure in detail. For an infeasible candidate, the selected strategy adds a Farkas cut~\eqref{eq:appendix_benders_master_feas_cut}, a Helly cut~\eqref{eq:guarded_cut}, or both. Feasible patterns are evaluated once and eliminated via a no-good cut. The algorithm terminates upon certified master infeasibility.

\begin{algorithm}[tb]
\caption{Hybrid Benders Verification}
\label{alg:hybrid_benders_full}
\begin{algorithmic}[1]\small
\Require Network $\mathcal N$, input set $\mathcal X$, ground-truth class $c_{\text{true}}$, competing class $p \neq c_{\text{true}}$, cut strategy.
\Ensure Robustness verdict and a witness if non-robust.
\State Initialize cut pool $\mathcal M\gets\varnothing$.
\While{the master admits a pattern $\mathbf{w}$ satisfying $\mathcal M$}
    \State Test $\mathbf{w}$ using Algorithm~\ref{alg:helly_cut_generation}.
    \If{$\mathbf{w}$ is infeasible}
        \State Add cuts for the selected strategy to $\mathcal M$.
    \Else
        \State Solve $\gamma=\min_{\mathbf x\in R_{\mathbf{w}}} m_p(\mathbf x)$, with minimizer $\mathbf x^\star$.
        \State If $\gamma\le0$, \Return \textsc{Violated}, $\mathbf x^\star$.
        \State Add a no-good cut excluding $\mathbf{w}$ to $\mathcal M$.
    \EndIf
\EndWhile
\State \Return \textsc{Robust}.
\end{algorithmic}
\end{algorithm}

\subsection{Proofs for Benders Complexity}\label{appendix_benders_proofs}

We prove the results of Section~\ref{subsec:benders_complexity}.

\subsubsection{Proof of Prefix Region Convexity and Counting Bound}

\begin{lemma}[Convexity of Consistent Prefix Regions for Piecewise Activations]\label{lem:region_convexity}
Suppose input set $\mathcal{X} \subseteq \mathbb{R}^{n_0}$ is convex. For any neural network with general piecewise (PWL or PWC) activations and any activation prefix $p = (\mathbf{w}^1, \dots, \mathbf{w}^t)$, its consistent region $R_p = \{ \mathbf{x} \in \mathcal{X} \mid \mathbf{x} \text{ consistent with } p \}$ is a convex set in $\mathbb{R}^{n_0}$. Furthermore, over $R_p$, intermediate pre-activations and activations at layer $r \le t$ are affine functions of input $\mathbf{x}$ (for PWC activations, intermediate activations evaluate to constants and $R_p$ degenerates to a single point).
\end{lemma}

\begin{proof}
Induct on layer depth $r \in \{0, 1, \dots, t\}$. The base case $R_\varnothing = \mathcal{X}$ is convex by assumption, and $\mathbf{a}^0(\mathbf{x}) = \mathbf{x}$ is affine. Assume inductively that $R_{p_{<r}}$ is convex and intermediate activations $\mathbf{a}^{r-1}(\mathbf{x})$ are affine on $R_{p_{<r}}$. At layer $r$, pre-activations $\mathbf{z}^r(\mathbf{x}) = \mathbf{W}^r \mathbf{a}^{r-1}(\mathbf{x}) + \mathbf{b}^r$ remain affine in $\mathbf{x}$. Fixing prefix decisions $\mathbf{w}^r$ selects a segment interval $[v_{i-1}, v_i]$ for each neuron $j \in \{1, \dots, n_r\}$, imposing affine inequalities $v_{i-1} \le z_j^r(\mathbf{x}) \le v_i$ that preserve convexity of $R_{p_{\le r}}$. For PWL activations, the activation on each segment is affine ($\sigma(z) = \alpha_i z + \gamma_i$), preserving affinity of $\mathbf{a}^r(\mathbf{x})$. For PWC activations, the activation evaluates to a constant vector $\gamma_i$ (affine with slope zero), under which the input region $R_p$ degenerates to a single point.
\end{proof}

\begin{lemma}[Reachable Prefix Completion and Counting Bound]\label{lem:prefix_count_proof}
For any depth $t \in \{0, \dots, L\}$ and any exact reachable prefix $p \in \mathcal{P}_t$, there exists at least one complete feasible activation pattern $a \in \mathcal{A}_{\mathcal{X}}$ whose prefix matches $p$. Consequently, $|\mathcal{P}_t| \le K_{\mathcal{X}} \triangleq |\mathcal{A}_{\mathcal{X}}|$.
\end{lemma}

\begin{proof}
For any $\mathbf x\in R_p$, a forward pass through the remaining layers extends $p$ to a complete feasible pattern. Thus truncation $\pi_t:\mathcal A_{\mathcal X}\to\mathcal P_t$ is surjective, giving $|\mathcal P_t|\le K_{\mathcal X}$ and $\sum_{t=0}^{L-1}|\mathcal P_t|\le LK_{\mathcal X}$.
\end{proof}

\begin{proof}[Proof of Lemma~\ref{lem:parent_feasibility}]
Since $R^0(\mathbf{w})=\mathcal X\ne\varnothing$ and $R^L(\mathbf{w})=\varnothing$, $t^\star$ exists. Its minimality gives $R_{p^\star}=R^{t^\star-1}(\mathbf{w})\ne\varnothing$, whereas $R^{t^\star}(\mathbf{w})=\varnothing$.
\end{proof}

\begin{proof}[Proof of Lemma~\ref{lem:helly_compression}]
Fix reachable parent context $p \in \mathcal{P}_{l-1}$. By Lemma~\ref{lem:region_convexity}, the consistent input domain $R_p \subseteq \mathbb{R}^{n_0}$ is convex, and pre-activations at layer $l$ are affine: $z_j^l(\mathbf{x}) = (\bm{\alpha}_{p,j}^l)^\top \mathbf{x} + \gamma_{p,j}^l$. 

Let $\mathbf{x}_0 \in R_p$, and let $V_p$ denote the linear subspace parallel to the affine hull $\operatorname{aff}(R_p)$. Projecting each vector $\bm{\alpha}_{p,j}^l$ onto $V_p$, the active variations are captured by the conflict subspace $U_p^l \triangleq \operatorname{span}\{\Pi_{V_p} \bm{\alpha}_{p,j}^l\}_{j=1}^q \subseteq V_p$. By definition, $d \triangleq \dim(U_p^l) = \rho_p^l \le \dim(V_p) \le n_0$. Since every $\mathbf{x} - \mathbf{x}_0 \in V_p$ can be decomposed as $\mathbf{u} + \mathbf{u}^\perp$ with $\mathbf{u} \in U_p^l$ and $\mathbf{u}^\perp \in (U_p^l)^\perp$, we have:
\[
(\bm{\alpha}_{p,j}^l)^\top \mathbf{x} + \gamma_{p,j}^l = (\Pi_{V_p} \bm{\alpha}_{p,j}^l)^\top \mathbf{u} + \tilde{\gamma}_j,
\]
where $\tilde{\gamma}_j = (\bm{\alpha}_{p,j}^l)^\top \mathbf{x}_0 + \gamma_{p,j}^l$. Hence, feasibility over $R_p$ depends solely on the projection $\mathbf{u} \in U_p^l \cong \mathbb{R}^d$.

Defining the non-empty convex set $Q_p \triangleq \Pi_{U_p^l}(R_p - \mathbf{x}_0)$ and the projected affine halfspaces $\widetilde{H}_j^{s_j} \triangleq \{ \mathbf{u} \in U_p^l \mid (-1)^{1-s_j} ((\Pi_{V_p} \bm{\alpha}_{p,j}^l)^\top \mathbf{u} + \tilde{\gamma}_j) \ge 0 \}$, infeasibility is equivalent to an empty intersection in $\mathbb{R}^d$:
\[
Q_p \cap \bigcap_{j=1}^q \widetilde{H}_j^{s_j} = \varnothing.
\]
By Helly's theorem in dimension $d$, the collection $\{Q_p, \widetilde{H}_1^{s_1}, \dots, \widetilde{H}_q^{s_q}\}$ contains an empty sub-collection of at most $d + 1$ sets. If $Q_p$ belongs to this sub-collection, at most $d = \rho_p^l$ halfspaces suffice to conflict with $Q_p$; otherwise, at most $d + 1 = \rho_p^l + 1$ halfspaces have an empty intersection on their own. In either case, the selected halfspace indices form a certificate $S \subseteq \{1, \dots, q\}$ of size $|S| \le \rho_p^l + 1$ such that $R_p \cap \bigcap_{i \in S} H_i^{s_i} = \varnothing$.
\end{proof}

\begin{proof}[Proof of Theorem~\ref{thm:benders_complexity}]
Charge each failure to its earliest infeasible layer $l$, reachable parent $p$, and signed certificate $(S,s_S)$ (Lemmas~\ref{lem:parent_feasibility} and~\ref{lem:helly_compression}). The retained cut~\eqref{eq:guarded_cut} excludes every later candidate with the same context and signed certificate, so no charge repeats. For each $(p,l)$ there are at most $\sum_{r=1}^{\rho_p^l+1}2^r\binom{q_p^l}r$ certificates. Summing and using $|\mathcal P_{l-1}|\le K_{\mathcal X}$ gives~\eqref{eq:n_fail_exact_sum}; $q_p^l\le P$ and constant $\rho_p^l\le\rho$ give~\eqref{eq:n_fail_poly_bound}.
\end{proof}

\subsubsection{Shallow-Network Complexity}\label{appendix_proof_shallow_poly}
\begin{proof}[Proof of Corollary~\ref{cor:shallow_poly}]
Lemma~\ref{lem:feasible_geometry} gives $K_{\mathcal X}=\mathcal O(P^{Ln_0})$, while~\eqref{eq:rho_input_bound} gives $\rho\le n_0$. Substituting into Theorem~\ref{thm:benders_complexity} yields
\[
\begin{aligned}
T_{\mathrm{Benders}}&\le K_{\mathcal X}+N_{\mathrm{fail}}\\
&=\mathcal O\!\left(LP^{(L+1)n_0+1}\right),
\end{aligned}
\]
which is polynomial for fixed $L,n_0$ and becomes $\mathcal O(P^{2n_0+1})$ for $L=1$. Algorithm~\ref{alg:helly_cut_generation} uses at most $L+P$ LP queries per candidate, so the total continuous oracle query count is also polynomial.
\end{proof}

\section{Details and Bounds for Pruning-Induced Robustness Transfer}\label{appendix_pruning_details}

\subsection{Proof of Theorem~\ref{thm:transfer} (Pruning-Induced Robustness Transfer)}\label{appendix_proof}
\begin{proof}
Lemma~\ref{lem:margin_stability} and $\|f-g\|_\infty\le\tau$ give $|m(f)-m(g)|\le2\tau$ throughout the perturbation ball. Taking infima yields $\Phi_g-2\tau\le\Phi_f\le\Phi_g+2\tau$; substituting $L_g\le\Phi_g\le U_g$ proves~\eqref{eq:transfer}.
\end{proof}

\subsection{Proof of Lemma~\ref{lem:margin_stability} (Margin Stability)}
\begin{proof}
The maximum is $1$-Lipschitz in $\|\cdot\|_\infty$, hence
\[
\begin{aligned}
|m(\mathbf a+\mathbf b)-m(\mathbf a)|
&\le |b_{c_{\text{true}}}|+\max_{p\ne c_{\text{true}}}|b_p|\\
&\le 2\|\mathbf b\|_\infty.
\end{aligned}
\]
\end{proof}

\subsection{Dataset-Level Certified Accuracy Bounds}
For a dataset $S$, any verifier providing bounds $(L_g, U_g)$ for the pruned model $g$ induces lower and upper bounds on the certified accuracy of the original model $f$:
\begin{align}
\underline{\mathrm{CA}}_f(\varepsilon) &\triangleq \frac{1}{|S|}\sum_{(\mathbf{x},y)\in S}\mathbf{1}\{L_g(\mathbf{x};\varepsilon)>2\tau\}, \\
\overline{\mathrm{CA}}_f(\varepsilon) &\triangleq 1-\frac{1}{|S|}\sum_{(\mathbf{x},y)\in S}\mathbf{1}\{U_g(\mathbf{x};\varepsilon)\le -2\tau\},
\end{align}
which satisfy the sandwich inequality:
\begin{equation}
\underline{\mathrm{CA}}_f(\varepsilon) \le \mathrm{CA}_f(\varepsilon) \le \overline{\mathrm{CA}}_f(\varepsilon).
\end{equation}
Therefore, robustness verification on the simplified model $g$ controls the certification bounds for the original model $f$ up to the safety margin $2\tau$.

\subsection{Closed-Form Derivation of the Uniform Residual Bound \texorpdfstring{$\tau$}{tau}}
To establish the uniform residual bound condition $\|\mathbf{r}(\mathbf{x} + \bm{\delta})\|_{\infty} \le \tau$ for all $\|\bm{\delta}\|_\infty \le \varepsilon$, we apply the general layerwise error propagation framework of Lemma~\ref{lem:general_layerwise_error}.

Define the layer maps of the original network $f$ and pruned network $g$:
\[
T^l(\mathbf{h}) \triangleq \sigma(\mathbf{W}^l \mathbf{h}+\mathbf{b}^l), \qquad {T'}^l(\mathbf{h}) \triangleq \sigma({\mathbf{W}'}^l \mathbf{h}+\mathbf{b}^l).
\]
Because $\sigma$ is $1$-Lipschitz under $\|\cdot\|_\infty$, the pruned layer map ${T'}^l$ has Lipschitz constant $\alpha^l = \|{\mathbf{W}'}^l\|_{\infty\to\infty}$. For identical intermediate inputs $\mathbf{h}^{l-1}$, the local pruning-injected error is $\|T^l(\mathbf{h}^{l-1}) - {T'}^l(\mathbf{h}^{l-1})\|_\infty \le \|\Delta \mathbf{W}^l\|_{\infty\to\infty} \|\mathbf{h}^{l-1}\|_\infty$.

Let $H^{l-1}$ denote a computable upper bound on intermediate activation norm $\|\mathbf{h}^{l-1}(\mathbf{x}+\bm{\delta})\|_\infty \le H^{l-1}$ over the perturbation ball (obtained via IBP, CROWN, or LP relaxations). Setting the local injection bound $\delta^l = \|\Delta \mathbf{W}^l\|_{\infty\to\infty} H^{l-1}$ and applying Lemma~\ref{lem:general_layerwise_error} directly yields the explicit closed-form residual bound:
\begin{equation}\label{eq:tau_closed}
\tau \le \sum_{l=1}^L \left(\prod_{j=l+1}^L \|{\mathbf{W}'}^j\|_{\infty\to\infty}\right) \|\Delta \mathbf{W}^l\|_{\infty\to\infty} H^{l-1},
\end{equation}
guaranteeing that $\|\mathbf{r}(\mathbf{x}+\bm{\delta})\|_\infty \le \tau$ holds uniformly for all $\|\bm{\delta}\|_\infty \le \varepsilon$.

\section{Detailed Experimental Setup}\label{appendix_exp_setup}

\paragraph{Datasets and networks.} For the global-QUBO experiments, we train a shallow network on Iris \cite{fisher1936use} and evaluate robustness over 100 test samples (see Figure~\ref{fig:iris_vulnerability} in Appendix~\ref{appendix_numerical_results} for an empirical scatterplot of clean vs.\ perturbed data distributions and misclassified boundary samples). For the hybrid Benders experiments, we train a larger network on \texttt{make\_moons} and again evaluate 100 test samples. Both networks achieve 100\% clean test accuracy, allowing the evaluation to isolate verification behavior from classification error. We consider \(\ell_\infty\)-bounded perturbations, sweeping \(\varepsilon\) from 0.1 to 0.6, step size 0.1 for the ReLU global-QUBO experiment and from 0.1 to 1.0, step size 0.1 for the Hardtanh and Sigmoid experiments. For the Benders evaluation, we sweep \(\varepsilon\) from 0.05 to 0.50, step size 0.05.

\paragraph{Methods.} We compare the proposed formulations with a comprehensive suite of optimization- and reachability-based baseline verifiers. Table~\ref{tab:comparison_methods} summarizes these baseline methods, including their supported activation functions and formal guarantee properties (soundness and completeness). MIP-Gurobi serves as the exact reference. We solve the global QUBO using both Gurobi (QUBO-Gurobi) and a Coherent Ising Machine (QUBO-CIM). Interval Bound Propagation is used to tighten neuron bounds before encoding. For nonlinear Sigmoid networks, the Step-Env model uses a five-segment piecewise-constant enclosure. In the scalability experiment, the proposed Benders method separates the discrete activation decisions from the continuous verification subproblem.

\begin{table}[t]
\centering
\resizebox{\columnwidth}{!}{
\begin{tabular}{llcc}
\toprule
\textbf{Method} & \textbf{Activation} & \textbf{Sound} & \textbf{Complete} \\
\midrule
\multicolumn{4}{c}{\textbf{Optimization-Based Verifiers}} \\
\midrule
BaB \cite{bunel2018unified}         & ReLU & \checkmark & \checkmark \\
ConvDual \cite{wong2018provable}   & ReLU & \checkmark & $\times$   \\
Duality \cite{dvijotham2018dual}    & monotonic & \checkmark & $\times$   \\
NSVerify \cite{lomuscio2017approach}   & ReLU & \checkmark & \checkmark \\
Reluplex \cite{katz2017reluplex}   & ReLU & \checkmark & \checkmark \\
Sherlock \cite{dutta2017output}   & ReLU & \checkmark & $\times$   \\
\midrule
\multicolumn{4}{c}{\textbf{Reachability-Based Verifiers}} \\
\midrule
AI2 \cite{gehr2018ai2}        & ReLU & \checkmark & $\times$   \\
DLV \cite{huang2017safety}        & any & \checkmark & $\times$   \\
ExactReach \cite{xiang2017reachable} & ReLU & \checkmark & \checkmark \\
FastLin \cite{weng2018towards}    & ReLU & \checkmark & $\times$   \\
FastLip \cite{weng2018towards}    & ReLU & \checkmark & $\times$   \\
MaxSens \cite{xiang2018output}    & monotonic & \checkmark & $\times$   \\
Neurify \cite{wang2018efficient}    & ReLU & \checkmark & $\times$   \\
ReluVal \cite{wang2018formal}    & ReLU & \checkmark & $\times$   \\
\bottomrule
\end{tabular}
}
\caption{Comparison of Verification Methods}
\label{tab:comparison_methods}
\end{table}

\paragraph{Metrics.} For global QUBO verification, we report the number of vulnerable samples for which a valid adversarial counterexample is found. This metric directly tests whether the QUBO encoding reproduces the decisions of the exact verifier. We separately report wall-clock or evolutionary solver time and the number of Ising spins. For Benders decomposition, we report certified accuracy, namely the percentage of test samples whose robustness is formally certified at a given perturbation budget.

\section{Complete Numerical Experimental Results}
\label{appendix_numerical_results}

In this section, we present the complete numerical results and empirical visualizations for all global-QUBO and hybrid Benders verification experiments. Figure~\ref{fig:iris_vulnerability} displays empirical data perturbation results on the Iris dataset under $\varepsilon = 0.5$. Tables~\ref{tab:relu_results}, \ref{tab:sigmoid_results}, and \ref{tab:benders_results} list the detailed vulnerability counts, certified accuracies, solver times, and Ising spin resources across the full sweeps of perturbation budgets $\varepsilon$.

\begin{figure*}[t]
    \centering
\includegraphics[width=0.9\textwidth]{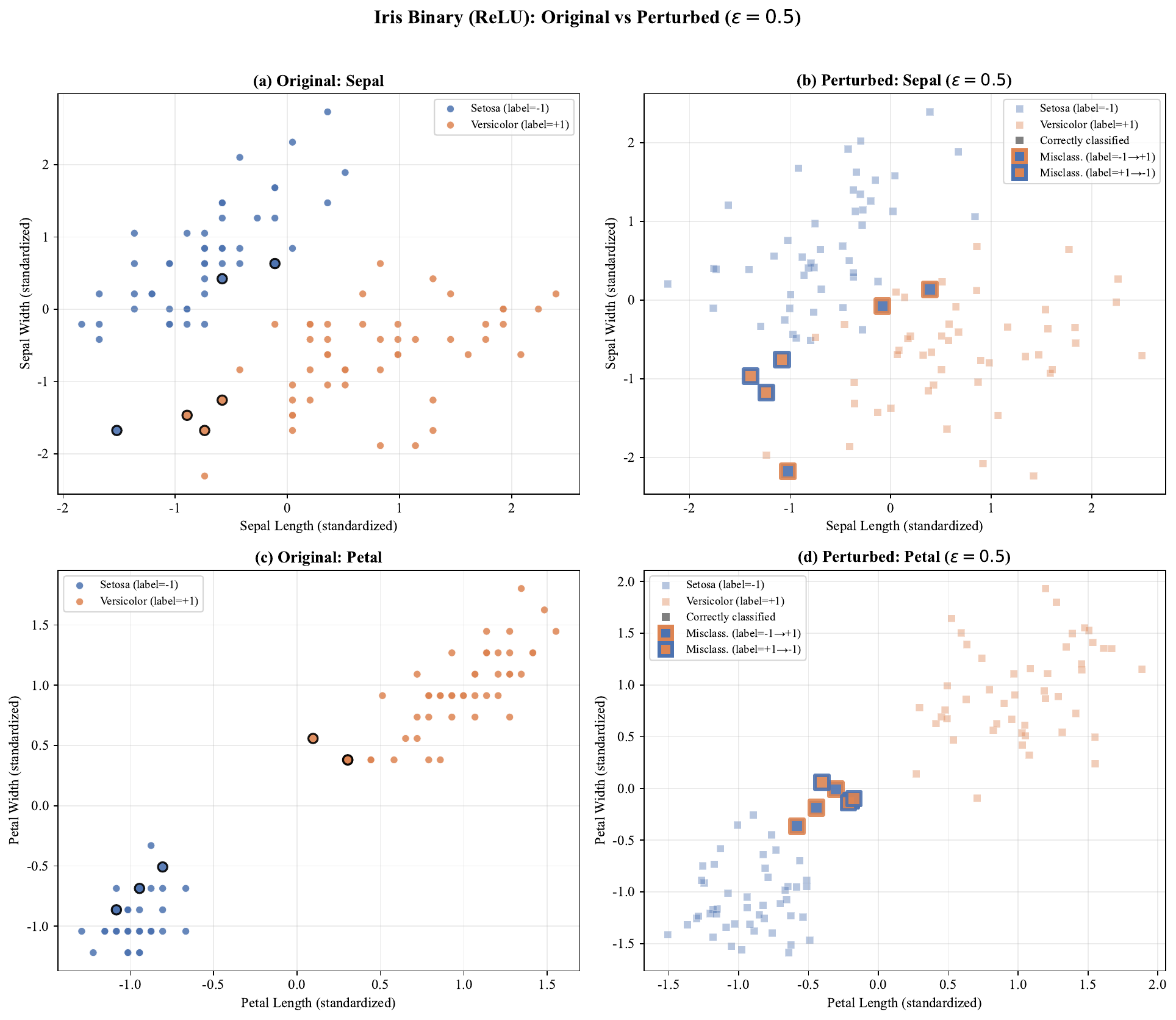} 
    \caption{Empirical experimental results on the Iris binary classification benchmark using a ReLU network under $\ell_\infty$ perturbation ($\varepsilon = 0.5$). Panels (a, c): original clean data in sepal space (length vs.\ width) and petal space (length vs.\ width), with the 6 samples susceptible to adversarial attack highlighted by bold black circles. Panels (b, d): the same data after $\varepsilon = 0.5$ perturbation. Six samples flip their classification predictions and are marked as squares with opposite-class borders (blue with orange border for true label $-1$ and predicted $+1$; orange with blue border for true label $+1$ and predicted $-1$). The ReLU network achieves 100\% accuracy on clean data and 94\% under perturbation.} 
    \label{fig:iris_vulnerability} 
\end{figure*}

\begin{table*}[t]
	\centering
	\resizebox{\textwidth}{!}{
	\begin{tabular}{ll cccccc}
	\toprule
	\multicolumn{2}{c}{\textbf{Perturbation Budget} \(\epsilon\)} & \textbf{0.1}   & \textbf{0.2}   & \textbf{0.3}   & \textbf{0.4}   & \textbf{0.5}   & \textbf{0.6} \\
	\midrule
	\multicolumn{1}{l}{\multirow{15}{*}{\makecell[l]{Vulnerable\\Samples\\$\uparrow$}}} & AI2 \cite{gehr2018ai2}   & 0   & 0   & 0   & 3    & 6    & 12 \\
		& DLV \cite{huang2017safety}   & 0   & 0   & 0   & 2    & 6    & 12 \\
		& ExactReach \cite{xiang2017reachable} & 0   & 0   & 0   & 2    & 6    & 12 \\
		& MaxSens \cite{xiang2018output} & 0   & 0   & 0   & 3    & 6    & 12 \\
		& Neurify \cite{wang2018efficient} & 0   & 0   & 0   & 2    & 6    & 12 \\
\cmidrule{2-8}          & BaB \cite{bunel2018unified}   & 0   & 0   & 0   & 2    & 6    & 12 \\
		& ConvDual \cite{wong2018provable} & 0   & 0   & 0   & 2    & 6    & 12 \\
		& Duality \cite{dvijotham2018dual} & 0   & 0   & 0   & 2    & 6    & 12 \\
		& ILP \cite{bastani2016measuring} & 0   & 0   & 0   & 2    & 6    & 11 \\
		& NSVerify \cite{lomuscio2017approach} & 0   & 0   & 0   & 2    & 6    & 12 \\
		& Reluplex \cite{katz2017reluplex} & 0   & 0   & 0   & 2    & 6    & 12 \\
		& Sherlock \cite{dutta2017output} & 0   & 0   & 0   & 2    & 6    & 12 \\
\cmidrule{2-8}          & MIP-Gurobi & 0   & 0   & 0   & 2    & 6    & 12 \\
		& QUBO-Gurobi & 0   & 0   & 0   & 2    & 6    & 12 \\
		& QUBO-CIM & 0   & 0   & 0   & 2    & 6    & 12 \\
	\midrule
	\multicolumn{1}{l}{\multirow{3}{*}{\makecell[l]{Time (ms)\\$\downarrow$}}} & MIP-Gurobi & 0.526 & 0.305 & 0.33  & 0.361 & 0.369 & 0.443 \\
		& QUBO-Gurobi & 1137  & \multicolumn{5}{c}{Timeout at 5000 ms} \\
		& QUBO-CIM (evolutionary time) & 2.496 & 2.602 & 2.755 & 11.966 & 2.799 & 8.489 \\
	\midrule
	\multicolumn{1}{l}{\multirow{3}{*}{\makecell[l]{\textbf{Ising Spins}}}} & Average & 198.03 & 240.89 & 255.72 & 300.25 & 320.63 & 338.28 \\
		& Max & 304 & 379 & 384 & 494 & 497 & 531 \\
		& Min & 102	& 116 & 120 & 139 & 143 & 196 \\
	\bottomrule
	\end{tabular}
	}
\caption{Complete Numerical Verification Results for the ReLU Network}
\label{tab:relu_results}
\end{table*}

\begin{table*}[t]
	\centering
	\resizebox{\textwidth}{!}{
	\begin{tabular}{ll cccccccccc}
	\toprule
	\multicolumn{2}{c}{\textbf{Perturbation Budget} \(\epsilon\)} & \textbf{0.1}   & \textbf{0.2}   & \textbf{0.3}   & \textbf{0.4}   & \textbf{0.5}   & \textbf{0.6} & \textbf{0.7} & \textbf{0.8} & \textbf{0.9} & \textbf{1.0} \\
	\midrule
	\multicolumn{1}{l}{\multirow{5}{*}{\makecell[l]{Vulnerable\\Samples\\$\uparrow$}}} & Duality \cite{dvijotham2018dual} & 0   & 0   & 0   & 2    & 5    & 10 & 26 & 53 & 67 & 79 \\
\cmidrule{2-12}          & MaxSens \cite{xiang2018output} & 0   & 0   & 0   & 2    & 5    & 10 & 26 & 53 & 67 & 79 \\
\cmidrule{2-12}          & MIP-Gurobi & 0   & 0   & 0   & 2    & 5    & 10 & 26 & 53 & 67 & 79 \\
		& QUBO-Gurobi & 0   & 0   & 0   & 2    & 5    & 10 & 26 & 53 & 67 & 79 \\
		& QUBO-CIM & 0   & 0   & 0   & 2    & 5    & 10 & 26 & 52 & 67 & 79 \\
	\midrule
	\multicolumn{1}{l}{\multirow{3}{*}{\makecell[l]{Time (ms)\\$\downarrow$}}} & MIP-Gurobi & 0.983 & 0.885 & 1.000 & 0.927 & 1.179 & 1.034 & 1.104 & 1.169 & 1.066 & 1.227 \\
		& QUBO-Gurobi & 3367  & \multicolumn{9}{c}{Timeout at 5000 ms} \\
		& QUBO-CIM (evolutionary time) & 2.240 & 2.584 & 2.535 & 2.545 & 4.573 & 3.295 & 4.379 & 3.683 & 4.370 & 3.078 \\
	\midrule
	\multicolumn{1}{l}{\multirow{3}{*}{\makecell[l]{\textbf{Ising Spins}}}} & Average & 64.37 & 74.55 & 79.92 & 85.38 & 90.67 & 91.88 & 96.88 & 97.64 & 100.49 & 103.49 \\
		& Max & 69 & 79 & 84 & 88 & 94 & 94 & 99 & 100 & 104 & 107 \\
		& Min & 60 & 71 & 78 & 83 & 89 & 89 & 95 & 95 & 99 & 101 \\
	\bottomrule
	\end{tabular}
    }
\caption{Complete Numerical Verification Results for the Sigmoid Network}
\label{tab:sigmoid_results}
\end{table*}

\begin{table*}[t]
  \centering
    \resizebox{\textwidth}{!}{
    \begin{tabular}{ll cccccccccc}
    \toprule
    \multicolumn{2}{c}{\multirow{2}{*}{\makecell[l]{\textbf{Certified Accuracy} (\%)\\$\uparrow$}}} & \multicolumn{10}{c}{\textbf{Perturbation Budget} \(\epsilon\)} \\
\cmidrule{3-12}    \multicolumn{2}{c}{} & \textbf{0.05}  & \textbf{0.1}   & \textbf{0.15}  & \textbf{0.2}   & \textbf{0.25}  & \textbf{0.3}   & \textbf{0.35}  & \textbf{0.4}   & \textbf{0.45}  & \textbf{0.5} \\
    \midrule
    \multicolumn{1}{c|}{\multirow{12}[6]{*}{ReLU}} & AI2 \cite{gehr2018ai2}   & 99    & 98    & 97    & 95    & 93    & 89    & 70    & 46    & 30    & 25 \\
          & FastLin \cite{weng2018towards} & 99    & 98    & 97    & 95    & 93    & 89    & 70    & 46    & 30    & 25 \\
          & FastLip \cite{weng2018towards} & 99    & 98    & 97    & 95    & 93    & 89    & 70    & 46    & 30    & 25 \\
          & MaxSens \cite{xiang2018output} & 98    & 89    & 64    & 37    & 15    & 9     & 6     & 2     & 0     & 0 \\
          & Neurify \cite{wang2018efficient} & 100   & 99    & 97    & 97    & 95    & 92    & 85    & 69    & 49    & 36 \\
          & ReluVal \cite{wang2018formal} & 99    & 98    & 97    & 94    & 89    & 81    & 64    & 43    & 31    & 25 \\
\cmidrule{2-12}          & BaB \cite{bunel2018unified}   & 100   & 99    & 97    & 97    & 95    & 92    & 85    & 68    & 49    & 36 \\
          & NSVerify \cite{lomuscio2017approach} & 100   & 99    & 97    & 97    & 95    & 92    & 85    & 69    & 49    & 36 \\
          & Reluplex \cite{katz2017reluplex} & 100   & 99    & 97    & 97    & 95    & 92    & 85    & 69    & 49    & 36 \\
          & Sherlock \cite{dutta2017output} & 99    & 98    & 97    & 95    & 93    & 90    & 74    & 59    & 37    & 29 \\
\cmidrule{2-12}          & MIP   & 100   & 99    & 97    & 97    & 95    & 92    & 85    & 69    & 49    & 36 \\
          & Benders & 100   & 99    & 97    & 97    & 95    & 92    & 85    & 69    & 49    & 36 \\
    \midrule
    \multicolumn{1}{c|}{\multirow{5}[6]{*}{Hardtanh}} & Duality \cite{dvijotham2018dual} & 99    & 93    & 79    & 47    & 42    & 38    & 29    & 19    & 8     & 0 \\
\cmidrule{2-12}          & DLV \cite{huang2017safety}   & 99    & 93    & 79    & 47    & 42    & 38    & 29    & 19    & 8     & 0 \\
          & MaxSens \cite{xiang2018output} & 99    & 91    & 73    & 41    & 38    & 33    & 26    & 15    & 6     & 0 \\
\cmidrule{2-12}          & MIP   & 100   & 100   & 99    & 97    & 95    & 92    & 89    & 82    & 72    & 49 \\
          & Benders & 100   & 100   & 99    & 97    & 95    & 92    & 89    & 82    & 72    & 49 \\
    \midrule
    \multicolumn{1}{c|}{\multirow{5}[5]{*}{Sigmoid}} & Duality \cite{dvijotham2018dual} & 100   & 95    & 84    & 74    & 63    & 49    & 44    & 36    & 26    & 19 \\
\cmidrule{2-12}          & DLV \cite{huang2017safety}   & 100   & 95    & 84    & 74    & 63    & 49    & 44    & 36    & 26    & 19 \\
          & MaxSens \cite{xiang2018output} & 100   & 90    & 78    & 67    & 56    & 44    & 39    & 32    & 22    & 12 \\
\cmidrule{2-12}          & MIP   & 100   & 96    & 85    & 76    & 66    & 52    & 46    & 38    & 31    & 21 \\
          & Benders & 100   & 96    & 85    & 76    & 66    & 52    & 46    & 38    & 30    & 21 \\
    \bottomrule
    \end{tabular}
}
\caption{Complete Numerical Verification Results of Benders Algorithm}
\label{tab:benders_results}
\end{table*}

\section{Extended Verification Results on Deep Architectures (Banknotes Benchmark)}\label{appendix_banknotes_details}

In this section, we provide extended experimental details, dataset specifications, and algorithmic profiling for the deep architecture verification experiments presented in Section~\ref{subsec:exp_banknotes}.

\subsection{Benchmark Dataset and Network Architecture}
The evaluation utilizes the UCI Banknote Authentication benchmark dataset \cite{dua2019uci}, a recognized standard for evaluating binary neural network classifiers on continuous inputs. The dataset features $n_0 = 4$ continuous real-valued attributes extracted from photographic images of banknote specimens via Continuous Wavelet Transform:
\begin{enumerate}
    \item $x_1$: Variance of Wavelet Transformed image;
    \item $x_2$: Skewness of Wavelet Transformed image;
    \item $x_3$: Curtosis of Wavelet Transformed image;
    \item $x_4$: Entropy of image.
\end{enumerate}
Features are standardized to min $0$ and max $1$ based on training statistics. The network is a deep fully connected neural network with architecture 4-8-8-8-1 and ReLU activation, comprising 24 hidden ReLU neurons across 3 hidden layers and 1 scalar output logit. Under our exact Benders verification mapping, each hidden ReLU neuron corresponds to a binary activation decision $w_j^l \in \{0, 1\}$, inducing a combinatorial configuration space of $2^{24} = 16,777,216$ possible activation patterns.

Verification is evaluated on $N_{\mathrm{test}} = 100$ independent test samples over six perturbation budgets with $\varepsilon$ from 0.02 to 0.12, step size 0.02 under the $\ell_\infty$ norm ball $\mathcal{X} = \{\mathbf{x} \in \mathbb{R}^4 \mid \|\mathbf{x} - \mathbf{x}_0\|_\infty \le \varepsilon\}$.

\subsection{Subproblem Solver Details}
Continuous feasibility subproblems and certificate extractions are computed using classical linear programming solvers (simplex/interior-point). When a candidate configuration $\mathbf{w}$ is infeasible, the subproblem returns:
    \begin{itemize}
        \item \emph{Standard Farkas}: The dual unbounded ray direction $(\hat{\bm{\pi}}, \hat{\bm{\lambda}})$, yielding continuous hyperplane constraint~\eqref{eq:appendix_benders_master_feas_cut}.
        \item \emph{Contextual Helly}: The earliest infeasible depth $t^\star = \min \{t \mid R^t(\mathbf{w}) = \varnothing\}$, the feasible parent context $p^\star = \mathbf{w}_{<t^\star}$, and a minimal conflict certificate $S \subseteq \{1, \dots, q\}$ of size $|S| \le \rho_{p^\star}^{t^\star} + 1 \le n_0 + 1 = 5$, yielding context-guarded cut~\eqref{eq:guarded_cut}.
        \item \emph{Hybrid}: Both the continuous Farkas hyperplane and the discrete context-guarded Helly cut.
    \end{itemize}

\subsection{Theoretical Rank Geometry and Scalability}
A crucial insight from Theorem~\ref{thm:benders_complexity} and Remark~\ref{rem:input_dim_rank} is that the contextual conflict rank is geometrically bounded by the input dimension: $\rho_p^l \le \dim(V_p) \le n_0$. For the Banknotes network, $n_0 = 4$, which rigorously ensures that:
\begin{equation}
    |S| \le \rho + 1 \le 4 + 1 = 5.
\end{equation}
Consequently, each extracted Helly certificate contains at most 5 literals out of 24, regardless of network depth. In the master search space, a 5-literal conflict cut immediately eliminates an entire affine subcube of size:
\begin{equation}
    2^{24 - |S|} \ge 2^{24 - 5} = 2^{19} = 524,288 \text{ activation patterns}.
\end{equation}
This explains the dramatic reduction in master queries observed in Table~\ref{tab:banknotes_results}: while Farkas cuts generate loose continuous boundaries that permit repeated queries within the same infeasible combinatorial subcube (requiring 1,578 CIM queries at $\varepsilon=0.120$), Contextual Helly and Hybrid cuts prune over half a million infeasible patterns per cut, eliminating combinatorial query saturation and accelerating end-to-end verification by $13.41\times$.

\end{document}